\documentclass[11pt]{article}
\usepackage[algo2e,linesnumbered,ruled,vlined]{algorithm2e}

\usepackage{amsmath}
\usepackage{amsthm}
\usepackage{amssymb}
\usepackage{amsbsy}
\usepackage{mathabx}
\usepackage{times}
\usepackage{bm}

\usepackage[utf8]{inputenc} % allow utf-8 input
\usepackage[T1]{fontenc}    % use 8-bit T1 fonts
\usepackage{booktabs}       % professional-quality tables
\usepackage{amsfonts}       % blackboard math symbols
\usepackage{nicefrac}       % compact symbols for 1/2, etc.
\usepackage{microtype}      % microtypography
\usepackage{tcolorbox}

\usepackage{mathrsfs}
\usepackage{enumerate}
\usepackage{enumitem}
\usepackage{caption}
\usepackage{natbib}
\usepackage{star}
\usepackage{xcolor}

\usepackage[colorlinks,
linkcolor=blue,
anchorcolor=blue,
citecolor=blue]{hyperref}

\allowdisplaybreaks
\newtheorem{thm}{Theorem}[section]
\newtheorem{lem}{Lemma}[section]
\newtheorem{cor}{Corollary}[section]

\newtheorem{asmp}{Assumption}[section]
\newtheorem{defn}{Definition}[section]

\def\##1\#{\begin{align}#1\end{align}}
\def\$#1\${\begin{align*}#1\end{align*}}

\usepackage{geometry}

\hypersetup{
	colorlinks=true,
	linkcolor=blue,
	filecolor=blue,
	citecolor = blue,
	urlcolor=cyan,
}

\usepackage{amsmath,amsfonts,bm}

\def\ceil#1{\lceil #1 \rceil}

\def\1{\bm{1}}

\def\tw{\widetilde{\w}}

\def\sV{{\mathbb{V}}}

\def\AM{{\mathcal A}}
\def\BM{{\mathcal B}}

\def\CM{{\mathcal C}}
\def\DM{{\mathcal D}}
\def\EM{{\mathcal E}}
\def\GM{{\mathcal G}}
\def\FM{{\mathcal F}}
\def\IM{{\mathcal I}}
\def\HM{{\mathcal H}}
\def\JM{{\mathcal J}}
\def\KM{{\mathcal K}}

\def\NM{{\mathcal N}}
\def\OM{{\mathcal O}}
\def\RM{{\mathcal R}}

\def\SM{{\mathcal S}}

\def\VM{{\mathcal V}}

\def\RB{{\mathbb R}}
\def\EB{{\mathbb E}}
\def\VB{{\mathbb V}}
\def\HVB{\widehat{\mathbb V}}

\def\PB{{\mathbb P}}
\def\HPB{\widehat{\mathbb P}}

\def\ph{\mbox{\boldmath$\phi$\unboldmath}}

\def\argmax{\mathop{\rm argmax}}
\def\argmin{\mathop{\rm argmin}}

\def\A{{\boldsymbol A}}

\def\B{{\boldsymbol B}}

\def\d{{\boldsymbol d}}

\def\e{{\boldsymbol e}}

\def\H{{\boldsymbol H}}
\def\I{{\boldsymbol I}}

\def\V{{\boldsymbol V}}
\def\v{{\boldsymbol v}}

\def\w{{\boldsymbol w}}

\def\x{{\boldsymbol x}}

\def\Z{{\boldsymbol Z}}

\def\0{{\boldsymbol 0}}
\def\1{{\boldsymbol 1}}

\def\Bphi{{\boldsymbol \phi}}
\def\TBphi{\widetilde{{\boldsymbol \phi}}}
\def\Bf{{\boldsymbol f}}
\def\Bmu{{\boldsymbol \mu}}
\def\BH{{\boldsymbol H}}

\def\Btheta{{\boldsymbol \theta}}
\def\Bdelta{{\boldsymbol \delta}}

\def\Beps{{\boldsymbol \varepsilon}}

\def\Bpsi{{\boldsymbol \psi}}
\def\est{{ {\boldsymbol \theta}}}
\def\estau{{ {\tau}}}

\def\overoptQ{\overline{Q}}
\def\optQ{\widehat{Q}}
\def\pesQ{\widecheck{Q}}
\def\overpesQ{\underline{Q}}

\def\overoptV{\overline{V}}
\def\optV{\widehat{V}}
\def\pesV{\widecheck{V}}
\def\overpesV{\underline{V}}

\def\BoveroptV{\overline{\V}}
\def\BoptV{\widehat{\V}}
\def\BpesV{\widecheck{\V}}
\def\BoverpesV{\underline{\V}}

\def\TOM{\widetilde{\mathcal O}}

\def\ttau{\widetilde{\tau}}
\def\TH{\widetilde{{\boldsymbol H}}}
\def\tw{\widetilde{w}}

\def\TRM{\widetilde{\mathcal R}}

\def\HR{\widehat{R}}

\def\kl{k_{\mathrm{last}}}
\begin{document}

\title{ \LARGE Variance-aware robust reinforcement learning with linear function approximation under heavy-tailed rewards}    

\author{
	Xiang Li\thanks{School of Mathematical Sciences, Peking University;  E-mail: \texttt{lx10077@pku.edu.cn}. The work was done when Xiang Li was a visiting student in the IVGS program at the Department of Statistical Sciences, University of Toronto.}\and 
	Qiang Sun\thanks{Department of Statistical Sciences, University of Toronto; E-mail: \texttt{qiang.sun@utoronto.ca}. }  
}
\date{ }

\maketitle

\vspace{-0.25in}

\begin{abstract}

This paper presents two algorithms, AdaOFUL and VARA, for online sequential decision-making in the presence of heavy-tailed rewards with only finite variances. For linear stochastic bandits, we address the issue of heavy-tailed rewards by modifying the adaptive Huber regression and proposing AdaOFUL. AdaOFUL achieves a state-of-the-art regret bound of $\widetilde{\mathcal{O}}\big(d\big(\sum_{t=1}^T \nu_{t}^2\big)^{1/2}+d\big)$ as if the rewards were uniformly bounded, where $\nu_{t}^2$ is the observed conditional variance of the reward at round $t$, $d$ is the feature dimension, and $\widetilde{\mathcal{O}}(\cdot)$ hides logarithmic dependence. Building upon AdaOFUL, we propose VARA for linear MDPs, which achieves a tighter variance-aware regret bound of $\widetilde{\mathcal{O}}(d\sqrt{H\mathcal{G}^*K})$. Here, $H$ is the length of episodes, $K$ is the number of episodes, and $\mathcal{G}^*$ is a smaller instance-dependent quantity that can be bounded by other instance-dependent quantities when additional structural conditions on the MDP are satisfied. Our regret bound is superior to the current state-of-the-art bounds in three ways: (1) it depends on a tighter instance-dependent quantity and has optimal dependence on $d$ and $H$, (2) we can obtain further instance-dependent bounds of $\mathcal{G}^*$ under additional structural conditions on the MDP, and (3) our regret bound is valid even when rewards have only finite variances, achieving a level of generality unmatched by previous works. Overall, our modified adaptive Huber regression algorithm may serve as a useful building block in the design of algorithms for online problems with heavy-tailed rewards.
\end{abstract}

\tableofcontents

\section{Introduction}\label{sec:intro}

% {\bf General descriptions of RL here?}
In reinforcement learning (RL), an agent interacts with an environment by taking actions and receiving rewards. The goal of the agent is to learn a policy that maximizes its expected cumulative reward over time. Since the agent does not initially know the optimal policy, it needs to explore different actions and learn from the feedback it receives \citep{sutton2018reinforcement}.  Reinforcement learning has demonstrated remarkable empirical successes in various applications, including robotics \citep{lillicrap2015continuous}, dialogue systems \citep{li2016deep}, and Go play \citep{silver2016mastering}. Despite its empirical successes, the theoretical understanding of RL is still in its early stages. An important theoretical goal is to characterize  the regret of an agent, defined as the  difference between the cumulative reward obtained by the agent and the cumulative reward that could have been obtained by following the optimal policy from the beginning. In other words, regret  measures the suboptimality of the agent.

% From tabular MDP to linear MDP.
Regret analysis in reinforcement learning (RL) has mainly focused on tabular Markov decision processes (MDPs) with finite and small state and action spaces~\citep{azar2017minimax,sidford2018near,dann2019policy,yang2019sample,tossou2019near,zhang2020almost,zanette2020learning,he2021nearly}. 
However, many real-world applications have high-dimensional or even infinite-dimensional state and action spaces, where function approximation is necessary for computational tractability and better generalization. Linear function approximation is a widely studied technique in RL, particularly in the context of linear MDPs, where the value or policy function can be parameterized linearly~\citep{yang2020reinforcement,jin2020provably,jin2020reward,wagenmaker2022reward,zanette2020learning,ayoub2020model,zhou2021nearly}. In this linear setting,  \cite{jin2020provably} presented the first provable RL algorithm  that achieves  $\TOM(\sqrt{d^3H^4K})$ regret with $H$ being the episode length and $K$ the number of episodes, which was then improved to the optimal $\TOM(\sqrt{d^2H^3K})$ regret by~\cite{he2022nearly}. Recently, there has been a shift in focus from designing algorithms with worst-case regrets to those with variance-aware regrets \citep{pananjady2020instance,khamaru2021temporal,li2021polyak,yin2021towards,min2021variance}. Variance-aware regrets depend on the variances of rewards and value functions and provide stronger guarantees than worst-case bounds by characterizing problem-dependent performances across different problem instances.

%\footnote{The uniform boundedness assumption is equivalent to assuming rewards are sub-Gaussian up to logarithmic factors.}

%%%%From lighttailed to heavy-tailed
However, the aforementioned works assume that rewards are either uniformly bounded or sub-Gaussian, limiting their applicability to real-world problems with heavy-tailed behaviors. For instance, stock returns in financial markets \citep{cont2001empirical, hull2012risk}, microarray data analysis \citep{posekany2011biological}, and advertiser values in online advertising \citep{arnosti2016adverse} exhibit heavy-tailed behaviors. To address this limitation, some researchers truncate rewards to achieve sub-linear worst-case regret bounds for multi-arm bandits \citep{bubeck2013bandits}, linear bandits \citep{medina2016no,shao2018almost,xue2021nearly}, and tabular MDPs \citep{zhuang2021no} with finite variances. However, these truncation-based methods have estimation errors that depend on the absolute moments of observations, rather than their central moments, and thus are non-diminishing in the noiseless case, indicating their suboptimality.

%Our contribution
%The resulting estimator admits a sub-Gaussian-type deviation bound even when the errors have only bounded variances.
%These results, however, do not directly apply to  online bandits because the data are collected in an adaptive manner and thus fail to be independently and identically distributed. 
%To handle non-i.i.d. data, we carefully chose different robustification parameters $\tau$'s for different data points, which helps balance robustness and asymptotic unbiasedness.  
%Using this modified technique, we propose the  \underline{Ada}ptive Huber regression based \underline{OFUL} (AdaOFUL) algorithm.

\paragraph{Our contributions} 
This paper proposes new algorithms for achieving tight variance-aware regret bounds that depend on the central moments instead of absolute moments,  for linear bandits and linear MDPs with heavy-tailed awards that have only finite variances. We use the term ``heavy-tailed rewards'' throughout the paper to refer to such rewards for simplicity.  
Our method is motivated by adaptive Huber regression~\citep{sun2020adaptive, sun2021we}, which was originally proposed for analyzing offline independently and identically distributed (i.i.d.) data. It uses the (pseudo-) Huber loss  to estimate the unknown coefficient  with a universal robustification parameter.  We adapt this method for online bandits and carefully choose different robustification parameters to handle non-i.i.d. data. The resulting algorithm, called AdaOFUL (short for \underline{Ada}ptive Huber regression based \underline{OFUL}), achieves the state-of-the-art regret bound $ \widetilde{\OM}\left(
d \sqrt{\sum_{t \in [T]} \nu_t^2} +d  \right)$ for linear bandits with heavy-tailed rewards, where $\nu_{t}^2$ is the observed conditional variance of the random reward at step $t$ and $d$ is the feature dimension.
Such a variance-aware regret bound has only been obtained in the literature of linear bandits with sub-Gaussian or uniformly bounded rewards~\citep{kirschner2018information, zhou2022computationally}. In contrast, truncation-based methods are suboptimal due to their estimation errors that depend on absolute moments instead of central moments. Our regret bound depends on the central moments instead, and is thus tighter. 

%In order to obtain an as tight as possible variance-aware regret bound, %when the number of episodes $K$ is sufficiently large, 

Building upon AdaOFUL, we then propose the \underline{V}ariance-\underline{A}ware \underline{R}egret via the \underline{A}datptive Huber regression (VARA) algorithm for linear MDPs with heavy-tailed rewards. 
Our algorithm improves upon state-of-the-art methods \citep{wagenmaker2021first,hu2022nearly,he2022nearly} by using finer variance estimators, resulting in a regret bound of $\TOM(d\sqrt{H \GM^* K})$, where $H$ is the episode length and $\GM^*$ is a variance-aware quantity bounded by the sum of weighted per-step conditional variances. Our regret bound is superior to the current state-of-the-art bounds in three ways. First, it depends on a tighter instance-dependent quantity $\GM^*$ and has optimal dependence on $d$ and $H$. Second, assuming additional structural conditions on the underlying MDP, we can obtain further instance-dependent bounds of $\GM^*$, including range-dependent and concentrability-dependent bounds. Third, our regret bound $\TOM(d\sqrt{H \GM^* K})$ is valid even when rewards have only finite variances, which achieves a level of generality that is unmatched by previous works. 

Our findings indicate that heavy-tailed rewards do not pose a limitation for developing online RL algorithms with linear function approximations that are provably efficient. Our proposed modified adaptive Huber regression algorithm can be used as a general approach to adapt existing online algorithms designed for light-tailed rewards to handle heavy-tailed ones while maintaining tight dependence on variance for regret bounds.

\paragraph{Outline}
The rest of the paper proceeds as  follows. Section~\ref{sec:notation}  introduces the heavy-tailed linear bandits and linear MDPs. We state our main results for heavy-tailed linear bandits in Section~\ref{sec:linear-bandit} and for linear MDPs in Section~\ref{sec:mdp}. We review related work in Section~\ref{sec:related} and conclude in Section~\ref{sec:conclusion}. All proofs are collected in the appendix.

% {\bf Notation:}
\paragraph{Notation}
We use $\|\cdot\|$ to denote the $\ell_2$-norm in $\RB^d$, and $\mathrm{Ball}_d(B)$ {the $\ell_2$-norm ball in $\RB^d$ with radius $B > 0$.}
For a positive definite matrix $\H \in \RB^{d \times d}$, let $\|\x\|_{\H} = \sqrt{\x^\top\H \x}$ for a vector $\x \in \RB^d$.
For two semidefinite positive matrices $\H_1, \H_2$, we denote $\H_1 \succeq \H_2$ if $\H_2 - \H_1$ is semidefinite positive.
For an integer $K \in \mathbb{N}^+$, let $[K]:= \{1, 2,\cdots, K\}$. For a set $\cA$, $|\cA|$ denotes its cardinality. 
For real numbers $a \le b$ and $x \in \RB$, we use $x_{[a, b]} := \max\{a, \min \{x, b\}\}$ to denote the projection of $x$ onto the closed interval $[a, b]$. SOTA is short for state-of-the-art.

%%%%%Preliminaries
\section{Preliminaries}\label{sec:notation}

In this section, we introduce both stochastic linear bandits and linear MDPs with heavy-tailed rewards.

%%%%%Linear Bandits
\subsection{Heavy-tailed Stochastic Linear Bandit}

Below, we define heavy-tailed stochastic linear bandits, where the mean-zero random noises $\varepsilon_t$ have only bounded variances. We emphasize that in linear bandits, data are collected in an adaptive manner, and therefore, the distribution of $\varepsilon_t$ depends on $\Bphi_t$. Moreover, the choice of $\Bphi_t$ depends on all past observations ${(\Bphi_{s}, y_{s}, \nu_s)}_{s < t}$.

\begin{defn}[Heavy-tailed stochastic linear bandit]\label{def:heavy}
	Let $\{\DM_t\}_{t\ge 1}$ denote a fixed sequence of decision sets and $\{ \FM_t\}_{t \ge 1}$ a filtration.
	At round $t$,  the agent chooses $\Bphi_t \in \DM_{t}$ and then observes the reward $y_t$ and its conditional variance $\nu_t^2$.
	We assume $y_t = \langle \Bphi_t,  \Btheta^*\rangle + \varepsilon_t$ where $\Btheta^* \in \RB^d$ is a vector unknown to the agent and $\varepsilon_t \in \RB$ is a martingale difference random noise such that $\EB[\varepsilon_t|\FM_{t-1}] = 0$ and  $\EB[\varepsilon_t^2|\FM_{t-1}] =\nu_t^2$.
	Both $\nu_t$ and $\Bphi_t$ are $\FM_{t-1}$-measurable and $\|\Bphi_t\| \le L$.
	We assume $\|\Btheta^*\| \le B$ with $B$ known {\it  a priori}.
The agent aims  to minimize the regret, formally defined as
	\begin{equation}
		\label{eq:regret}
		\mathrm{Reg}(T) := \sum_{t=1}^T \left[  \sup_{\Bphi \in \DM_t} \langle \Bphi, \Btheta^* \rangle - \langle \Bphi_t, \Btheta^* \rangle  \right].
	\end{equation}
\end{defn}

%%%%%linear MDP
\subsection{Linear MDP with Heavy-tailed Rewards}

An episodic finite horizon MDP is denoted by a tuple $\mathcal{M} = (\SM, \AM, H, \{r_h\}_{h \in[H]}, \{\PB_h\}_{h \in [H]})$ where $\SM$ is the state space, $\AM$  the action space, $H \in \mathbb{Z}^+$  the length of each episode, $\PB_h: \SM \times \AM \to \Delta(\SM)$  the transition probability function,  and $r_h: \SM \times \AM \to \RB$  the expected reward function. 
A linear MDP   assumes that  both the transition probability and the expected reward are  linear in a known state-action feature map $\Bphi(\cdot, \cdot) \in \RB^d$~\citep{bradtke1996linear,melo2007q,yang2019sample,jin2020provably}.

\begin{defn}[Linear MDP]
	$\mathcal{M}$ is called a time-inhomogeneous linear MDP, if there
	exist some known feature map $\Bphi(s, a): \SM \times \AM  \to \mathrm{Ball}_d(1)$, unknown signed measures $ \{\Bmu_h^*\}_{h \in [H]} \subseteq \RB^{d \times |\SM|}$, and unknown coefficients $\{ \Btheta_h^*\}_{h \in [H]} \subseteq \mathrm{Ball}_d(W)$ such that
	\[
	r_h(s, a) =  \langle \Bphi(s, a), \Btheta_h^* \rangle,
	\PB_h(\cdot|s, a) = \langle \Bphi(s, a), \Bmu_h^*(\cdot) \rangle,
	\]
	for any $(s, a) \in \SM \times \AM, ~  h \in [H],$
	where $\| \Bmu_h^*(\SM)\| := \| \sum_{s \in \SM} \Bmu_h^*(s)\|  \le \sqrt{d}$ for all $h \in [H]$.
\end{defn}

For a time-inhomogeneous MDP, we denote its deterministic and time-dependent policy by $\pi = \{\pi_h\}_{h \in [H]}$. %Now fix a policy $\pi$. 
Let $\{(s_h, a_h)\}_{h \in [H]}$ be state-action pairs such that $a_h = \pi_h(s_h)$ and $s_{h+1} \sim \PB_h(\cdot|s_h, a_h)$.
Define the occupancy measure for the policy $\pi$ at the $h$-th round  by $d_h^{\pi}(s, a) = \PB^{\pi}(s_h=s, a_h=a|s_1)$ where $(a_1, s_2, a_2, \cdots, s_h, a_h)$ is a trajectory starting from $s_1$ and following the policy $\pi$.
The state-action function $Q_h^\pi(\cdot,\cdot)$ and value function $V_h^\pi(\cdot)$ at the $h$-th round are defined as 
\begin{gather*}
	Q_h^\pi(\cdot,\cdot) = \EB [\sum_{i=h}^H r_i(s_i, a_i)|(s_h, a_h) =( \cdot, \cdot)]
\end{gather*}
and $	V_h^\pi(\cdot) = \sum_{a \in \AM} 	Q_h^\pi(\cdot,a)$ respectively. 
For any value function $V$, write 
\[
[\PB_h V](s, a) = \EB_{s' \sim \PB_h(\cdot|s, a)} V(s'),~~~ [\VB_h V](s, a) = [\PB_h V^2](s, a) - [\PB_h V]^2(s, a).
\]
With a slight abuse of notation, let $[\PB_h R_{h}](s, a)$ and $[\VB_hR_h](s, a)$ denote the expectation and variance of the random reward $R_h(s, a)$ at the $h$-th round given state-action pair $(s,a)$.
We consider linear MDPs with heavy-tailed random rewards that satisfy the following assumptions. %Recall that $[\PB_h R_{h}](s, a) = r_h(s, a) =  \langle \Bphi(s, a), \Btheta_h^* \rangle$ by the definition of linear MDPs. 

\begin{asmp}[Realizable reward]\label{asmp:realizable}
 %Let $R_h(s, a)$ denote the random reward received at the $h$-th step given the state-action pair $(s, a)$.
%We assume:
We assume that the following hold. 
\begin{enumerate}
\item For all $(s, a) \in \SM \times \AM$ and $h \in [H]$, the random reward $R_h(s,a)$ is independent of $s_{h+1}(s,a)$, where $s_{h+1}(s,a) \sim \PB_h(\cdot | s,a)$ represents the next state transitioned from $(s,a)$ at the $h$-th round. 

\item There exists known feature maps $\TBphi(s, a): \SM \times \AM  \to \mathrm{Ball}_d(1)$ and unknown coefficients  $\{\Bpsi_h^*\}_{h \in [H]}  \subseteq \mathrm{Ball}_d(W)$ so that $[\PB_h R_{h}^2](s, a) = \langle\TBphi(s, a), \Bpsi_h^* \rangle$ for all $(s, a) \in \SM \times \AM$ and $h \in [H]$. 
  
\end{enumerate}
\end{asmp}

\begin{asmp}[Bounded variance]
	\label{asmp:H-V}
We assume that the following hold. 
	\begin{enumerate}
		\item There exist known constants $\sigma_{R}, \sigma_{R^2} > 0$ such that $[\VB_hR_h](s, a)  \le \sigma_{R}^2$ and $[\VB_hR_h^2](s, a)  \le \sigma_{R^2}^2$ for all $(s, a) \in \SM \times \AM$ and $h \in [H]$.
		\item 	For any policy $\pi$, we have $0 \le \EB R_\pi \le \mathcal{\HM}$ and $\Var(R_{\pi}) \le \VM^2$ where $R_\pi = \sum_{h =1}^H R_h(s_h, a_h)$ denotes the sum of random rewards along the trajectory following  $\pi$.
	\end{enumerate}
\end{asmp}

%Assumption \ref{asmp:realizable} assumes that the random reward at each round is independent with future states  and its second moment is realizable in some known feature map. 
%With this assumption, linear MDPs recover tabular MDPs by taking $d = |\SM \times \AM|$ as the size of the state-action space and $\Bphi(s, a) = \TBphi(s, a) = \e_{(s, a)}$ as the canonical basis in $\RB^d$.
%Assumption \ref{asmp:H-V}  upper bounds the means of the cumulative rewards, and the variances of the random rewards, the squared random rewards, and the cumulative rewards. 
%This assumption generalizes uniformly bounded reward assumption in the literature, that is, $0 \le \sup_{(s, a) \in \SM \times \AM} \sup_{h \in [H]}R_h(s, a) \le 1$.
%To the best of our knowledge, Assumption \ref{asmp:H-V} is the weakest moment condition on random rewards in the RL literature.
Assumption \ref{asmp:realizable} assumes that the random reward at each round is independent of future states and its second moment can be realized using a known feature map. Under this assumption, linear MDPs can recover tabular MDPs by setting the size of the state-action space as $d = |\SM||\AM|$ and using the canonical basis $\Bphi(s, a) = \TBphi(s, a) = \e_{(s, a)}$ in $\RB^d$. Assumption \ref{asmp:H-V} places upper bounds on the means of the cumulative rewards, as well as on the variances of the random rewards, the squared random rewards, and the cumulative rewards. This assumption generalizes the uniformly bounded reward assumption found in the literature, which requires $0 \le \sup_{(s, a) \in \SM \times \AM} \sup_{h \in [H]}R_h(s, a) \le 1$. As far as we know, Assumption \ref{asmp:H-V} is the weakest moment condition on random rewards in the RL literature.

\paragraph{Learning protocol} 
Let $\FM_{h,k}$ denote the $\sigma$-field generated by all random variables up to, and including, the $h$-th round and $k$-th episode. At the beginning of each episode $k$, the environment selects the initial state $s_{1, k}$. The agent proposes a policy $\pi_k = \{ \pi_h^k \}_{h \in [H]}$ based on the history up to the end of episode $k-1$, and then executes $\pi_k$ to generate a new trajectory $ \{ (s_{h, k}, a_{h, k}, r_{h, k}) \}_{h \in [H]}$. Here  $a_{h, k} = \pi_h^k(s_{h, k}), r_{h, k} \sim R_h(s_{h, k}, a_{h, k})$ and $s_{h+1, k}\sim \PB(\cdot|s_{h, k}, a_{h, k})$.  The agent aims to minimize the cumulative regret over $K$ episodes, given by:
\[
\mathrm{Reg}(K) := \sum_{k=1}^K (V_1^*-V_1^{\pi_k})(s_{1, k}).
\]

\section{Variance-aware Regret for Linear Bandits}
\label{sec:linear-bandit}

In this section, we introduce the AdaOFUL algorithm for linear bandits and demonstrate its ability to achieve state-of-the-art variance-aware regret even when faced with heavy-tailed rewards, as though the rewards were uniformly bounded.

%%%%%%%%%%%%%%%%%%%%%%%%%%%%%%%%%
%%%%%Algorithm Description%%%%%%%
%%%%%%%%%%%%%%%%%%%%%%%%%%%%%%%%%
\subsection{Algorithm Description}

%uses adaptive Huber regression to compute a new parameter vector $\Btheta_t$ that takes into account the heavy-tailed rewards observed.
%which directly selects  the agent then makes the most optimistic estimate $\widetilde\Btheta_t$ in order to make an arm selection $\Bphi_{t}$ which is equivalent to the following joint optimization problem
%\$
%(\Bphi_{t}, \tilde\Btheta_t) = \argmax_{ \Bphi \in \DM_{t}, \Btheta \in \CM_{t-1}} \langle \Bphi, \Btheta \rangle. 
%\$ 
%Following the decision, the agent then plays $\Bphi_t$ and observes the random reward $y_{t}$ and its conditional variance $\nu_{t}$. Comparing with the standard OFUL algorithm \citep{abbasi2011improved}, which directly selects  the agent then makes the most optimistic estimate $\widetilde\Btheta_t$ in order to make an arm selection $\Bphi_{t}$ which is equivalent to the following joint optimization problem
%\$
%(\Bphi_{t}, \tilde\Btheta_t) = \argmax_{ \Bphi \in \DM_{t}, \Btheta \in \CM_{t-1}} \langle \Bphi, \Btheta \rangle. 
%\$ 
%Following the decision, the agent then plays $\Bphi_t$ and observes the random reward $y_{t}$ and its conditional variance $\nu_{t}$. Comparing with the standard OFUL algorithm \citep{abbasi2011improved}, AdaOFUL does not keep  $\widetilde\Btheta_t$ but relies on adaptive Huber regression to compute a new $\Btheta_t$ in the presence of heavy-tailed rewards, which will be discussed in detail later.  The last step of round $t$ updates intermediate parameters $\beta_t$ and $\H_t$ for confidence set construction in the next round

This section presents the AdaOFUL algorithm for heavy-tailed linear bandits. The AdaOFUL algorithm is given in Algorithm \ref{algo:adap}.  AdaOFUL follows the principle of \underline{O}ptimism in the \underline{F}ace of \underline{U}ncertainty (OFU)~\citep{abbasi2011improved} to solve the heavy-tailed heterogeneous linear bandit problem.
At each round $t$, it maintains  a confidence set
   \begin{equation}
    \label{eq:CI}
    \CM_{t-1} := \left\{ \est \in \mathrm{Ball}_d(B):
    \| \est	- \est_{t-1}\|_{\H_{t-1}} \le  \beta_{t-1} \right\}
    \end{equation} 
such that $\Btheta^* \in \CM_{t}$ uniformly for all $t \ge 1$ with high probability when  the exploration radius $\beta_{t-1}$ is properly chosen.  
Unlike the standard OFUL algorithm \citep{abbasi2011improved} which directly selects the most optimistic estimator $\widetilde\Btheta_t$ to make an arm selection $\Bphi_{t}$, AdaOFUL uses adaptive Huber regression to compute a new estimator $\Btheta_t$ that takes into account the heavy-tailed rewards. The agent then selects the arm $\Bphi_{t}$ that maximizes the inner product $\langle \Bphi, \Btheta \rangle$ with the new estimator $\Btheta_t$. After playing the selected arm, the agent observes the reward $y_t$ and its conditional variance $\nu_t$. The last step of round $t$ updates the exploration radius $\beta_t$ and the shape matrix $\H_t$ for the confidence set construction in the next round.

\begin{algorithm}[t!]
	\DontPrintSemicolon
	\SetAlgoLined
	\SetNoFillComment
	\SetKwInOut{Require}{Require}
	\SetKwInOut{Ini}{Initialization}
	%\SetSideCommentLeft
%	\tcc{iterate over all training examples}
%\Require{Total iteration $T$, an upper bound $B$ for $\|\Btheta^*\|$,  an upper bound $L$ for decision vectors, robust parameter $\tau > 0$, regularization parameter $\lambda > 0$.}
\Ini{$\H_0 = \lambda\I, \Btheta_0 = \0, \beta_0 = \sqrt{\lambda} B, c_0 = \frac{1}{6\sqrt{3\log\frac{2T^2}{\delta}}}, c_1 = \frac{1}{42 \cdot\log\frac{2T^2}{\delta}}, \sigma_{\min} = \frac{1}{\sqrt{T}}$.}
\For{$t=1$ {\bfseries to} $T$}{
    Construct the confidence set $\CM_{t-1}$ as in~\eqref{eq:CI} \;
    
    Solve  $(\Bphi_{t}, \cdot) = \argmax_{ \Bphi \in \DM_{t}, \Btheta \in \CM_{t-1}} \langle \Bphi, \Btheta \rangle$.\;
    
    Play $\Bphi_{t}$ and observe $(y_t, \nu_t)$.\;
    
    Set $\sigma_t, w_t$ and $\tau_t$ according to~\eqref{eq:parameters} and record $\{\sigma_s, w_s, \tau_s: 1\leq s\leq t \}$.\;
    
    Compute $\Btheta_{t}$ by minimizing~\eqref{eq:Theta_T}.\;
    
    Define $\beta_t$ as in~\eqref{eq:beta} and set $ \H_t = \H_{t-1} + \frac{\Bphi_t}{\sigma_t}\frac{\Bphi_t^\top}{\sigma_t}$.
    % \begin{equation}
    % \label{eq:H}
    % \H_t = \H_{t-1} + \frac{\Bphi_t}{\sigma_t}\frac{\Bphi_t^\top}{\sigma_t}
    % .	\end{equation}
	}
	\caption{Adaptive Huber regression based OFUL (AdaOFUL) }
	\label{algo:adap}
\end{algorithm}

%When $\beta_{t-1}$ is properly picked, it can be guaranteed that $\Btheta^* \in \CM_{t}$ uniformly for all $t \ge 1$ with high probability. 
%This implies that $\sup_{\Btheta \in \CM_{t-1}}\langle \Bphi_t, \Btheta \rangle$ is a valid upper bound of $\EB[y_t|\FM_{t-1}] := \langle \Bphi_t, \Btheta^* \rangle$ and thus achieves optimism.
%\iffalse
%\scolor{
%With all parameters properly picked, we ensure that $\Btheta^* \in \CM_{t}$ uniformly for all $t \ge 1$ with high probability. 
%This implies that $\sup_{\Btheta \in \CM_{t-1}}\langle \Bphi_t, \Btheta \rangle$ is a valid upper bound of $\EB[y_t|\FM_{t-1}] := \langle \Bphi_t, \Btheta^* \rangle$ and thus achieves optimism.
%To deal with the heavy-tailed noise, we use a pseudo-Huber estimator $\Btheta_{t}$ by solving the convex optimization problem~\eqref{eq:Theta_T}.
%We then update all intermediate parameters $\beta_t$ and $\H_t$.
%}
%\fi

\paragraph{Pseudo-Huber regression}

The pseudo-Huber loss~\citep{hastie2009elements, sun2021we} is  defined as 
\begin{equation}
	\label{eq:l}
	\ell_{\tau}(x) = \tau (\sqrt{\tau^2 + x^2} - \tau), 
\end{equation}
which is a smooth approximation to the well-known Huber loss~\citep{huber1964robust}. Similar to the Huber loss, the pseudo-Huber loss resembles a quadratic function for small values of $|x|$ and is approximately linear  when $x$ is large in magnitude, making the loss strongly convex when close to the origin and less sensitive to changes in the tails. 
The parameter $\tau$ controls the balance between the quadratic and linear regions and is referred to as the robustification parameter by \cite{sun2020adaptive} in the case of the Huber loss. Since the value of the robustification parameter needs to be adaptive to the data for an optimal tradeoff between robustness and unbiasedness, we shall also refer to the pseudo-Huber regression with a data-adaptive $\tau$ as adaptive pseudo-Huber regression or simply adaptive Huber regression, in line with \citep{sun2020adaptive}.

To compute the pseudo-Huber estimator $\Btheta_{t}$ for $\Btheta^*$, given the history $\{ (\Bphi_s, y_s, \nu_s) \}_{s \in [t]}$ up to time $t$, we solve the following convex optimization problem~\citep{sun2021we}:
\begin{gather}\label{eq:Theta_T}
	\est_{t}  := \argmin_{\Btheta \in \mathrm{Ball}_d(B)} L_t(\est)  ~\text{with} ~
		L_t(\est) :=
		\frac{\lambda}{2}\| \Btheta\|^2 + \sum_{s=1}^t 
		\ell_{\tau_s}\left(  \frac{y_s -  \langle \Bphi_s, \Btheta \rangle }{\sigma_s}  \right)
\end{gather}
where $\sigma_t$'s are surrogate conditional variances, and $\tau_t$'s are the robustification parameters, given by:
\begin{equation}\label{eq:parameters}
\hspace{-5pt}	\sigma_t = \max\left\{ \nu_t, \sigma_{\min},\frac{\|\Bphi_{t}\|_{\H_{t-1}^{-1}}}{c_0}, \frac{\sqrt{LB}\|\Bphi_t\|^{\frac{1}{2}}_{\H_{t-1}^{-1}}}{c_1^{\frac{1}{4}}d^{\frac{1}{4}}}
	\right\}, 
	w_t = \left\|\frac{\Bphi_t}{\sigma_t}\right\|_{\H_{t-1}^{-1}},~ \tau_t = \tau_0\frac{\sqrt{1+w_t^2}}{w_t}, 
\end{equation}
in which $\sigma_{\min}$ is a small positive constant to avoid singularity, $\tau_0$ is a hyper-parameter, $w_t$'s are importance measures, $c_0$ and $c_1$ are specified in Algorithm \ref{algo:adap}, and $L$ and $B$ are constants deinfed in Definition \ref{def:heavy}.

As shown in \eqref{eq:parameters}, the robustification parameter $\tau_t$ is set differently for each data point $(\Bphi_{t}, y_t, \nu_t)$ in the pseudo-Huber regression. This is a significant departure from the case of i.i.d. data, where all robustification parameters are typically set to the same value $\tau$, as i.i.d. data are naturally weighted equally \citep{sun2020adaptive}. In linear bandits, the data is generated adaptively, where the choice of $\Bphi_t$ can depend on all past observations. Since observations collected in later rounds are less important as they are based on previous observations and contribute less to the estimation accuracy, we assign greater weight to earlier observations. To measure the importance of the $t$-th observation, we use $w_t = \left\|\Bphi_t\right\|_{\H_{t-1}^{-1}}/\sigma_t$ as the importance measure for the $t$-th observation and set $\tau_t = \tau_0 {\sqrt{1+w_t^2}}/{w_t}$ as the corresponding robustification parameter.

When taking $\tau_0 = \infty$, the optimization problem in \eqref{eq:Theta_T} reduces to weighted regularized least-squares, which has been proven to achieve worst-case optimality for linear bandits with uniformly bounded or sub-Gaussian rewards \citep{kirschner2018information,zhou2022computationally}. However, an appropriate value of $\tau_0$ is necessary to balance robustness against heavy-tailed rewards and asymptotic unbiasedness. In Corollary~\ref{cor:tau-regret}, we will demonstrate that setting $\tau_0 = \TOM(\sqrt{d})$ is sufficient to achieve the state-of-the-art regret bound.

%the other direction depends on some normalized differences because

Lastly, we choose $\sigma_t \ge \sqrt{LB} {\|\Bphi_t\|^{\frac{1}{2}}_{\H_{t-1}^{-1}}}/(c_1^{\frac{1}{4}}d^{\frac{1}{4}})$, which implies $c_1d \ge {L^2B^2w_t^2}/({\sigma_t^2})$.
This condition is used to lower bound the Hessian matrix $\nabla^2 L_T(\est)$.
For any $\est \in  \mathrm{Ball}_d(B)$, we expect $\nabla^2 L_T(\est) \approx \H_T$ up to universal constant factors to proceed with theoretical analysis.
A direct computation yields $\nabla^2 L_T(\est) \preceq \H_T$, while for the other direction we show  $\nabla^2 L_T(\est) \succeq  \left(c - \sup_{t \in [T]} \left|\frac{\langle\Bphi_t, \est^*-\est\rangle}{\tau_t\sigma_t}\right|^2 \right)\H_T$ for some universal constant $c >0$ with high probability.
With the last condition on $\sigma_t$, for any feasible solution $\est \in  \mathrm{Ball}_d(B)$, the following quantity 
\$
\left|\frac{\langle\Bphi_t, \est^*-\est\rangle}{\tau_t\sigma_t}\right|^2
 \le \frac{
\|\Bphi_t\|^2 \|\est^*-\est\|^2}{\tau_t^2\sigma_t^2} \le \frac{4w_t^2L^2B^2}{\tau_0^2 \sigma_t^2} \le \frac{4c_1d}{\tau_0^2}
\$ 
can be sufficiently small provided that $\tau_0^2 \gtrsim c_1 d$.

\subsection{Regret Analysis}\label{sec:bandit}

We first validate that the optimism holds with high probability in Theorem~\ref{thm:heavy} and then establish a high probability bound for the regret in Theorem~\ref{thm:main-bandit}.
%All proofs are deferred to Appendix~\ref{proof:bandit}. 
\begin{thm}\label{thm:heavy}
Let $\kappa = d \cdot \log\left(1 + {T L^2}/({d\lambda \sigma_{\min}^2})\right)$.
For the heavy-tailed linear bandit in Definition~\ref{def:heavy}, if  $\tau_0  \sqrt{\log({2T^2}/{\delta})}\ge \max\{\sqrt{2\kappa}, 2\sqrt{d} LB \}$,  then with probability at least $1-3\delta$, it holds that,  for all $0 \le t \le T$,
\[
\|\Btheta_{t}-\Btheta^*\|_{\H_t} \le \beta_t, 
\]
where
\begin{equation}\label{eq:beta}
\beta_t
= 32\bigg(
\underbrace{ \frac{\kappa}{\tau_0}}_{\text{bias term}}
+ \underbrace{\vphantom{ \left(\frac{a^{0.3}}{b}\right) } 
 \sqrt{\kappa \log\frac{2t^2}{\delta}}}_{\text{variance term}}
+ \underbrace{ \vphantom{ \left(\frac{a^{0.3}}{b}\right) } 
 \tau_0\log\frac{2t^2}{\delta} }_{\text{range term}}
\bigg) 
+ \underbrace{\vphantom{ \left(\frac{a^{0.3}}{b}\right) } 
 5\sqrt{\lambda} B}_{\text{ridge term}}. 
\end{equation}
\end{thm}

Theorem~\ref{thm:heavy} establishes that $\est^*$ is contained in the set $\CM_t:=\left\{ \est \in \mathrm{Ball}_d(B):
\| \est	- \est_{t}\|_{\H_{t}} \le  \beta_{t} \right\}$ for all $t \ge 0$ with high probability.
It is proved by using Bernstein-type concentration inequality for self-normalized vector-valued martingales with additional care paid to deal with heavy-tailed rewards. 
To explain  the source of each term in~\eqref{eq:beta}, we first note that at a high level, 
 $\|\Btheta_{t}-\Btheta^*\|_{\H_t} \lesssim \|\nabla L_t(\est^*) \|_{\H_t^{-1}}$.
Following~\citet{zhou2021nearly}, an iterative analysis yields that $\|\nabla L_t(\est^*) \|_{\H_t^{-1}}^2 = \sum_{t=1}^T (X_t + Y_t)$ for two sequences of bounded random variables $X_t, Y_t \in \FM_t$.
To illustrate the proof idea, we explain how to bound $\sum_{t=1}^T X_t$, since $\sum_{t=1}^T Y_t$ can be bounded similarly. Here $\{X_t\}_{t \in [T]}$ is not a martingale difference sequence but $\{X_t-\EB[X_t|\FM_{t-1}]\}_{t \in [T]}$ is. We apply a standard Bernstein inequality to provide an upper bound for $\sum_{t=1}^T (X_t -\EB[X_t|\FM_{t-1}])$ that contributes to the variance and range terms. Thanks to the different robustification parameters $\tau_t$, we can control $\sum_{t=1}^T \EB[X_t|\FM_{t-1}]$ deterministically within $\OM\left({\kappa^2}/{\tau_0^2} \right)$, resulting in the bias term. Finally, the last ridge term $5\sqrt{\lambda} B$ exists because we use ridge regression to ensure that the Hessian is always invertible.

\begin{thm}\label{thm:main-bandit}
Let $\sigma_{\min} = {1}/{\sqrt{T}}$. Then with probability at least $1-3\delta$, we have
\[
\mathrm{Reg}(T)
\le 2\beta_T  \cdot \left[
\sqrt{2\kappa} \cdot \sqrt{\sum_{t =1}^T \nu_{t}^2 + 1} 
+ \frac{2L\kappa}{c_0^2\sqrt{\lambda}}+\frac{2LB\kappa}{\sqrt{c_1d}}
\right]
\]
where $\beta_T$ is defined in~\eqref{eq:beta}, and $c_0, c_1 = \TOM(1)$ are positive constants given in Algorithm~\ref{algo:adap}. %Here, $\widetilde{\OM}(\cdot)$ hides constant factors and logarithmic dependence on $T$.
\end{thm}

Theorem \ref{thm:main-bandit} provides a regret bound in a general form that depends on $\beta_T$. As shown in \eqref{eq:parameters}, $\beta_t$ is a hyperbolic function of the robustification parameter $\tau_0$. Increasing $\tau_0$ decreases the bias term $\OM\left({\kappa}/{\tau_0}\right)$ while increasing the range term $\OM\left( \tau_0\log({2t^2}/{\delta})\right)$. Therefore, choosing $\tau_0$ carefully is essential to achieve the optimal trade-off between unbiasedness and robustness. Setting $\tau_0 = \TOM(\sqrt{d})$ minimizes the right-hand side of~\eqref{eq:beta}. This, combined with Theorem~\ref{thm:main-bandit}, yields the simplified regret bound~\eqref{eq:reg2} in the following corollary.

\begin{cor}\label{cor:tau-regret}
Let $\tau_0 = \max\big\{ \sqrt{2\kappa}, 2\sqrt{d} \big\}/ \sqrt{\log({2T^2}/{\delta})}$ and $\lambda = {d}/{B^2}$, then 
\$
\beta_T 
\le 64\left(2\sqrt{\kappa\log({2T^2}/{\delta})} + \sqrt{d\log({2T^2}/{\delta})}\right) + 5\sqrt{d}.
\$
Consequently, the regret bound in Theorem \ref{thm:main-bandit} reduces to
\begin{equation}\label{eq:reg2}
\mathrm{Reg}(T) 
= \widetilde{\OM}\left(d \sqrt{\sum_{t \in [T]} \nu_t^2} +d \cdot \max\{LB, 1\}\right),
\end{equation}
where $\widetilde{\OM}(\cdot)$ hides constant factors and logarithmic dependence on $T$.
\end{cor}

Corollary~\ref{cor:tau-regret} demonstrates that AdaOFUL achieves the state-of-the-art regret bound in the presence of heavy-tailed rewards, comparable to the case where rewards are uniformly bounded or sub-Gaussian. The regret upper bound in the noiseless case reduces to $\widetilde{\OM}(d)$, and in the noisy case, it reduces to $\widetilde{\OM}\left( d \sqrt{\sum_{t \in [T]} \nu_t^2}
\right)$. In the worst case scenario where $\nu_t = \Theta(1)$ for all $t \ge 1$, the regret bound reduces to $\TOM(d\sqrt{T})$, which matches the worst-case minimax lower bound~\citep{dani2008stochastic}. Hence, our variance-aware regret bound \eqref{eq:reg2} is tighter than the pessimistic worst-case bound $\TOM(d\sqrt{T})$ when $\sum_{t=1}^T \nu_t^2 \ll T$. To the best of our knowledge, such a variance-aware regret bound has only been obtained in the literature for sub-Gaussian rewards~\citep{kirschner2018information} or uniformly bounded rewards~\citep{zhou2022computationally}. We are the first to provide a variance-aware regret bound for heavy-tailed stochastic linear bandits.

%%%%%%%%%%%%%%%%%%%%%%%%%%%%%%%%%%%%%%%%%%%%%%%%%%%%%%%%%%%%
%%%%%%%%%%%%%%%%%%%%%%%%%%VARA%%%%%%%%%%%%%%%%%%%%%%%%%%%%%%
%%%%%%%%%%%%%%%%%%%%%%%%%%%%%%%%%%%%%%%%%%%%%%%%%%%%%%%%%%%%
\section{Variance Aware Regret for Linear MDPs}\label{sec:mdp}

%%%%%%%%%%%%%%%%%%%%%%%%%%%%%%%%%%%%%%%%%%%%%%%%%%%%%%%%%
%%%%%%%%%%%%%%%%%Algorithm Description%%%%%%%%%%%%%%%%%%%
%%%%%%%%%%%%%%%%%%%%%%%%%%%%%%%%%%%%%%%%%%%%%%%%%%%%%%%%%

\subsection{High-level Algorithm Description}\label{sec:algo}

In this section, we present VARA, an algorithm collected in Algorithm~\ref{algo:linear}, that extends AdaOFUL to solve linear MDPs with heavy-tailed rewards.
The VARA algorithm is built on LSVI-UCB++\citep{he2022nearly}, an algorithm proposed recently to achieve minimax optimality for linear MDPs.
LSVI-UCB++\citep{he2022nearly} uses weighted ridge regression, where the weights depend on some proper variance estimators $\sigma_{h, k}$'s. 
The variance estimation techniques in LSVI-UCB++ can also be used to obtain variance-aware regrets. These techniques include (i) separate variance estimation, (ii) monotonicity of value functions, and (iii) rare-switching value function update.

We present a detailed algorithm description in the next subsection, focusing on the differences between VARA and LSVI-UCB++ \citep{he2022nearly}.
To obtain variance-aware regrets under heavy-tailed rewards, we made two improvements to LSVI-UCB++.
First, while LSVI-UCB++ assumes a deterministic, uniformly bounded, and known reward function, we use AdaOFUL to estimate the parameters $\Btheta_{h}^*$ and $\Bpsi_h^*$ for both the expected reward functions and their second-order moments. This complicates the construction of the variance estimators $\sigma_{h, k}$ and requires a more detailed analysis of their impacts on the final regrets (see Lemma \ref{lem:sum-bonus}).
Second, previous works use the Azuma-Hoeffding inequality to analyze the concentration effect in the suboptimality gap, which leads to regret of $\TOM(\sqrt{K})$. Instead, we use a variance-aware Bernstein inequality and produce a much tighter upper bound  of $\TOM(1)$  for the concentration effect(see Lemma~\ref{lem:sub-gap}). We explain the analytical novelty in detail in Appendix~\ref{proof:mdp-regret}. %This yields a concentration effect of $\TOM(1)$ (see Lemma~\ref{lem:sub-gap}). We explain the analytical novelty in detail in Appendix~\ref{proof:mdp-regret}.

\begin{algorithm}[t!]
	\DontPrintSemicolon
	\SetAlgoLined
	\SetNoFillComment
	\SetKwInOut{Require}{Require}
	\SetKwInOut{Ini}{Initialization}
	%\SetSideCommentLeft
	%	\tcc{iterate over all training examples}
	\Require{ $K, H, \HM, \VM, W, \sigma_{R}, \sigma_{R^2},\tau_0, \ttau_0$. }
	\Ini{$\H_{h,0} =\TH_{h,0} = \lambda \I, c_0 = \frac{1}{6\sqrt{3\log\frac{2HK^2}{\delta}}}, c_1 = \frac{1}{42 \cdot \frac{2HK^2}{\delta}}, \lambda = \frac{1}{\HM^2+W^2}, k_{\mathrm{last}} = 1$.}
	\setcounter{AlgoLine}{0}
	\For{episode $k=1$ {\bfseries to} $K$ }{
		$\overoptV_{H+1}^k(\cdot) = \overpesV_{H+1}^k(\cdot) =0$ \;
		\For{round $h=H$ {\bfseries to} $1$ }{  
			\eIf{there exists a stage $h' \in [H]$ such that $\mathrm{det}(\H_{h', k-1}) \ge 2 \mathrm{det}
				(\H_{h', k_{\mathrm{last}}-1})$}{
				$\optQ_h^k(\cdot, \cdot) = \langle \Bphi(\cdot, \cdot), \Btheta_{h, k-1} + \Bmu_{h, k-1}  \BoveroptV_{h+1}^{k} \rangle + \beta \|\Bphi(\cdot, \cdot)\|_{\H_{h, k-1}^{-1}}$. \;
				
				$\pesQ_h^k(\cdot, \cdot) = \langle \Bphi(\cdot, \cdot), \Btheta_{h, k-1} + \Bmu_{h, k-1}  \BoverpesV_{h+1}^{k} \rangle - \beta \|\Bphi(\cdot, \cdot)\|_{\H_{h, k-1}^{-1}}$. \;
				
				$\overoptQ_h^k(\cdot, \cdot) =
				\min \left\{\optQ_h^k(\cdot, \cdot), \overoptQ_h^{k-1}(\cdot, \cdot), \HM \right\}$, $\overpesQ_h^k(\cdot, \cdot) = \max \left\{
				\pesQ_h^k(\cdot, \cdot), \overpesQ_h^{k-1}(\cdot, \cdot),  0 \right\}$.  \;  %\tcp*[f]{Pessimism}  \;
				
				Record the last updating episode $\kl  = k$.
			}{$\overoptQ_h^k(\cdot, \cdot) = \overoptQ_h^{k-1}(\cdot, \cdot), \overpesQ_h^k(\cdot, \cdot)=\overpesQ_h^{k-1}(\cdot, \cdot)$.}
			
			$\overoptV_h^k(\cdot) = \max_{a} \overoptQ_h^k(\cdot, a), \overpesV_h^k(\cdot) =  \max_{a} \overpesQ_h^k(\cdot, a)$.\;
			
			$\pi_h^k(\cdot) \in \argmax_{a} \overoptQ_h^k(\cdot, a)$. \;  
		} 
		Receive the initial state $s_{1, k}$.\;
		
		\For{round $h=1$ {\bfseries to} $H$}{
			Play $a_{h, k}  = \pi_h^k(s_{h, k})$ and observe $r_{h, k} \sim R_h(s_{h, k}, a_{h, k}), s_{h+1, k}\sim \PB(\cdot|s_{h, k}, a_{h, k})$. \;
			
			Observe feature vectors $\Bphi_{h, k} = \Bphi(s_{h, k}, a_{h, k})$ and $\TBphi_{h, k} =\TBphi(s_{h,k}, a_{h, k})$. \;
			
			Set the bonus as $b_{h, k} = \max\{
			{\|\Bphi_{h, k}\|}_{\H_{h, k-1}^{-1}}, {\| \TBphi_{h, k}\|}_{\TH_{h, k-1}^{-1}}
			\}$. \;
			
			Set the estimated variance $\sigma_{h,k}$ as in~\eqref{eq:sigma}. \;
			% \tcp*[f]{Estimate conditional variances} \;
			% Set $w_{h, k} =\sigma_{h, k}^{-1} \left\|\Bphi_{h, k}\right\|_{\H_{h, k-1}^{-1}}$ and $\tw_{h, k} =\sigma_{h, k}^{-1} \left\|\TBphi_{h, k}\right\|_{\TH_{h, k-1}^{-1}}$. \;
			%	 \tcp*[f]{Update sample uncertainty weights} \;
			
			Compute $\Bmu_{h, k}$ via~\eqref{eq:mu_hk}. \;
			
			% \tcp*[f]{For transition parameters} \;
			Compute $\Btheta_{h, k}$ via~\eqref{eq:theta_hk} with $\tau_{h, k} = \tau_0 \sqrt{1+w_{h, k}^2}/{w_{h, k}}$ and $w_{h, k} =\sigma_{h, k}^{-1} \left\|\Bphi_{h, k}\right\|_{\H_{h, k-1}^{-1}}$. \;
			
			% \tcp*[f]{For expetcaton of rewards} \;
			Compute $\Bpsi_{h, k}$ via~\eqref{eq:psi_hk} with $\ttau_{h, k} = \ttau_0 \sqrt{1+\tw_{h, k}^2}/{\tw_{h, k}}$ and $\tw_{h, k} =\sigma_{h, k}^{-1} \left\|\TBphi_{h, k}\right\|_{\TH_{h, k-1}^{-1}}$.  \;
			
			% \tcp*[f]{For 2nd moment of rewards} \;
			Update $\H_{h, k} = \H_{h, k-1}+ \sigma_{h, k}^{-2} \Bphi_{h, k}\Bphi_{h, k}^\top$ and $\TH_{h, k} = \TH_{h, k-1}+ \sigma_{h, k}^{-2} \TBphi_{h, k}\TBphi_{h, k}^\top$.
		}
	}
	\caption{The VARA algorithm}
	\label{algo:linear}
\end{algorithm}
	\vspace{-0.1in}
% the reward estimation, transition estimation, the deompose-then-merge approach, monotoniciy, and variance sdf

%%%%%%%%%%%%%%%%%%%%%%%%%%%%%%%%%%%%%%%
%%%%%%%%Detailed description%%%%%%%%%%%
%%%%%%%%%%%%%%%%%%%%%%%%%%%%%%%%%%%%%%%
\subsection{Detailed Algorithm Description}
For each episode $k$, we  perform optimistic value iterations (Lines 3-11), compute the greedy policy $\pi_{h}^k$ with respective to the pessimistic value function $\overoptQ_h^k$ (Line 12), and  then execute it to collect a new trajectory of data (Lines 16-17).
The rest of Algorithm~\ref{algo:linear} updates maintained estimators, including 
the conditional variances $\sigma_{h, k}^2$  (Line 19),  the transition parameters $\Bmu_{h, k}$ (Line 20), the reward parameters $\Btheta_{h, k}, \Bpsi_{h, k}$ (Lines 21-22), and the Hessian matrices $\H_{h,k}, \TH_{h, k}$ (Line 23).
In what follows, we discuss in detail the key steps of Algorithm \ref{algo:linear} in more detail.

\paragraph{Reward estimation}
Since rewards are collected in an adaptive manner and have only finite second moments, we use the same strategy adopted in AdaOFUL to estimate $\Btheta_h^*$:
\begin{equation}\label{eq:theta_hk}
	\est_{h, k}  
	:= \argmin_{\Btheta \in \mathrm{Ball}_d(W)}  \left\{ %L_{h, k}^{(R)}(\est) \quad \text{with} \quad 
	L_{h, k}^{(R)}(\est):=\frac{\lambda}{2}\| \Btheta\|^2 + \sum_{j=1}^{k}
	\ell_{\tau_{h, j}}\left(  \frac{r_{h, j} -  \langle \Bphi_{h, j}, \Btheta \rangle }{\sigma_{h, j}}  \right) \right\}.
\end{equation}
Following the spirit of Theorem~\ref{thm:heavy}, we set $\tau_0 = \TOM(\sqrt{d})$ with its detailed expression provided in~\eqref{eq:tau} of the online supplement. 
% \scolor{As we shall see later, $\sigma_{h, k}^2 \ge [\VB_hR_h](s_{h, k}, a_{h, k})$ holds with high probability.
	% Conditioning on this event,  Lemma~\ref{lem:CI-rewards} shows $\|\Btheta_{h, k} - \Btheta_{h}^*\|_{\H_{h,k}} \le \beta_R = \TOM(\sqrt{d})$ holds uniformly for all $h \in [H]$ and $k \in [K]$ with high probability.}

\paragraph{Transition estimation}

Let $\Bdelta(s) \in \RB^{|\SM|}$ be a one-hot vector that is zero everywhere except for the entry corresponding to the state $s$, which is one. We define $\Beps_{h, k} = \PB_h(\cdot|s_{h, k}, a_{h, k}) - \Bdelta(s_{h+1, k})$. As $\EB[\Beps_{h, k}|\FM_{h, k}] = \0$, $\Bdelta(s_{h+1, k})$ is an unbiased estimator of $\PB_h(\cdot|s_{h, k}, a_{h, k}) = \Bmu_h^\top \Bphi(s_{h, k}, a_{h, k}) = \Bmu_h^\top \Bphi_{h, k}$. Thus, we can learn $\Bmu_h$ by regressing $\Bdelta(s_{h+1, k})$ on $\Bphi_{h, k}:= \ph(s_{h, k}, a_{h, k})$:
\begin{equation}\label{eq:mu_hk}
	\Bmu_{h, k} := \argmin_{\Bmu \in \RB^{d \times |\SM|}} \left\{ L_{h, k}^{(P)}(\Bmu):= 
	\frac{\lambda}{2}\| \Bmu\|_F^2 + \sum_{j=1}^{k}
	\left\|  \frac{\Bmu_h^\top \Bphi_{h, k}-\Bdelta(s_{h+1, j})}{\sigma_{h, j}} \right\|^2\right\}
\end{equation}
where $\|\cdot\|_F$ denotes the Frobenius norm. This problem admits a closed-form solution given by $\Bmu_{h, k}  = \H_{h, k}^{-1} \sum_{j=1}^k \sigma_{h, j}^{-2} \Bphi_{h, j} \Bdelta(s_{h+1, j})^\top$.

\paragraph{Variance estimation for rewards}
In linear MDPs, estimating the variance of the reward $R_h(s, a)$ is straightforward.
Since $\PB_h R_h^2(s, a) =\langle\TBphi(s, a), \Bpsi_h^*\rangle$, we estimate $\Bpsi_h^*$ by 
\begin{equation}
	\label{eq:psi_hk}
	\Bpsi_{h, k}  := \argmin_{ \Bpsi \in \mathrm{Ball}_d(W) } \left\{  L_{h, k}^{(R^2)}(\Bpsi):=
	\frac{\lambda}{2}\| \Bpsi\|^2 + \sum_{j=1}^{k}
	\ell_{\ttau_{h, j}}\left(  \frac{r_{h, j}^2 -  \langle \TBphi_{h, j}, \Bpsi \rangle }{\sigma_{h, j}}  \right)\right\}
\end{equation}
where $\ttau_{h, k} = \ttau_0 \sqrt{1+\tw_{h, k}^2}/\tw_{h, k}$ is the corresponding robustification parameter and $\tw_{h, k} = \|\TBphi_{h, k}\|_{\H_{h,k-1}^{-1}}$ is the importance weight.
We then estimate $[\VB_h R_h](s_h, a_h)$ by
	\begin{equation}
		\label{eq:variance-R}
		[\HVB_h R_h](s_{h, k}, a_{h, k}) = \langle \TBphi_{h, k}, \Bpsi_{h, k-1}  \rangle -  \left[
		\langle \Bphi_{h, k}, \est_{h, k-1} \rangle_{[0, \HM]} 
		\right]^2.
	\end{equation}

	\paragraph{Variance estimation}
 Inspired by~\citet{hu2022nearly}, we set the variance estimator $\sigma_{h,k}$ to be
	\begin{equation}
		\label{eq:sigma}
		\sigma_{h,k}^2 = \max\left\{  \sigma_{\min}^2,\,  
		d^3H \cdot E_{h, k},\, J_{h, k},\, 
		c_0^{-2} b_{h, k}^2,\, \left( \frac{W}{\sqrt{c_1 d}} + \HM d^{2.5} H\right) b_{h, k}
		\right\}
	\end{equation}
	where $\sigma_{\min}$ is a small positive constant to avoid singularity, $b_{h, k} = \max\{{\|\Bphi_{h, k}\|}_{\H_{h, k-1}^{-1}}, {\| \TBphi_{h, k}\|}_{\TH_{h, k-1}^{-1}}\}$ is the bonus term, $E_{h, k}$ and $J_{h, k}$ 
	are defined as
 \begin{gather}
\label{eq:J}
J_{h, k} = [\HVB_{h, k}R_h+ \HVB_{h, k} \overoptV_{h+1}^k](s_{h, k}, a_{h, k}) + R_{h, k}  + U_{h, k},\\  
E_{h, k} =\min\left\{
\HM^2, 2\HM\beta_0 \cdot \|\Bphi_{h, k}\|_{\H_{h, k-1}^{-1}}
+ \HM \cdot	\left[ \HPB_{h, k}(\overoptV_{h+1}^k - \overpesV_{h+1}^k)\right](s_{h, k}, a_{h, k})
\right\}, \label{eq:E}
 \end{gather}
in which $\beta_0 = \TOM\left(\sigma_{\min}^{-1}{\HM}\sqrt{d^3H}\right)$ is an initial exploration radius,  $\widehat{\PB}_{h, k} (\cdot| s, a) = \Bmu_{h, k-1}^\top \Bphi(s, a)$ is the empirical transition kernel at the $h$-th round and $k$-the episode,  $\HVB_h(\cdot)$ the empirical variance operator defined in~\eqref{eq:variance-R}, and $R_{h, k}, U_{h, k}$ are defined as
 \begin{gather}
    R_{h, k} :=\beta_{R^2}  \| \TBphi_{h, k}\|_{\TH_{h, k-1}^{k-1}}  + 2\HM \beta_R \|\Bphi_{h, k}\|_{\H_{h, k-1}^{-1}},\label{eq:R}\\
    U_{h, k} = \min\left\{
    \VM^2, 11\HM \beta_0\cdot \|\Bphi_{h, k}\|_{\H_{h, k-1}^{-1}} + 4\HM \cdot \HPB_{h, k}(\overoptV_{h+1}^k -\overpesV_{h+1}^k)(s_{h, k}, a_{h, k}) 
    \right\}, \label{eq:U} 
 \end{gather}
with $\beta_R = \TOM(\sqrt{d}), \beta_{R^2} = \TOM\left(\sqrt{d} +\sqrt{d} \frac{\sigma_{R^2}}{\sigma_{\min}} \right)$ being two  initial exploration radiuses.
 In Appendix~\ref{append:variance}, we explain in detail why  $\sigma_{h,k}$'s are taken in the above way.
	
\subsection{Regret Analysis}
This section presents the statistical, space, and computational complexities of Algorithm~\ref{algo:linear}.
%In particular, Algorithm~\ref{algo:linear} is the first algorithm that achieves a variance-aware regret in linear MDPs with heavy-tailed rewards, while both the space and computational complexities are no worse than those prior works.
\begin{thm}
    \label{thm:mdp}
    Consider a linear MDP satisfying Asumption~\ref{asmp:realizable} and~\ref{asmp:H-V}.
    For any $\delta \in (0, 1)$, with probability at least $1-10\delta$, Algorithm~\ref{algo:linear} achieves 
    \begin{align}
    \label{eq:MDP-regret}
    \begin{split}
         &  \sum_{k=1}^K (V_1^*-V_1^{\pi_k})(s_{1, k})=\TOM\left(
    	d\sqrt{H K  \GM^*}   + H  d \sqrt{K} \sigma_{\min}\right)
     \\& + \TOM\left(\frac{H^{2.5}d^{6}\HM^2 + Hd^2 \sigma_{R^2}}{\sigma_{\min}} +  H^3 d^{5} \HM +  H d \sigma_{R} + Hd^2
    	\right), 
    \end{split}
    \end{align}
    where $\sigma_{\min}$ is a manually set lower bound for all variance estimators  $\sigma_{h,k}$'s, 
    \begin{equation}
    \label{eq:G}
    \GM^*= \min \left\{
    \sum_{h=1}^H \sum_{(s, a)}  \widetilde{d}_h^K(s, a) \cdot [\VB_h R_h + \VB_h V_{h+1}^{*}](s, a),
    \VM^2 
    \right\},
\end{equation}
\begin{equation}
\label{eq:dh}
    \widetilde{d}_h^K(s, a) = \frac{1}{K}\sum_{k=1}^K d_h^{\pi_k}(s, a),
\end{equation}
and $d_h^{\pi_k}(s, a) = \PB^{\pi_k}(s_h=s, a_h=a|s_0=s_{1,k})$ is the probability of reaching $(s_{h,k}, a_{h, k}) = (s, a)$ at the $h$-th step when the agent starts from $s_{1, k}$ and follows the policy $\pi_k$.
\end{thm}

\paragraph{Trade-off by $\sigma_{\min}$}

Theorem~\ref{thm:mdp} reveals a trade-off that arises from the choice of $\sigma_{\min}$.
The second term in~\eqref{eq:MDP-regret} is positively dependent on $\sigma_{\min}$ due to the sum of bonuses $\sum_{k=1}^K\sum_{h=1}^H b_{h, k}$, while the third term is negatively dependent on $\sigma_{\min}$ due to the initial exploration radiuses $\beta_{R^2}, \beta_0$ being linearly dependent on ${1}/{\sigma_{\min}}$. 
A larger $\sigma_{\min}$ leads to smaller initial exploration radiuses and higher accuracies of the confidence sets used in variance estimation, resulting in a decrease in regret.
However, it also leads to larger accumulated bonuses, which can be seen as a cost for exploration, and increases regret.
The choice of $\sigma_{\min}$ must balance these opposing effects.
Corollary~\ref{cor:initial} implies that choosing the optimal $\sigma_{\min}^* = \sqrt{H^{1.5}d^{5}\HM^2 + d \sigma_{R^2}} \cdot K^{-\frac{1}{4}}$ yields a regret barrier of $\TOM\left( H d \cdot \sqrt{d^5 \HM^2 + d \sigma_{R^2}} \cdot \sqrt[4]{K}\right)$.
When $K$ is sufficiently large, the regret bound in~\eqref{eq:sim-regret} can be further simplified to $\TOM\left(d\sqrt{H \GM^* K}\right)$.
To the best of our knowledge, Theorem~\ref{thm:mdp} is the first to derive the variance-aware regret for linear MDPs, especially with heavy-tailed rewards.

 \begin{cor}
 \label{cor:initial}
 Under the same setting of Theorem~\ref{thm:mdp}, if we set $\sigma_{\min} = \sqrt{H^{1.5}d^{5}\HM^2 + d \sigma_{R^2}} \cdot K^{-\frac{1}{4}}$, the regret of VARA is bounded by
 \begin{equation}
 	 	 		\label{eq:sim-regret}
 \begin{gathered}
 	\TOM\left(
 	d\sqrt{H \GM^* K} 
 	+  H d \sqrt{d^5\HM^2 + d\sigma_{R^2}}  \cdot \sqrt[4]{K} +
 	 H^3 d^{5} \HM
 	+H d \sigma_{R} + Hd^2
 	\right).
 \end{gathered}
 \end{equation}
 \end{cor}

\paragraph{Instance-dependent quantity $\GM^*$}
The quantity $\GM^*$ is given by~\eqref{eq:G}. Firstly, it is bounded above by $\VM^2$ in Assumption~\ref{asmp:H-V}, which sets an upper bound on the variance of the sum of random rewards received when following any policy. Secondly, $\GM^*$ is no greater than the sum of per-round conditional variances $[\VB_h R_h + \VB_h V_{h+1}^{*}](s, a)$, weighted by an averaged occupancy measure $\widetilde{d}h^K(s, a) := \frac{1}{K}\sum{k=1}^K d_h^{\pi_k}(s, a)$. The function $\widetilde{d}_h^K(\cdot, \cdot)$ introduces a probability measure on $\SM \times \AM$ for any fixed $h \in [H]$, in accordance with the definition of $d_h^{\pi}$.

\subsection{Other Instance-dependent Regrets}

Our variance-aware regret has two key features. Firstly, we do not require any prior knowledge of $\GM^*$ to achieve variance awareness. Secondly, the additional conditions imposed on the MDP structure lead to other instance-dependent regrets. In the following, we also impose Assumption~\ref{asmp:bounded-R} for a fair comparison with related work. However, we would like to emphasize that all of our results are obtained in the presence of heavy-tailed rewards.

\begin{asmp}
\label{asmp:bounded-R}
We assume uniformly bounded rewards where $0 \le R_h(s, a) \le 1$ for all $h \in [H]$ and $(s, a) \in \SM \times \AM$.
\end{asmp}

\paragraph{Worst-case regret}
Under Assumption~\ref{asmp:bounded-R}, $\VM^2 = H^2$ according to the law of total variance~\citep{azar2013minimax}. Consequently, we can infer that $\GM^* \le H^2$, and the regret reduces to the minimax optimal $\TOM(dH\sqrt{HK})$~\citep{he2022nearly}. The authors achieved this regret by directly setting $\sigma_{\min} = 1/H$, without taking into account the trade-off introduced by $\sigma_{\min}$. Although this was sufficient for their worst-case scenario, it was not suitable for our goal of achieving variance awareness. If we also set $\sigma_{\min} = 1/H$, the second term in \eqref{eq:MDP-regret} becomes $d\sqrt{K}$, and we cannot determine the dominant term between $d\sqrt{HK \GM^*}$ and $d\sqrt{K}$. However, once we balance the trade-off of $\sigma_{\min}$, the second term becomes much smaller, making $\TOM(d\sqrt{H\GM^*K})$ the dominant term.

\paragraph{Range-dependent regret} 
Let $\SM_{s, a}$ be the set of immediate successor states after one transition from state $s$ upon taking action $a$, which is also the support set of $\PB(\cdot|s, a)$. Define $\Phi_{\mathrm{succ}}$ as the maximum value function range when restricted to the immediate successor states: 
\[
\Phi_{\mathrm{succ}} := \sup_{h \in [H]} \sup_{(s, a)} [ \sup_{s' \in \SM_{s, a}} V_{h+1}^*(s') - \inf_{s' \in \SM_{s, a}} V_{h+1}^*(s')]. 
\]
Since the variance is upper bounded by one-fourth of the square range of a random
variable, we have 
\[
\sup_{h \in [H]}\sup_{(s, a)}[\VB_h V_{h+1}^*](s, a) \le \frac{1}{4} \Phi_{\mathrm{succ}}^2
~\text{and thus}~\GM^* \le H (\sigma_{R}^2 +  \Phi_{\mathrm{succ}}^2).
\]
Therefore, our regret reduces to $\TOM\left(dH \sqrt{(\sigma_{R}^2 +  \Phi_{\mathrm{succ}}^2) K}\right)$.
It is worth noting that similar range-dependent regrets have been derived for tabular MDPs with bounded rewards~\citep{bartlett2009regal,fruit2018efficient,zanette2019tighter}, but to the best of our knowledge, we obtain the first such result for linear MDPs with heavy-tailed rewards.

 \paragraph{First-order regret} 
 The first-order regret that scales proportionally to $V_1^*$, where $V_1^*:= V_1^*(s_{1})$ is the value of the optimal value policy at the initial state $s_1$, has been studied for tabular MDPs~\citep{jin2020reward} and linear MDPs~\citep{wagenmaker2021first}.\footnote{They assume all episodes start from the same initial state so that $s_{1, k} \equiv s_1$. However, our regret can be easily extended to the setting where initial states are different, in which case  one should replace $V_1^*K$ with $\sum_{k=1}^K V_1^*(s_{1, k})$ in the regret upper bound.}
However, under Assumption~\ref{asmp:bounded-R}, the corresponding instance-dependent quantity $H^2V_1^*$ can be much larger than $\GM^*$. This is because
\begin{align*}
	\GM^* 
	&\le \sum_{h=1}^H \sum_{(s, a)} \widetilde{d}_h^k(s, a) \cdot [\VB_h R_h + \VB_h V_{h+1}^{*}](s, a) \\
	&\overset{(a)}{\le} H \sum_{h=1}^H \sum_{(s, a)} \widetilde{d}_h^k(s, a) \cdot[r_h +  \PB_h V_{h+1}^{*}](s, a)\\
	&\overset{(b)}{\le} H \sum_{h=1}^H \sum_{(s, a)} \widetilde{d}_h^k(s, a) \cdot V_h^*(s)\overset{(c)}{\le}  H^2 V_1^*
\end{align*}
where $(a)$ uses $0 \le R_h \le 1$ and $0 \le V_{h+1}^* \le H-1$ under Assumption~\ref{asmp:bounded-R}, $(b)$ uses the optimality condition $V_h^*(s) = [r_h + \PB_h V_{h+1}^*](s, a)$ for any $(s, a) \in \SM \times \AM$, and $(c)$ uses $V_{h+1}^*(s) \le V_h^*(s)$ for any $h \in [H]$ and $s \in \SM$ and $\sum_{a \in \SM}\widetilde{d}_h^k(s_1, a) = 1$ since each episode starts at a fixed state $s_1$.
Moreover, even replacing $\cG^*$ with the coarse upper bound $H^2V_1^*$, our regret bound becomes $\TOM(\sqrt{d^2H^3V_1^*K})$, which has a better dependence on $d$ than $\TOM(\sqrt{d^3H^3V_1^*K})$ in~\citep{wagenmaker2021first}.

\paragraph{Concentrability-dependent regret} 
Let $R_{\pi^*}$ denote the sum of random rewards collected in a trajectory following the optimal policy $\pi^*$.
It is straightforward  to see that $\Var(R_{\pi^*})=\sum_{h=1}^H \sum_{(s, a)} d_h^{\pi^*} (s, a) \cdot [\VB_h R_h + \VB_h V_{h+1}^{*}](s, a)$.
Since $\GM^* \le \sup_{\pi} \sum_{h=1}^H \sum_{(s, a)} d_h^{\pi} (s, a) \cdot [\VB_h R_h + \VB_h V_{h+1}^{*}](s, a)$, we can show that $\GM^* \le  C^\dagger \cdot \Var(R_{\pi^*})$ where $C^\dagger$ is a data coverage measure defined as
\[
C^{\dagger} := \sup_{\pi} \frac{\sum_{h=1}^H \sum_{(s, a)} d_h^{\pi} (s, a) \cdot [\VB_h R_h + \VB_h V_{h+1}^{*}](s, a)}{\sum_{h=1}^H \sum_{(s, a)} d_h^{\pi^*} (s, a) \cdot [\VB_h R_h + \VB_h V_{h+1}^{*}](s, a)}.
\]
Therefore, our regret reduces to $\TOM(d\sqrt{C^\dagger \Var(R_{\pi^*})HK})$ given $C^\dagger < \infty$.
The $C^\dagger$ is a counterpart of the generalized concentrability coefficient which quantifies the effect of the distribution shift in offline RL~\citep{chen2019information,xie2021bellman,cheng2022adversarially}.
Specifically, given a data distribution $\mu = \{ \mu_h\}_{h \in [H]}$, the generalized concentrability coefficient $\mathfrak{C}_{\text {conc}}(\mu)$ for linear function approximation is defined as
\[
\mathfrak{C}_{\text {conc}}(\mu) := \sup_{\mathrm{linear} \ Q} \sup_{\pi} \frac{\sum_{h=1}^H \sum_{(s, a)} d_h^{\pi} (s, a) \cdot [ \mathbb{B}_h Q_h](s, a)}{\sum_{h=1}^H \sum_{(s, a)}  \mu_h(s, a) \cdot [ \mathbb{B}_h Q_h](s, a)}, 
\]
where $Q = \{ Q_h\}_{h \in [H]}$ is any Q-value function that is linear in features $\Bphi$ and $[\mathbb{B}_h Q_h](s, a) :=[Q_{h} -r_h- \PB_h V_{h+1}]^2(s, a)$ is the Bellman residual with $V_{h+1}(\cdot) = \sup_{a \in \AM} Q_{h+1}(\cdot, a)$ being the corresponding value function.

\subsection{Space and Computational Complexities}

\begin{thm}[Space and computational complexity]
\label{thm:space}
Solving a $H$-horizon finite MDP in $K$ episodes, VARA takes $\OM(d^3H^2 + d|\AM|HK)$ space and has a running time of $\TOM(d^4|\AM|H^3K + HK(d + H^{-3/4} d^{-3/2}K^{3/4}))$.
\end{thm}

On one hand, VARA achieves the same space complexity as~LSVI-UCB++, but is slightly worse than the original LSVI-UCB~\citep{jin2020provably} that needs $\OM(d^2H + d|\AM|HK)$ space. 
This is because the technique of monotone value function update requires remembering at most $\TOM(dH)$ latest value functions, incurring a slightly worse dependence on $d$ and $H$.
On the other hand, the computational complexity of VARA $\TOM(d^4|\AM|H^3K + HK(d + H^{-3/4} d^{-3/2}K^{3/4}))$ is slightly worse than~LSVI-UCB++'s $\TOM(d^4|\AM|H^3K)$ in terms of the dependence on $K$.
 This is because the adaptive Huber regression estimator does not have a closed-form solution, and we assume we use the Nesterov accelerated method as a solver, which incurs a slightly larger computational complexity.
 However, VARA's computational complexity is better than LSVI-UCB's $\TOM(d^2|\AM|HK^2)$ due to the rare-switching mechanism.
 
\section{Related Work}
\label{sec:related}

\paragraph{Heavy-tailed rewards in online decision making}
On one hand, there exists a large body of work that studies heavy-tailed feedbacks in multi-arm bandits, including deterministic \citep{vakili2013deterministic} and non-deterministic settings \citep{bubeck2013bandits,carpentier2014extreme,lattimore2017scale,bhatt2022nearly}. To handle heavy-tailed rewards, robust mean estimation methods such as median of means and truncation have been applied to linear bandits \citep{medina2016no,shao2018almost,lu2019optimal,xue2021nearly}, achieving the minimax optimal $\widetilde{\OM}(d\sqrt{T})$ regret.
On the other hand, there is a lack of RL algorithms designed to handle heavy-tailed rewards. One exception is~\citep{zhuang2021no}, which modifies UCRL2 and Q-Learning by using truncated rewards and achieves minimax optimal regret in tabular MDPs. However, none of these methods for linear bandits or MDPs provide variance-aware regrets, even if variance information is available. Moreover, simple truncation methods are not optimal in the noiseless setting. 

%\vspace{-0.1in}

\paragraph{Variance-aware regret for linear bandits}

A weighted ridge regression-based algorithm proposed by \citet{kirschner2018information} achieves the same regret~\eqref{eq:reg2} by assuming each $\varepsilon_t$ is $\nu_t$-sub-Gaussian. More recently, \citet{zhou2022computationally} obtained the same regret assuming each $\varepsilon_t$ is uniformly bounded and has finite conditional variance $\nu_t^2$. In the case where the information of conditional variances $\{\nu_t\}_{t \ge 0}$  is unknown, \cite{zhang2021improved} and \cite{kim2021improved}   achieved regret bounds that involve sub-optimal dependence on $d$. All of these works consider light-tailed noises, which are either sub-Gaussian or uniformly bounded.

\paragraph{Robust approach to instance-dependent bounds}

Recent research explores the robust mean estimation approach to obtain instance-dependent regrets, leveraging the observation that robust estimators can achieve estimation errors that only depend on the noise scale. Such estimators often have better theoretical guarantees than non-robust ones, whose estimation errors additionally depend on the range of the problem noise. For instance, \citet{pananjady2020instance} use the median-of-means technique~\citep{lecue2020robust} to achieve local minimax optimality that depends on the standard deviations of the optimal value function and random rewards for synchronous tabular MDPs. In linear bandits, \citet{wagenmaker2021first} use Catoni's estimator~\citep{catoni2012challenging} to estimate the mean of $\v^\top \H_T^{-1} \Bphi_t y_t/\sigma_t^2$ for a fixed unit-norm vector $\v$. In contrast, we modify the adaptive Huber regression to estimate $\Btheta^*$ directly. This difference makes their bounds depend on the second moments of $y_t$'s, while ours only relies on their variances. Moreover, all of these works, except ours, still assume light-tailed rewards.

\paragraph{Variance-aware quantity $\GM^*$}

In the context of online episodic MDPs, \citet{zanette2019tighter} first derived a variance-aware regret bound in the tabular setting with uniformly bounded rewards. Their model-based algorithm, Euler, achieves a regret that can be bounded by either $\TOM(\sqrt{\mathbb{Q}^SAHK})$ or $\TOM(\sqrt{\GM^2 SAK})$, where $\mathbb{Q}^* = \max_{(s, a, h)} (\VB_hR_h+ \VB_h V_{h+1}^*)(s, a) $ is the maximum per-round conditional variance and $\GM$ is a deterministic upper bound on the maximum attainable reward on a single trajectory for any policy $\pi$, such that $\sum_{h=1}^H R_h(s_h, \pi(s_h)) \le \GM$.
One can show that our instance-dependent quantity $\GM^*$ is smaller than $\min\{ H \mathbb{Q}^*, \GM^2 \}$ therein. 
Later,~\citet{jin2020reward} adopted  a  modified analysis of Euler  to obtain the regret bound $\TOM(\sqrt{SAH^3V_1^* K})$ with $V_1^* = V_1^*(s_1)$.\footnote{Unlike our setting, they assume all initial states are the same, denoted as $s_1$. 
Furthermore, the original regret $\TOM(\sqrt{SAH  \cdot V_1^* K})$ by~\citet{jin2020reward} was derived for an MDP where the reward function equals to one deterministically only at a single $(h, s)$ pair.
In this way, they have $0 \le V_1^* \le 1$.
To convert it in the considered setting where $0 \le V_1^* \le H$, an additional factor of $H$ should be multiplied to their regret.
}
In linear MDPs, there are several recent works on obtaining regret bounds for model-free algorithms. For example, \citet{wagenmaker2021first} proposed an optimistic algorithm with a regret bound that scales as $\widetilde{\OM}(\sqrt{d^3H^3V_1^* K})$. However, this algorithm is computationally inefficient. A computationally efficient alternative suffers from a slightly worse regret $\widetilde{\OM}(\sqrt{d^4H^3V_1^* K})$.
All of these works utilize the instance-dependent quantity $H^2 V_1^*$ (assuming $\HM = H$). However, as we argued, our proposed quantity $\GM^*$ is smaller than $H^2V_1^*$, which implies that our algorithm may achieve better performance than these previous works.

\paragraph{Other instance-dependent bounds}
In the infinite-horizon setting, \citet{pananjady2020instance,khamaru2021temporal,li2021polyak} provided variance-aware sample complexities for Q-Learning and its variants in tabular MDPs, given a generative model that produces independent samples for all state-action pairs in every round.
Variance-aware performance guarantees have also been established for offline RL optimization~\citep{yin2021towards}, off-policy evaluation~\citep{min2021variance}, stochastic approximation~\citep{mou2020linear,mou2022optimal}. 
Another approach to instance-dependence bounds focuses on the minimum suboptimality gap, which is the minimum gap between the best and second-best actions over all states~\citep{he2021logarithmic, wagenmaker2022beyond, wagenmaker2022instance, dong2022asymptotic}. However, due to the differences in the settings, we cannot make a meaningful comparison between these bounds and ours.

\section{Conclusion}
\label{sec:conclusion}

This paper introduces two new algorithms, AdaOFUL for linear bandits and VARA for linear MDPs, both of which are modifications of the original adaptive Huber regression and designed to handle online sequential decision-making. With only the assumption of bounded reward variances, our algorithms achieve either state-of-the-art or tighter variance-aware regrets. Additionally, in linear MDPs, the instance-dependent quantity $\GM^*$ can be bounded by other instance-dependent quantities when additional structure assumptions are available. Our modified adaptive Huber regression can be a useful building block for algorithm design in online problems with heavy-tailed rewards.

\bibliography{bib/Qlearning,bib/stat,bib/robust,bib/lmdp}

\begin{thebibliography}{71}
\providecommand{\natexlab}[1]{#1}
\providecommand{\url}[1]{\texttt{#1}}
\expandafter\ifx\csname urlstyle\endcsname\relax
  \providecommand{\doi}[1]{doi: #1}\else
  \providecommand{\doi}{doi: \begingroup \urlstyle{rm}\Url}\fi

\bibitem[Abbasi-Yadkori et~al.(2011)Abbasi-Yadkori, P{\'a}l, and
  Szepesv{\'a}ri]{abbasi2011improved}
Yasin Abbasi-Yadkori, D{\'a}vid P{\'a}l, and Csaba Szepesv{\'a}ri.
\newblock Improved algorithms for linear stochastic bandits.
\newblock In \emph{Advances in Neural Information Processing Systems},
  volume~24, 2011.

\bibitem[Arnosti et~al.(2016)Arnosti, Beck, and Milgrom]{arnosti2016adverse}
Nick Arnosti, Marissa Beck, and Paul Milgrom.
\newblock Adverse selection and auction design for internet display
  advertising.
\newblock \emph{American Economic Review}, 106\penalty0 (10):\penalty0
  2852--66, 2016.

\bibitem[Ayoub et~al.(2020)Ayoub, Jia, Szepesvari, Wang, and
  Yang]{ayoub2020model}
Alex Ayoub, Zeyu Jia, Csaba Szepesvari, Mengdi Wang, and Lin Yang.
\newblock Model-based reinforcement learning with value-targeted regression.
\newblock In \emph{International Conference on Machine Learning}, pages
  463--474. PMLR, 2020.

\bibitem[Azar et~al.(2013)Azar, Munos, and Kappen]{azar2013minimax}
Mohammad~Gheshlaghi Azar, R{\'e}mi Munos, and Hilbert~J Kappen.
\newblock Minimax {PAC} bounds on the sample complexity of reinforcement
  learning with a generative model.
\newblock \emph{Machine Learning}, 91\penalty0 (3):\penalty0 325--349, 2013.

\bibitem[Azar et~al.(2017)Azar, Osband, and Munos]{azar2017minimax}
Mohammad~Gheshlaghi Azar, Ian Osband, and R{\'e}mi Munos.
\newblock Minimax regret bounds for reinforcement learning.
\newblock In \emph{International Conference on Machine Learning}, pages
  263--272, 2017.

\bibitem[Bartlett and Tewari(2009)]{bartlett2009regal}
Peter Bartlett and Ambuj Tewari.
\newblock Regal: a regularization based algorithm for reinforcement learning in
  weakly communicating mdps.
\newblock In \emph{Conference on Uncertainty in Artificial Intelligence}, pages
  35--42. AUAI Press, 2009.

\bibitem[Bhatt et~al.(2022)Bhatt, Fang, Li, and Samorodnitsky]{bhatt2022nearly}
Sujay Bhatt, Guanhua Fang, Ping Li, and Gennady Samorodnitsky.
\newblock Nearly optimal {Catoni’s M}-estimator for infinite variance.
\newblock In \emph{International Conference on Machine Learning}, pages
  1925--1944, 2022.

\bibitem[Bradtke and Barto(1996)]{bradtke1996linear}
Steven~J Bradtke and Andrew~G Barto.
\newblock Linear least-squares algorithms for temporal difference learning.
\newblock \emph{Machine Learning}, 22\penalty0 (1):\penalty0 33--57, 1996.

\bibitem[Bubeck et~al.(2013)Bubeck, Cesa-Bianchi, and
  Lugosi]{bubeck2013bandits}
S{\'e}bastien Bubeck, Nicolo Cesa-Bianchi, and G{\'a}bor Lugosi.
\newblock Bandits with heavy tail.
\newblock \emph{IEEE Transactions on Information Theory}, 59\penalty0
  (11):\penalty0 7711--7717, 2013.

\bibitem[Bubeck et~al.(2015)]{bubeck2015convex}
S{\'e}bastien Bubeck et~al.
\newblock Convex optimization: Algorithms and complexity.
\newblock \emph{Foundations and Trends{\textregistered} in Machine Learning},
  8\penalty0 (3-4):\penalty0 231--357, 2015.

\bibitem[Carpentier and Valko(2014)]{carpentier2014extreme}
Alexandra Carpentier and Michal Valko.
\newblock Extreme bandits.
\newblock In \emph{Advances in Neural Information Processing Systems}, 2014.

\bibitem[Catoni(2012)]{catoni2012challenging}
Olivier Catoni.
\newblock Challenging the empirical mean and empirical variance: A deviation
  study.
\newblock \emph{Annales de l'IHP Probabilit{\'e}s et statistiques}, 48\penalty0
  (4):\penalty0 1148--1185, 2012.

\bibitem[Chen and Jiang(2019)]{chen2019information}
Jinglin Chen and Nan Jiang.
\newblock Information-theoretic considerations in batch reinforcement learning.
\newblock In \emph{International Conference on Machine Learning}, pages
  1042--1051. PMLR, 2019.

\bibitem[Cheng et~al.(2022)Cheng, Xie, Jiang, and
  Agarwal]{cheng2022adversarially}
Ching-An Cheng, Tengyang Xie, Nan Jiang, and Alekh Agarwal.
\newblock Bellman-consistent pessimism for offline reinforcement learning.
\newblock In \emph{International Conference on Machine Learning}, volume 162,
  pages 3852--3878, 2022.

\bibitem[Cont(2001)]{cont2001empirical}
Rama Cont.
\newblock Empirical properties of asset returns: {S}tylized facts and
  statistical issues.
\newblock \emph{Quantitative {F}inance}, 1\penalty0 (2):\penalty0 223, 2001.

\bibitem[Dani et~al.(2008)Dani, Hayes, and Kakade]{dani2008stochastic}
Varsha Dani, Thomas~P Hayes, and Sham~M Kakade.
\newblock Stochastic linear optimization under bandit feedback.
\newblock 2008.

\bibitem[Dann et~al.(2019)Dann, Li, Wei, and Brunskill]{dann2019policy}
Christoph Dann, Lihong Li, Wei Wei, and Emma Brunskill.
\newblock Policy certificates: Towards accountable reinforcement learning.
\newblock In \emph{International Conference on Machine Learning}, pages
  1507--1516, 2019.

\bibitem[Dong and Ma(2022)]{dong2022asymptotic}
Kefan Dong and Tengyu Ma.
\newblock Asymptotic instance-optimal algorithms for interactive decision
  making.
\newblock \emph{arXiv preprint arXiv:2206.02326}, 2022.

\bibitem[Freedman(1975)]{freedman1975tail}
David~A Freedman.
\newblock On tail probabilities for martingales.
\newblock \emph{The Annals of Probability}, 3\penalty0 (1):\penalty0 100--118,
  1975.

\bibitem[Fruit et~al.(2018)Fruit, Pirotta, Lazaric, and
  Ortner]{fruit2018efficient}
Ronan Fruit, Matteo Pirotta, Alessandro Lazaric, and Ronald Ortner.
\newblock Efficient bias-span-constrained exploration-exploitation in
  reinforcement learning.
\newblock In \emph{International Conference on Machine Learning}, pages
  1578--1586, 2018.

\bibitem[Hastie et~al.(2009)Hastie, Tibshirani, and
  Friedman]{hastie2009elements}
Trevor Hastie, Robert Tibshirani, and Jerome~H Friedman.
\newblock \emph{The Elements of Statistical Learning: Data Mining, Inference,
  and Prediction}.
\newblock Springer, 2009.

\bibitem[He et~al.(2021{\natexlab{a}})He, Zhou, and Gu]{he2021logarithmic}
Jiafan He, Dongruo Zhou, and Quanquan Gu.
\newblock Logarithmic regret for reinforcement learning with linear function
  approximation.
\newblock In \emph{International Conference on Machine Learning}, pages
  4171--4180. PMLR, 2021{\natexlab{a}}.

\bibitem[He et~al.(2021{\natexlab{b}})He, Zhou, and Gu]{he2021nearly}
Jiafan He, Dongruo Zhou, and Quanquan Gu.
\newblock Nearly minimax optimal reinforcement learning for discounted {MDP}s.
\newblock In \emph{Advances in Neural Information Processing Systems},
  volume~34, pages 22288--22300, 2021{\natexlab{b}}.

\bibitem[He et~al.(2022)He, Zhao, Zhou, and Gu]{he2022nearly}
Jiafan He, Heyang Zhao, Dongruo Zhou, and Quanquan Gu.
\newblock Nearly minimax optimal reinforcement learning for linear markov
  decision processes.
\newblock \emph{arXiv preprint arXiv:2212.06132}, 2022.

\bibitem[Hu et~al.(2022)Hu, Chen, and Huang]{hu2022nearly}
Pihe Hu, Yu~Chen, and Longbo Huang.
\newblock Nearly minimax optimal reinforcement learning with linear function
  approximation.
\newblock In \emph{International Conference on Machine Learning}, pages
  8971--9019, 2022.

\bibitem[Huber(1964)]{huber1964robust}
Peter~J Huber.
\newblock Robust estimation of a location parameter.
\newblock \emph{The Annals of Mathematical Statistics}, pages 73--101, 1964.

\bibitem[Hull(2012)]{hull2012risk}
John Hull.
\newblock \emph{{R}isk {M}anagement and {F}inancial {I}nstitutions}.
\newblock John Wiley \& Sons, 2012.

\bibitem[Jin et~al.(2018)Jin, Allen-Zhu, Bubeck, and Jordan]{jin2018q}
Chi Jin, Zeyuan Allen-Zhu, Sebastien Bubeck, and Michael~I Jordan.
\newblock Is {Q}-learning provably efficient?
\newblock In \emph{Advances in Neural Information Processing Systems},
  volume~31, 2018.

\bibitem[Jin et~al.(2020{\natexlab{a}})Jin, Krishnamurthy, Simchowitz, and
  Yu]{jin2020reward}
Chi Jin, Akshay Krishnamurthy, Max Simchowitz, and Tiancheng Yu.
\newblock Reward-free exploration for reinforcement learning.
\newblock In \emph{International Conference on Machine Learning}, pages
  4870--4879. PMLR, 2020{\natexlab{a}}.

\bibitem[Jin et~al.(2020{\natexlab{b}})Jin, Yang, Wang, and
  Jordan]{jin2020provably}
Chi Jin, Zhuoran Yang, Zhaoran Wang, and Michael~I Jordan.
\newblock Provably efficient reinforcement learning with linear function
  approximation.
\newblock In \emph{Conference on Learning Theory}, pages 2137--2143,
  2020{\natexlab{b}}.

\bibitem[Khamaru et~al.(2021)Khamaru, Pananjady, Ruan, Wainwright, and
  Jordan]{khamaru2021temporal}
Koulik Khamaru, Ashwin Pananjady, Feng Ruan, Martin~J Wainwright, and Michael~I
  Jordan.
\newblock Is temporal difference learning optimal? {A}n instance-dependent
  analysis.
\newblock \emph{SIAM Journal on Mathematics of Data Science}, 3\penalty0
  (4):\penalty0 1013--1040, 2021.

\bibitem[Kim et~al.(2021)Kim, Yang, and Jun]{kim2021improved}
Yeoneung Kim, Insoon Yang, and Kwang-Sung Jun.
\newblock Improved regret analysis for variance-adaptive linear bandits and
  horizon-free linear mixture mdps.
\newblock \emph{arXiv preprint arXiv:2111.03289}, 2021.

\bibitem[Kirschner and Krause(2018)]{kirschner2018information}
Johannes Kirschner and Andreas Krause.
\newblock Information directed sampling and bandits with heteroscedastic noise.
\newblock In \emph{Conference on Learning Theory}, pages 358--384, 2018.

\bibitem[Lattimore(2017)]{lattimore2017scale}
Tor Lattimore.
\newblock A scale free algorithm for stochastic bandits with bounded kurtosis.
\newblock In \emph{Advances in Neural Information Processing Systems},
  volume~30, 2017.

\bibitem[Lecu{\'e} and Lerasle(2020)]{lecue2020robust}
Guillaume Lecu{\'e} and Matthieu Lerasle.
\newblock Robust machine learning by median-of-means: {T}heory and practice.
\newblock \emph{The Annals of Statistics}, 48\penalty0 (2):\penalty0 906--931,
  2020.

\bibitem[Li et~al.(2021)Li, Cai, Chen, Gu, Wei, and Chi]{li2021q}
Gen Li, Changxiao Cai, Yuxin Chen, Yuantao Gu, Yuting Wei, and Yuejie Chi.
\newblock Is {Q}-learning minimax optimal? {A} tight sample complexity
  analysis.
\newblock \emph{arXiv preprint arXiv:2102.06548}, 2021.

\bibitem[Li et~al.(2016)Li, Monroe, Ritter, Jurafsky, Galley, and
  Gao]{li2016deep}
Jiwei Li, Will Monroe, Alan Ritter, Dan Jurafsky, Michel Galley, and Jianfeng
  Gao.
\newblock Deep reinforcement learning for dialogue generation.
\newblock In \emph{Conference on Empirical Methods in Natural Language
  Processing}, pages 1192--1202, 2016.

\bibitem[Li et~al.(2023)Li, Yang, Liang, Zhang, and Jordan]{li2021polyak}
Xiang Li, Wenhao Yang, Jiadong Liang, Zhihua Zhang, and Michael~I Jordan.
\newblock A statistical analysis of {P}olyak-{R}uppert averaged q-learning.
\newblock In \emph{International Conference on Artificial Intelligence and
  Statistics}, 2023.

\bibitem[Lillicrap et~al.(2015)Lillicrap, Hunt, Pritzel, Heess, Erez, Tassa,
  Silver, and Wierstra]{lillicrap2015continuous}
Timothy~P Lillicrap, Jonathan~J Hunt, Alexander Pritzel, Nicolas Heess, Tom
  Erez, Yuval Tassa, David Silver, and Daan Wierstra.
\newblock Continuous control with deep reinforcement learning.
\newblock \emph{arXiv preprint arXiv:1509.02971}, 2015.

\bibitem[Lu et~al.(2019)Lu, Wang, Hu, and Zhang]{lu2019optimal}
Shiyin Lu, Guanghui Wang, Yao Hu, and Lijun Zhang.
\newblock Optimal algorithms for {L}ipschitz bandits with heavy-tailed rewards.
\newblock In \emph{International Conference on Machine Learning}, pages
  4154--4163, 2019.

\bibitem[Medina and Yang(2016)]{medina2016no}
Andres~Munoz Medina and Scott Yang.
\newblock No-regret algorithms for heavy-tailed linear bandits.
\newblock In \emph{International Conference on Machine Learning}, pages
  1642--1650, 2016.

\bibitem[Melo and Ribeiro(2007)]{melo2007q}
Francisco~S Melo and M~Isabel Ribeiro.
\newblock Q-learning with linear function approximation.
\newblock In \emph{International Conference on Computational Learning Theory},
  pages 308--322, 2007.

\bibitem[Min et~al.(2021)Min, Wang, Zhou, and Gu]{min2021variance}
Yifei Min, Tianhao Wang, Dongruo Zhou, and Quanquan Gu.
\newblock Variance-aware off-policy evaluation with linear function
  approximation.
\newblock In \emph{Advances in Neural Information Processing Systems},
  volume~34, pages 7598--7610, 2021.

\bibitem[Mou et~al.(2020)Mou, Li, Wainwright, Bartlett, and
  Jordan]{mou2020linear}
Wenlong Mou, Chris~Junchi Li, Martin~J Wainwright, Peter~L Bartlett, and
  Michael~I Jordan.
\newblock On linear stochastic approximation: {F}ine-grained {P}olyak-{R}uppert
  and non-asymptotic concentration.
\newblock In \emph{Conference on Learning Theory}, pages 2947--2997, 2020.

\bibitem[Mou et~al.(2022)Mou, Khamaru, Wainwright, Bartlett, and
  Jordan]{mou2022optimal}
Wenlong Mou, Koulik Khamaru, Martin~J Wainwright, Peter~L Bartlett, and
  Michael~I Jordan.
\newblock Optimal variance-reduced stochastic approximation in banach spaces.
\newblock \emph{arXiv preprint arXiv:2201.08518}, 2022.

\bibitem[Pananjady and Wainwright(2020)]{pananjady2020instance}
Ashwin Pananjady and Martin~J Wainwright.
\newblock Instance-dependent $\ell_\infty$ bounds for policy evaluation in
  tabular reinforcement learning.
\newblock \emph{IEEE Transactions on Information Theory}, 67\penalty0
  (1):\penalty0 566--585, 2020.

\bibitem[Posekany et~al.(2011)Posekany, Felsenstein, and
  Sykacek]{posekany2011biological}
Alexandra Posekany, Klaus Felsenstein, and Peter Sykacek.
\newblock Biological assessment of robust noise models in microarray data
  analysis.
\newblock \emph{Bioinformatics}, 27\penalty0 (6):\penalty0 807--814, 2011.

\bibitem[Shao et~al.(2018)Shao, Yu, King, and Lyu]{shao2018almost}
Han Shao, Xiaotian Yu, Irwin King, and Michael~R Lyu.
\newblock Almost optimal algorithms for linear stochastic bandits with
  heavy-tailed payoffs.
\newblock In \emph{Advances in Neural Information Processing Systems}, 2018.

\bibitem[Sidford et~al.(2018)Sidford, Wang, Wu, Yang, and Ye]{sidford2018near}
Aaron Sidford, Mengdi Wang, Xian Wu, Lin~F Yang, and Yinyu Ye.
\newblock Near-optimal time and sample complexities for solving {M}arkov
  decision processes with a generative model.
\newblock In \emph{Advances in Neural Information Processing Systems},
  volume~31, pages 5192--5202, 2018.

\bibitem[Silver et~al.(2016)Silver, Huang, Maddison, Guez, Sifre, Van
  Den~Driessche, Schrittwieser, Antonoglou, Panneershelvam, Lanctot,
  et~al.]{silver2016mastering}
David Silver, Aja Huang, Chris~J Maddison, Arthur Guez, Laurent Sifre, George
  Van Den~Driessche, Julian Schrittwieser, Ioannis Antonoglou, Veda
  Panneershelvam, Marc Lanctot, et~al.
\newblock Mastering the game of {Go} with deep neural networks and tree search.
\newblock \emph{Nature}, 529\penalty0 (7587):\penalty0 484--489, 2016.

\bibitem[Sun(2021)]{sun2021we}
Qiang Sun.
\newblock Do we need to estimate the variance in robust mean estimation?
\newblock \emph{arXiv preprint arXiv:2107.00118}, 2021.

\bibitem[Sun et~al.(2020)Sun, Zhou, and Fan]{sun2020adaptive}
Qiang Sun, Wen-Xin Zhou, and Jianqing Fan.
\newblock Adaptive {H}uber regression.
\newblock \emph{Journal of the American Statistical Association}, 115\penalty0
  (529):\penalty0 254--265, 2020.

\bibitem[Sutton and Barto(2018)]{sutton2018reinforcement}
Richard~S Sutton and Andrew~G Barto.
\newblock \emph{Reinforcement Learning: An Introduction}.
\newblock MIT Press, 2018.

\bibitem[Tossou et~al.(2019)Tossou, Basu, and Dimitrakakis]{tossou2019near}
Aristide Tossou, Debabrota Basu, and Christos Dimitrakakis.
\newblock Near-optimal optimistic reinforcement learning using empirical
  bernstein inequalities.
\newblock \emph{arXiv preprint arXiv:1905.12425}, 2019.

\bibitem[Vakili et~al.(2013)Vakili, Liu, and Zhao]{vakili2013deterministic}
Sattar Vakili, Keqin Liu, and Qing Zhao.
\newblock Deterministic sequencing of exploration and exploitation for
  multi-armed bandit problems.
\newblock \emph{IEEE Journal of Selected Topics in Signal Processing},
  7\penalty0 (5):\penalty0 759--767, 2013.

\bibitem[Wagenmaker and Jamieson(2022)]{wagenmaker2022instance}
Andrew Wagenmaker and Kevin Jamieson.
\newblock Instance-dependent near-optimal policy identification in linear
  {MDP}s via online experiment design.
\newblock In \emph{Advances in Neural Information Processing Systems}, 2022.

\bibitem[Wagenmaker et~al.(2022{\natexlab{a}})Wagenmaker, Chen, Simchowitz, Du,
  and Jamieson]{wagenmaker2021first}
Andrew~J Wagenmaker, Yifang Chen, Max Simchowitz, Simon Du, and Kevin Jamieson.
\newblock First-order regret in reinforcement learning with linear function
  approximation: {A} robust estimation approach.
\newblock In \emph{International Conference on Machine Learning}, pages
  22384--22429. PMLR, 2022{\natexlab{a}}.

\bibitem[Wagenmaker et~al.(2022{\natexlab{b}})Wagenmaker, Chen, Simchowitz, Du,
  and Jamieson]{wagenmaker2022reward}
Andrew~J Wagenmaker, Yifang Chen, Max Simchowitz, Simon Du, and Kevin Jamieson.
\newblock Reward-free {RL} is no harder than reward-aware {RL} in linear
  {Ma}rkov decision processes.
\newblock In \emph{International Conference on Machine Learning}, pages
  22430--22456, 2022{\natexlab{b}}.

\bibitem[Wagenmaker et~al.(2022{\natexlab{c}})Wagenmaker, Simchowitz, and
  Jamieson]{wagenmaker2022beyond}
Andrew~J Wagenmaker, Max Simchowitz, and Kevin Jamieson.
\newblock Beyond no regret: Instance-dependent pac reinforcement learning.
\newblock In \emph{Conference on Learning Theory}, pages 358--418,
  2022{\natexlab{c}}.

\bibitem[Xie et~al.(2021)Xie, Cheng, Jiang, Mineiro, and
  Agarwal]{xie2021bellman}
Tengyang Xie, Ching-An Cheng, Nan Jiang, Paul Mineiro, and Alekh Agarwal.
\newblock Bellman-consistent pessimism for offline reinforcement learning.
\newblock In \emph{Advances in Neural Information Processing Systems},
  volume~34, pages 6683--6694, 2021.

\bibitem[Xue et~al.(2021)Xue, Wang, Wang, and Zhang]{xue2021nearly}
Bo~Xue, Guanghui Wang, Yimu Wang, and Lijun Zhang.
\newblock Nearly optimal regret for stochastic linear bandits with heavy-tailed
  payoffs.
\newblock In \emph{International Conference on International Joint Conferences
  on Artificial Intelligence}, pages 2936--2942, 2021.

\bibitem[Yang and Wang(2019)]{yang2019sample}
Lin Yang and Mengdi Wang.
\newblock Sample-optimal parametric {Q}-learning using linearly additive
  features.
\newblock In \emph{International Conference on Machine Learning}, pages
  6995--7004, 2019.

\bibitem[Yang and Wang(2020)]{yang2020reinforcement}
Lin Yang and Mengdi Wang.
\newblock Reinforcement learning in feature space: Matrix bandit, kernels, and
  regret bound.
\newblock In \emph{International Conference on Machine Learning}, pages
  10746--10756. PMLR, 2020.

\bibitem[Yin and Wang(2021)]{yin2021towards}
Ming Yin and Yu-Xiang Wang.
\newblock Towards instance-optimal offline reinforcement learning with
  pessimism.
\newblock In \emph{Advances in Neural Information Processing Systems}, pages
  4065--4078, 2021.

\bibitem[Zanette and Brunskill(2019)]{zanette2019tighter}
Andrea Zanette and Emma Brunskill.
\newblock Tighter problem-dependent regret bounds in reinforcement learning
  without domain knowledge using value function bounds.
\newblock In \emph{International Conference on Machine Learning}, pages
  7304--7312. PMLR, 2019.

\bibitem[Zanette et~al.(2020)Zanette, Lazaric, Kochenderfer, and
  Brunskill]{zanette2020learning}
Andrea Zanette, Alessandro Lazaric, Mykel Kochenderfer, and Emma Brunskill.
\newblock Learning near optimal policies with low inherent bellman error.
\newblock In \emph{International Conference on Machine Learning}, pages
  10978--10989, 2020.

\bibitem[Zhang et~al.(2020)Zhang, Zhou, and Ji]{zhang2020almost}
Zihan Zhang, Yuan Zhou, and Xiangyang Ji.
\newblock Almost optimal model-free reinforcement learning via
  reference-advantage decomposition.
\newblock In \emph{Advances in Neural Information Processing Systems},
  volume~33, pages 15198--15207, 2020.

\bibitem[Zhang et~al.(2021)Zhang, Yang, Ji, and Du]{zhang2021improved}
Zihan Zhang, Jiaqi Yang, Xiangyang Ji, and Simon~S Du.
\newblock Improved variance-aware confidence sets for linear bandits and linear
  mixture {MDP}.
\newblock In \emph{Advances in Neural Information Processing Systems}, pages
  4342--4355, 2021.

\bibitem[Zhou and Gu(2022)]{zhou2022computationally}
Dongruo Zhou and Quanquan Gu.
\newblock Computationally efficient horizon-free reinforcement learning for
  linear mixture mdps.
\newblock \emph{arXiv preprint arXiv:2205.11507}, 2022.

\bibitem[Zhou et~al.(2021)Zhou, Gu, and Szepesvari]{zhou2021nearly}
Dongruo Zhou, Quanquan Gu, and Csaba Szepesvari.
\newblock Nearly minimax optimal reinforcement learning for linear mixture
  markov decision processes.
\newblock In \emph{Conference on Learning Theory}, pages 4532--4576, 2021.

\bibitem[Zhuang and Sui(2021)]{zhuang2021no}
Vincent Zhuang and Yanan Sui.
\newblock No-regret reinforcement learning with heavy-tailed rewards.
\newblock In \emph{International Conference on Artificial Intelligence and
  Statistics}, pages 3385--3393, 2021.

\end{thebibliography}
\bibliographystyle{plainnat}

\appendix
\begin{appendix}
	\onecolumn
	\begin{center}
	{\huge \textbf{Appendix}}
\end{center}
%\section{Algorithm Description of VARA}
%\label{appen:VARA}

% \scomment{Give some overview?}

\paragraph{Overview}
 In Appendix~\ref{append:variance}, we explain the rationale behind the variance estimator used by VARA. Appendix~\ref{proof:bandit} contains proofs for Theorem~\ref{thm:heavy} and~\ref{thm:main-bandit} specifically for linear bandits. The theoretical analysis for VARA is presented in Appendix~\ref{proof:mdp}, where we offer a proof sketch for Theorem~\ref{thm:mdp}, while all related technical lemmas are deferred to Appendices~\ref{proof:lemmas-mdp} and~\ref{proof:auxiliary}. We also highlight the differences between our analysis and previous work. In Appendix~\ref{proof:space-mdp}, we provide a proof for Theorem~\ref{thm:space} that analyzes the space and computational complexity of VARA.

\section{Variance Estimation for Value Functions}
\label{append:variance}
In order to achieve worst-case optimality, \citet{he2022nearly} proposes two important techniques we adopt in Algorithm~\ref{algo:linear}. 

The first is the monotonicity of value functions.
Specifically, we aim to enforce a decrease in $k$ for the actual optimistic value function $\overoptQ_h^k(\cdot,\cdot)$ and an increase in $k$ for the actual pessimistic value function $\overpesQ_h^k(\cdot,\cdot)$. 
This concept is explained in detail below.
In linear MDPs, we have $[\PB_h V_{h+1}](s, a) = \langle\Bphi(s, a), \Bmu_h^* \V_{h+1}\rangle$  for any value function $V = \{ V_{h}\}_{h \in [H]}$ and $[\PB_h R_h](s, a) = \langle\Bphi(s, a), \Btheta_h^*\rangle$ for all $h \in [H]$.
One crucial aspect of typical analysis (including ours) is demonstrating the high probability of the following event outlined in Appendix~\ref{proof:high-prob-event}. 
Specifically, for all $h \in [H]$ and $k \in [K]$, we need to establish that the following equations hold simultaneously with high pribability:
\begin{gather}
	\label{eq:high-prob-event}
	\begin{split}
		\|\Btheta_{h, k} - \Btheta_{h}^*\|_{\H_{h,k}} \le \beta_R = \TOM(\sqrt{d}), \\
		\|(\Bmu_{h,k-1}-\Bmu_h^*)\BoveroptV_{h+1}^k\|_{\H_{h,k-1}} \le \beta_V= \TOM(\sqrt{d}), \\
		\|(\Bmu_{h,k-1}-\Bmu_h^*)\BoverpesV_{h+1}^k\|_{\H_{h,k-1}} \le \beta_V= \TOM(\sqrt{d}).
	\end{split}
\end{gather}
Conditional on  the event that all inequalities in \eqref{eq:high-prob-event} hold, we can easily verify that 
\begin{gather*}
    |\langle \Bphi(\cdot, \cdot), \Btheta_{h, k-1} \rangle- [\PB_h R_h](\cdot,\cdot)| \le \beta_R \|\Bphi(\cdot, \cdot)\|_{\H_{h,k-1}^{-1}},\\
    |\langle \Bphi(\cdot, \cdot), \Bmu_{h, k-1} \V_{h+1} \rangle -[\PB_h V_{h+1}](\cdot, \cdot)| \le \beta_V \|\Bphi(\cdot, \cdot)\|_{\H_{h,k-1}^{-1}}
\end{gather*}
for both $\V_{h+1} \in \{\BoverpesV_{h+1}^k, \BoveroptV_{h+1}^k\}$.
Therefore, we define the temporary optimistic value function by 
\[
\optQ_h^k(\cdot, \cdot) = \langle \Bphi(\cdot, \cdot), \Btheta_{h, k-1} + \Bmu_{h, k-1}\BoveroptV_{h+1}^{k} \rangle + \beta \|\Bphi(\cdot, \cdot)\|_{\H_{h, k-1}^{-1}},
\]
and the temporary pessimistic value function by 
\[
\pesQ_h^k(\cdot, \cdot) = \langle \Bphi(\cdot, \cdot), \Btheta_{h, k-1} + \Bmu_{h, k-1} \BoveroptV_{h+1}^{k} \rangle -\beta \|\Bphi(\cdot, \cdot)\|_{\H_{h, k-1}^{-1}}
\]
where $\beta := \beta_R + \beta_V = \TOM(\sqrt{d})$.
The actual optimistic value function $\overoptQ_h^k(\cdot,\cdot)$ is the minimum function of history temporary optimistic value functions $\optQ_h^k(\cdot, \cdot)$, and the actual pessimistic value function $\overpesQ_h^k(\cdot,\cdot)$ is the maximum function of history temporary pessimistic value functions $\pesQ_h^k(\cdot, \cdot)$ (Line 7 in Algorithm~\ref{algo:linear}). 
In this way, $\overoptQ_h^k(\cdot, \cdot)$ is always non-increasing in $k$ and $\overpesQ_h^k(\cdot, \cdot)$ is always non-decreasing in $k$.

The second is the rare-switching value function update, which updates the value function only when the determinant of the covariance matrix significantly exceeds the previous value (Line 6 in Algorithm~\ref{algo:linear}). 
This approach allows the complexity, as measured by the metric entropy, of the function class to which $\overoptV_h^k(\cdot)$ or $\overpesV_h^k(\cdot)$ belongs to be independent of $K$. 
Notably, the metric entropy is linearly dependent on $\TOM(dH)$.
Moreover, on the event~\eqref{eq:high-prob-event}, we can establish optimism and pessimism in Lemma~\ref{lem:opt-pes}, i.e., for all $k \in [K]$ and $h \in [H]$,
\begin{equation}
	\label{eq:opt-pes}
	\overpesV_{h+1}^k(\cdot)
	\le V_{h+1}^*(\cdot)
	\le \overoptV_{h+1}^k(\cdot).
\end{equation}

	Directly estimating the variance of the optimistic value function $\overoptV_{h+1}^k(\cdot)$ will encounter the dependence issue, which is discussed in~\citep{jin2020provably} and will introduce an additional $\sqrt{d}$ factor in the regret due to the covering-based decoupling argument. 
	To eliminate this factor, after noting the inequality
	\[
	[\VB_h\overoptV_{h+1}^k](\cdot, \cdot) \le
	2[\VB_hV_{h+1}^*](\cdot, \cdot) + 2[\VB_h(\overoptV_{h+1}^k-V_{h+1}^*)](\cdot,\cdot),
	\]
	\citet{hu2022nearly} decompose the optimistic value function  $\overoptV_{h+1}^k(\cdot)$ into the optimal value function $V_{h+1}^*(\cdot)$ and the sub-optimality
	gap $[\overoptV_{h+1}^k-V_{h+1}^*](\cdot)$ and estimate their variances $[\VB_hV_{h+1}^*](\cdot, \cdot)$ and $[\VB_h(\overoptV_{h+1}^k-V_{h+1}^*)](\cdot,\cdot)$ separately.
	The key insight is that: (i) as $V_{h+1}^*$ is deterministic, 
 there is no additional $\sqrt{d}$ dependence in estimating $[\VB_hV_{h+1}^*](\cdot, \cdot)$, and (ii) as $\overoptV_{h+1}^k$ gradually converges to $V_{h+1}^*$, though a uniform argument is still used, the incurred $\sqrt{d}$ factor in the estimation of $[\VB_h(\overoptV_{h+1}^k-V_{h+1}^*)]$ has ignorable effects on the final regret.
	We now describe the way we estimate these two variances.
	\begin{itemize}[leftmargin=*]
		\item For $[\VB_hV_{h+1}^*](\cdot, \cdot)$, since $V_{h+1}^*$ is unknown, a natural choice is to estimate it by the optimistic value function $\overoptV_{h+1}^k$.
		Hence, we  estimate $ [\VB_h V_{h+1}^*](s_{h, k}, a_{h, k})$ via
		\begin{equation}
			\label{eq:HVB}
			[\HVB_{h, k} \overoptV_{h+1}^k](s_{h, k}, a_{h, k}) := [\widehat{\PB}_{h, k} (\overoptV_{h+1}^k)^2](s_{h, k}, a_{h, k})_{[0, \HM^2]} - \left[
			[\widehat{\PB}_{h, k} \overoptV_{h+1}^k](s_{h, k}, a_{h, k})_{[0, \HM^2]}
			\right]^2.
		\end{equation}
		To measure estimation accuracy, we introduce an error term $U_{h, k}$ to guarantee that with high probability, $\left|
		[\HVB_{h, k} \overoptV_{h+1}^k](s_{h, k}, a_{h, k}) - [\VB_h V_{h+1}^*](s_{h, k}, a_{h, k})
		\right| \le U_{h, k}$ holds uniformly over all $h, k$ where
		\begin{equation}
			\tag{\ref{eq:U}}
			U_{h, k} =
			\min\left\{
			\VM^2, 11\HM \beta_0\cdot \|\Bphi_{h, k}\|_{\H_{h, k-1}^{-1}} + 4\HM \cdot \HPB_{h, k}(\overoptV_{h+1}^k -\overpesV_{h+1}^k)(s_{h, k}, a_{h, k})
			\right\}
		\end{equation}
		and $\beta_0 = \TOM\left(\frac{\HM}{\sigma_{\min}} \sqrt{d^3H}\right)$ is an exploration radius.
		\item For $[\VB_h(\overoptV_{h+1}^k-V_{h+1}^*)](\cdot, \cdot)$, in order to meet the measurability condition of a concentration inequality (Lemma~\ref{lem:self-bern}), we require
		\begin{equation}
			\label{eq:sigma-bound}
			\sigma_{h,k}^2 \ge d^3 H \cdot \sup_{k \le j \le K} [\VB_h(\overoptV_{h+1}^j- V_{h+1}^*)](s_{h, k}, a_{h, k}).
		\end{equation}
		Note that $\sigma_{h,k}$ is $\FM_{h, k}$-measurable while $\overoptV_{h+1}^k(\cdot)$ is $\FM_{H, k-1}$-measurable.
		The condition \eqref{eq:sigma-bound} essentially requires a $\FM_{h,k}$-measurable upper bound for $[\VB_h(\overoptV_{h+1}^j- V_{h+1}^*)](s_{h, k}, a_{h, k})$ even if $j \ge k$.
		Fortunately, we have for any $k \le j\le K$, 
		\begin{align}
			[\VB_h(\overoptV_{h+1}^j-V_{h+1}^*)](s_{h, k}, a_{h, k})
			&\le[ \PB_h(\overoptV_{h+1}^j-V_{h+1}^*)^2](s_{h, k}, a_{h, k})\nonumber \\
			&\overset{(a)}{\le} \HM [ \PB_h(\overoptV_{h+1}^j-V_{h+1}^*)](s_{h, k}, a_{h, k})\nonumber \\
			&\overset{(b)}{\le}\HM [ \PB_h(\overoptV_{h+1}^j-\overpesV_{h+1}^j)](s_{h, k}, a_{h, k})\nonumber \\
			&\overset{(c)}{\le}\HM [ \PB_h(\overoptV_{h+1}^k-\overpesV_{h+1}^k)](s_{h, k}, a_{h, k})
			\label{eq:sigma-bound1}
		\end{align}
		where $(a)$ uses $|\overoptV_{h+1}^j-V_{h+1}^*|(\cdot) \le \HM$ and the optimism of $\overoptV_{h+1}^j(\cdot)$, $(b)$ follows from the pessimism in~\eqref{eq:opt-pes}, and $(c)$ uses the monotonicity of value functions.
		The RHS of~\eqref{eq:sigma-bound1} is $\FM_{h,k}$-measurable but intractable due to the population expectation $\PB_h(\cdot)$.
		By replacing $\PB_h(\cdot)$ with the tractable $\HPB_{h, k}(\cdot)$, we introduce $E_{h, k}$ to overestimate the RHS of~\eqref{eq:sigma-bound1} where
		\begin{equation}
			\tag{\ref{eq:E}}
			E_{h, k} =
			\min\left\{
			\HM^2, 2\HM\beta_0 \cdot \|\Bphi_{h, k}\|_{\H_{h, k-1}^{-1}}
			+ \HM \cdot	\left[ \HPB_{h, k}(\overoptV_{h+1}^k - \overpesV_{h+1}^k)\right](s_{h, k}, a_{h, k})
			\right\}.
		\end{equation}
		Hence,~\eqref{eq:sigma-bound} is guaranteed by $\sigma_{h,k}^2 \ge d^3H \cdot E_{h, k}$.
		The extra $d^3H$ factor is introduced to offset the error caused by the covering number argument.
	\end{itemize}

\section{Proof for Section~\ref{sec:bandit}}
\label{proof:bandit}
\subsection{Proof of Theorem~\ref{thm:heavy}}
Let $z_t(\Btheta) =\frac{y_t -  \langle \Bphi_t, \Btheta \rangle}{\sigma_t}$ for simplicity.
Notice that the gradient is given by
\[
\nabla L_T(\Btheta) := \lambda \Btheta - \sum_{t=1}^T \frac{\tau_t z_t(\Btheta) }{\sqrt{ \tau_t^2 +  z_t(\Btheta)^2}  } \frac{\Bphi_t}{\sigma_t}.
\]
Our estimator is the solution of~\eqref{eq:Theta_T}.
By Proposition 1.3 in~\citep{bubeck2015convex},
the first-order stationary condition of the constrained convex optimization~\eqref{eq:Theta_T} implies that $\langle \nabla L_{T}(\Btheta_T), \Btheta_T- \Btheta \rangle \le 0$ for all $\Btheta \in \mathrm{Ball}_d(B)$.
More specifically, due to $\|\Btheta^*\| \le B$, we have
\begin{equation}
	\label{eq:first-order}
	\langle \nabla L_{T}(\Btheta_T), \Btheta_T- \Btheta^* \rangle \le 0.
\end{equation}
By the convexity of $L_T(\cdot)$, it follows that
\begin{align}
	\label{eq:help1}
	0 \le \langle  \Btheta_T - \Btheta^*, \nabla L_{T}(\Btheta_T) - \nabla L_{T}(\Btheta^*)\rangle .
\end{align}

% For a technical reason to facilitate analysis, we construct an intermediate quantity to proceed with our proof.
% We define it by $\widehat{\Btheta}_{T}$ with the following definition
% \[
% \widehat{\Btheta}_{T}
% = \left\{ \begin{array}{ll}
% 	\Btheta_T &  \text{if} \ \|\Btheta_{T} - \Btheta^*\|_{\H_T} \le \beta_T  \\
% 	\eta_T \Btheta_T + (1-\eta_T) \Btheta^* &  \text{if} \ \|\Btheta_{T} - \Btheta^*\|_{\H_T} > \beta_T. \\
% \end{array}
% \right.
% \]
% Here we use the fact that if $\ \|\Btheta_{T} - \Btheta^*\|_{\H_T} > \beta_T$, we can always choose $\eta_T \in (0,1)$ so that $\|
% \widehat{\Btheta}_{T}- \Btheta^*\|_{\H_T} = \beta_T$.
% As a result, we construct a quantity $
% \widehat{\Btheta}_{T}$ that always satisfies $\|
% \widehat{\Btheta}_{T} - \Btheta^*\|_{\H_T} \le \beta_T$.
% Noticing that $L_T(\cdot)$ is convex, by Lemma A.1 in~\citep{zhou2018new}, we have
% \begin{align}
% 	\label{eq:help1}
% 	0 \le \langle  \widehat{\Btheta}_{T}- \Btheta^*, \nabla L_{T}(\widehat{\Btheta}_{T}) - \nabla L_{T}(\Btheta^*)\rangle
% 	&\le \eta_T  \langle  \Btheta_T - \Btheta^*, \nabla L_{T}(\Btheta_T) - \nabla L_{T}(\Btheta^*)\rangle .
% \end{align}

In Lemma~\ref{lem:hessian}, we show that with high probability and up to a constant factor, $\nabla^2 L_{T}(\Btheta)$ is a good approximation of $\H_T$ uniformly for all $T\ge 1$ and $\|\Btheta\| \le B$.
Its proof is deferred to Section~\ref{proof:hessian}.
In the following, we will frequently mention a quantity denoted by $\kappa$.
\begin{equation}
	\label{eq:kappa}
	\kappa := d \log\left(1 + \frac{T L^2}{d\lambda \sigma_{\min}^2}\right) .
\end{equation}
\begin{lem}
	\label{lem:hessian}
	Assume $\EB[z_t^2(\Btheta^*)|\FM_{t-1}] \le b^2$ for all $t \ge 1$.
	If we set
	\[
	\tau_0 \sqrt{\log \frac{2T^2}{\delta}}\ge \max\{ \sqrt{2\kappa} b, 2\sqrt{d}   \},
	\]
	with probability at least $1-2\delta$, we have that for all $T \ge 0$,
	\[
	\frac{1}{4}\H_T
	\preceq \nabla^2 L_T(\Btheta) \preceq \H_T
	\quad \text{for any} \quad
	\|\Btheta\| \le B.
	\]
\end{lem}

By the mean value theorem for vector-valued functions, we have
\[
\nabla L_{T}(\widehat{\Btheta}_{T}) - \nabla L_{T}(\Btheta^*)
= \int_0^1 \nabla^2  L_{T}((1-\eta) \Btheta^* + \eta \widehat{\Btheta}_{T}) \mathrm{d} \eta \cdot (  \widehat{\Btheta}_T - \Btheta^*).
\]
Using Lemma~\ref{lem:hessian} and $\|(1-\eta) \Btheta^* + \eta {\Btheta}_{T}) \| \le B$ for all $\eta \in [0, 1]$, we have
\begin{equation}
	\label{eq:help2}
	\frac{1}{4}\|{\Btheta}_{T} - \Btheta^* \|_{\H_T}^2  \le \langle  {\Btheta}_{T}- \Btheta^*, \nabla L_{T}({\Btheta}_{T}) - \nabla L_{T}(\Btheta^*)\rangle.
\end{equation}
By~\eqref{eq:help2},~\eqref{eq:help1} and~\eqref{eq:first-order}, we have
\begin{align*}
	\frac{1}{4}\|{\Btheta}_{T} - \Btheta^* \|_{\H_T}^2 
	&\le  \langle  \Btheta_T - \Btheta^*, \nabla L_{T}(\Btheta_T) - \nabla L_{T}(\Btheta^*)\rangle\\
	&\le  \langle  \Btheta_T - \Btheta^*, - \nabla L_{T}(\Btheta^*)\rangle\\
 &\le  \|  {\Btheta}_T - \Btheta^* \|_{\H_T}  \| \nabla L_{T}(\Btheta^*)\|_{\H_T^{-1}}
\end{align*}
which implies that
\begin{equation}
	\label{eq:target1}
	\|{\Btheta}_{T} - \Btheta^* \|_{\H_T} \le4\| \nabla L_{T}(\Btheta^*)\|_{\H_T^{-1}}.
\end{equation}

In the following, we are going to provide a high-probability bound for $ \| \nabla L_{\tau}(\Btheta^*)\|_{\H_T^{-1}}$.
In particular, we have the following lemma whose proof is provided in Section~\ref{proof:grad}.
\begin{lem}
	\label{lem:grad}
	Assume $\EB[z_t^2(\Btheta^*)|\FM_{t-1}] \le b^2$ for all $t \ge 1$.
	With probability at least $1-\delta$,  for all $T \ge 1$, it follows that
	\[
	\| \nabla L_{T}(\Btheta^*)\|_{\H_T^{-1}} 
	\le 8\left[ \frac{\kappa b^2}{\tau_0} + b \sqrt{\kappa\log\frac{2T^2}{\delta}} + \tau_0\log\frac{2T^2}{\delta}   \right]+ \sqrt{\lambda} B
	\]
	where $\kappa$ is defined in~\eqref{eq:kappa}.
\end{lem}

Combing~\eqref{eq:target1}, Lemma~\ref{lem:hessian} and Lemma~\ref{lem:grad}, we know that if 
$\tau_0\sqrt{\log \frac{2T^2}{\delta}}= \max\{ \sqrt{2\kappa}b, 2\sqrt{d}  \}$, with probability at least $1-2\delta$, we have that  $\|{\Btheta}_t - \Btheta^*\|_{\H_t} < \beta_t$ for all $1 \le t \le T$ where
\[
\beta_t = 32\left[b\sqrt{\frac{\kappa}{2}\log \frac{2t^2}{\delta}} +
b\sqrt{\kappa \log\frac{2t^2}{\delta}}+ (\sqrt{2\kappa} b+2\sqrt{d}  )\sqrt{\log\frac{2t^2}{\delta}  } \right] + 4\sqrt{\lambda} B.
\]
It implies that $\Btheta^*$ indeed locals in all constructed confidence regimes, i.e., for all $1 \le t \le T$,
\[
\Btheta^* \in \CM_t = \left\{ \Btheta \in \mathrm{Ball}_d(B) :  \|\Btheta_t - \Btheta\|_{\H_t} \le  \beta_t  \right\}.
\]
Notice that by the choice of $\beta_0$ and $\H_0$, we still have $\Btheta^* \in \CM_0$.
Finally, $b=1$ in our case completes the proof.

\subsection{Proof of Theorem~\ref{thm:main-bandit}}

Now, we turn to the regret~\eqref{eq:regret}.
Recall that at iteration $t$, we set 
\[
(\Bphi_t, *) = \argmax_{\Bphi \in \DM_t, \est \in \CM_{t-1}} \langle\Bphi,  \est \rangle.
\]
Due to $\sup_{\Bphi \in \cup_{t \ge 0}\DM_t} |\langle \Bphi, \Btheta^* \rangle |  \le R:=LB$, it follows that
\begin{align*}
	\mathrm{Reg}(T) &:= \sum_{t=1}^T \left[  \sup_{\Bphi \in \DM_t} \langle \Bphi, \Btheta^* \rangle - \langle \Bphi_t, \Btheta^* \rangle  \right]\\
	&\le \sum_{t=1}^T \left[ \sup_{\Bphi \in \DM_t, \est \in \CM_{t-1}} 
	\langle \Bphi, \Btheta \rangle - \langle \Bphi_t, \Btheta^* \rangle \right]\\
	&= \sum_{t=1}^T \left[ \sup_{ \est \in \CM_{t-1}} \langle \Bphi_t, \Btheta \rangle - \langle \Bphi_t, \Btheta^* \rangle \right]\\
	&\le \sum_{t=1}^T\| \Bphi_t\|_{\H_{t-1}^{-1}(\est_{t-1})} \cdot  \sup_{ \est \in \CM_{t-1}} \| \est - \est^*\|_{\H_{t-1}}
\end{align*}
Notice that with probability $1-\delta$, $\est^* \in \CM_t$ for all $t \ge 1$, i.e., $\| \est_t - \est^*\|_{\H_t} \le \beta_t$.
Hence, 
\[
\sup_{ \est \in \CM_t} \| \est - \est^*\|_{\H_t}
\le  \sup_{ \est \in \CM_t} \| \est - \est_t\|_{\H_t}+ \| \est_t - \est^*\|_{\H_t} \le 2 \beta_t.
\]
Notice that $\beta_t$ is increasing in $t$ and $w_t = \left\|\frac{\Bphi_t}{\sigma_t}\right\|_{\H_{t-1}^{-1}}$.
Therefore,
\begin{align}
	\label{eq:reg}
	\mathrm{Reg}(T)
	&\le 2\beta_T \sum_{t=1}^T \|\Bphi_t\|_{\H_{t-1}^{-1}}
	= 2 \beta_T \sum_{t=1}^T\sigma_t w_t
	=2 \beta_T \sum_{t=1}^T\sigma_t\min\{1, w_t\}.
\end{align}
The last equality uses $w_t\le 1$ (which is due to $\sigma_t \ge \|\Bphi_{t}\|_{\H_{t-1}^{-1}}/c_0$ and $c_0 \le 1$).
Notice that $\|\Bphi_t\|/\sigma_t  \le \|\Bphi_t\|/\sigma_{\min}  \le L/\sigma_{\min}$.
Then by Lemma~\ref{lem:w-sum}, 
\begin{equation}
	\label{eq:w-sum0}
	\sum_{t=1}^T \min\left\{1, \left\| \frac{\Bphi_t}{\sigma_t}\right\|_{\H_{t-1}^{-1}}^2 \right\}
	= 	\sum_{t=1}^T \min\left\{1,w_t^2\right\} 
	\le 2d \log\left(1 + \frac{T L^2}{d\lambda \sigma_{\min}^2}\right) = 2 \kappa.
\end{equation}

Recall that 
\[
\sigma_t = \max\left\{ \nu_t, \sigma_{\min},\frac{\|\Bphi_{t}\|_{\H_{t-1}^{-1}}}{c_0}, \frac{\sqrt{LB}\|\Bphi_t\|^{\frac{1}{2}}_{\H_{t-1}^{-1}}}{c_1^{\frac{1}{4}}d^{\frac{1}{4}}}
\right\}.
\]
According to what value $\sigma_t$ takes, we decompose $[T]$ into three sets $[T] \subseteq \cup_{i=1}^3 \JM_i$ where
\begin{gather*}
\JM_1 = \{
t \in [T]: \sigma_t \in \{ \nu_t, \sigma_{\min}\}
\},\\
\JM_2 = \left\{
t \in [T]: \sigma_t  = \frac{\|\Bphi_{t}\|_{\H_{t-1}^{-1}}}{c_0} \right\}, \\
\JM_3 = \left\{
t \in [T] : \sigma_t =\sqrt{LB} \frac{\|\Bphi_t\|^{\frac{1}{2}}_{\H_{t-1}^{-1}}}{c_1^{\frac{1}{4}}d^{\frac{1}{4}}}
\right\}.
\end{gather*}

First, it follows that
\begin{align}
	\label{eq:reg1-1}
	\sum_{t \in \JM_1}\sigma_t   \min\left\{ 1,  w_t  \right\}
	&\le 	\sum_{t \in \JM_1} \max\{ \nu_t, \sigma_{\min} \}  \min\left\{ 1,  w_t  \right\} \nonumber \\
	&\le \sum_{t \in [T]} \max\{ \nu_t, \sigma_{\min} \}  \min\left\{ 1,  w_t  \right\} \nonumber \\
	&\overset{(a)}{\le} \sqrt{\sum_{t \in [T]} (\nu_t^2 + \sigma_{\min}^2)} \sqrt{\sum_{t\in[T]}  \min\left\{ 1,  w_t^2  \right\} } \nonumber \\
	&\overset{(b)}{\le} \sqrt{2\kappa} \cdot \sqrt{\sum_{t \in [T]} \nu_{t}^2 + 1} .
\end{align}
Here $(a)$ holds due to Cauchy-Schwarz inequality and $(b)$ uses~\eqref{eq:w-sum0} and $\sigma_{\min} = \frac{1}{\sqrt{T}}$. 

Second, for any $t \in \JM_2$, we have $w_t = \left\|\frac{\Bphi_t}{\sigma_t}\right\|_{\H_{t-1}^{-1}} = c_0 \le 1$. Therefore,
\begin{align}
	\label{eq:reg1-2}
	\sum_{t \in \JM_2}\sigma_t   \min\left\{ 1,  w_t  \right\}
	&=	\sum_{t \in \JM_2}\sigma_t w_t
	= \frac{1}{c_0}\sum_{t \in \JM_2}\sigma_t w_t^2
	\le\frac{\sup_{t \in \JM_2} \sigma_t}{c_0}\sum_{t \in \JM_2}  w_t^2\nonumber\\
	&\le \frac{\sup_{t \in [T]}\|\Bphi_{t}\|_{\H_{t-1}^{-1}} }{c_0^2} \cdot \sum_{t \in \JM_2}  \min\{  1, w_t^2\} \nonumber\\
	&\le \frac{\sup_{t \in [T]}\|\Bphi_{t}\|_{\H_{t-1}^{-1}} }{c_0^2} \cdot \sum_{t \in [T]}  \min\{  1, w_t^2\} \le  \frac{2L\kappa}{c_0^2\sqrt{\lambda}}
\end{align}
where the last inequality uses $\|\Bphi_{t}\|_{\H_{t-1}^{-1}} \le  \frac{1}{\sqrt{\lambda}} \|\Bphi_{t}\| \le \frac{L}{\sqrt{\lambda}}$ for all $t\ge1$ and~\eqref{eq:w-sum0}.

Finally, for any $t \in \JM_3$, we have $L^2B^2w_t^2  = c_1d\sigma_t^2$ due to $w_t^2 = \left\|\frac{\Bphi_{t}}{\sigma_t}\right\|_{\H_{t-1}^{-1}}^2$.
It implies $\sigma_t = LBw_t/ \sqrt{c_1 d}  = LB \min\{1, w_t\}/\sqrt{c_1d}$ with the fact that $w_t \le 1$.
Therefore,
\begin{align}
	\label{eq:reg1-3}
	\sum_{t \in \JM_3}\sigma_t   \min\left\{ 1,  w_t  \right\}
	=	\frac{LB}{\sqrt{c_1d}}\cdot \sum_{t \in \JM_3}   \min\left\{ 1,  w_t^2  \right\}
	\le \frac{LB}{\sqrt{c_1d}}\cdot \sum_{t \in [T]}   \min\left\{ 1,  w_t^2  \right\} \le \frac{2LB\kappa}{\sqrt{c_1d}}.
\end{align}
Plugging~\eqref{eq:reg1-1},~\eqref{eq:reg1-2} and~\eqref{eq:reg1-3} into~\eqref{eq:reg}, we have
\[
\mathrm{Reg}(T)
\le 2\beta_T  \left[
\sqrt{2\kappa} \cdot \sqrt{\sum_{t \in [T]} \nu_{t}^2 + 1} 
+ \frac{2L\kappa}{c_0^2\sqrt{\lambda}}+\frac{2LB\kappa}{\sqrt{c_1d}}
\right].
\]

\subsection{Proof of Lemma~\ref{lem:hessian}}
\label{proof:hessian}
Recall that  $z_t(\Btheta) =\frac{y_t -  \langle \Bphi_t, \Btheta \rangle}{\sigma_t}$.
Direct computation yields that
\[
\nabla^2 L_T(\Btheta)
= \lambda \I  + \sum_{t=1}^T \left(  \frac{\tau_t}{\sqrt{\tau_t^2 + z_t^2(\Btheta)}} \right)^3 \frac{\Bphi_t \Bphi_t^\top}{\sigma_t^2}.
\]
Clearly, for any $\Btheta \in \RB^d$,
\[
\nabla^2 L_T(\Btheta) \preceq  \lambda \I  + \sum_{t=1}^T  \frac{\Bphi_t \Bphi_t^\top}{\sigma_t^2} =  \H_T.
\]
For the other direction, we decompose it into four terms and analyze them respectively.
\begin{align}
	\begin{split}
		\label{eq:hessian}
		\nabla^2 L_T(\Btheta)
		&=  \H_T -  
		\underbrace{
			\sum_{t=1}^T \left[1- \left(  \frac{\tau_t}{\sqrt{\tau_t^2 + z_t^2(\Btheta^*)}} \right)^3\right] \frac{\Bphi_t \Bphi_t^\top}{\sigma_t^2}}_{\H_{1, T}}\\
		& \qquad
		+ 	\underbrace{
			\sum_{t=1}^T \left[\left(  \frac{\tau_t}{\sqrt{\tau_t^2 + z_t^2(\Btheta)}} \right)^3- \left(  \frac{\tau_t}{\sqrt{\tau_t^2 + z_t^2(\Btheta^*)}} \right)^3   \right] \frac{\Bphi_t \Bphi_t^\top}{\sigma_t^2}}_{\H_{2,T}}.
	\end{split}
\end{align}
where $\EB_t[\cdot] = \EB[\cdot|\FM_{t-1}]$ for simplicity.

Since $\nu_t, \Btheta_{t-1} \in \FM_{t-1}$, from Algorithm~\ref{algo:adap}, we have $\sigma_t, w_t, \tau_t \in \FM_{t-1}$.

\paragraph{Analysis of $\H_{1, T}$}
Notice that for any unit norm $\v \in \RB^d$, it follows that
\begin{align*}
	\v^\top \H_{1, T} \v
	&= \sum_{t=1}^T\left[1- \left(  \frac{\tau_t}{\sqrt{\tau_t^2 + z_t^2(\Btheta^*)}} \right)^3\right] \left\langle  \frac{\Bphi_t}{\sigma_t}, \v \right\rangle^2\\
	&\le3 \sum_{t=1}^T\left[1- \frac{\tau_t}{\sqrt{\tau_t^2 + z_t^2(\Btheta^*)}}\right] \left\langle  \frac{\Bphi_t}{\sigma_t}, \v \right\rangle^2\\
	&\le3 \sum_{t=1}^T\left[1- \frac{\tau_t}{\sqrt{\tau_t^2 + z_t^2(\Btheta^*)}}\right] \cdot \sup_{t \in [T]} \left\langle  \frac{\Bphi_t}{\sigma_t}, \v \right\rangle^2\\
	&\le3 \sum_{t=1}^T\left[1- \frac{\tau_t}{\sqrt{\tau_t^2 + z_t^2(\Btheta^*)}}\right] \cdot \sup_{t \in [T]}  \left\| \frac{\Bphi_t}{\sigma_t}\right\|_{\H_T^{-1}}^2  \cdot \v^\top \H_T \v\\
	&\overset{(a)}{\le} 3\sum_{t=1}^T\left[1- \frac{\tau_t}{\sqrt{\tau_t^2 + z_t^2(\Btheta^*)}}\right] \cdot \sup_{t \in [T]}  \left\| \frac{\Bphi_t}{\sigma_t}\right\|_{\H_t^{-1}}^2  \cdot \v^\top \H_T \v\\
	&\overset{(b)}{=} 3\sum_{t=1}^T\left[1- \frac{\tau_t}{\sqrt{\tau_t^2 + z_t^2(\Btheta^*)}}\right] \cdot \sup_{t \in [T]} \frac{w_t^2}{1+w_t^2} \cdot \v^\top \H_T \v,
\end{align*}
where $(a)$ uses $\H_T^{-1} \preceq \H_t^{-1}$ for all $t \in [T]$ and $(b)$ follows from
\[
\left\| \frac{\Bphi_t}{\sigma_t}\right\|_{\H_t^{-1}}^2 
= \frac{\Bphi_t^\top}{\sigma_t} \left( \H_{t-1}^{-1}  - \frac{\H_{t-1}^{-1} \frac{\Bphi_t}{\sigma_t} \frac{\Bphi_t^\top}{\sigma_t}\H_{t-1}^{-1}}{1+\frac{\Bphi_t^\top}{\sigma_t} \H_{t-1}^{-1}\frac{\Bphi_t^\top}{\sigma_t}} \right)\frac{\Bphi_t}{\sigma_t}
= w_t^2 - \frac{w_t^4}{1+w_t^2} =  \frac{w_t^2}{1+w_t^2}.
\]
By the arbitrariness of $\v$, we know that
\begin{equation}
	\label{eq:H1-0}
	\H_{1, T} \preceq
	3 \sum_{t=1}^T\left[1- \frac{\tau_t}{\sqrt{\tau_t^2 + z_t^2(\Btheta^*)}}\right] \cdot  \sup_{t \in [T]}\frac{w_t^2}{1+w_t^2}  \cdot  \H_T .
\end{equation}

Let $X_t = 1-  \frac{\tau_t}{\sqrt{\tau_t^2 + z_t^2(\Btheta^*)}}$.
It is obvious that $0 \le X_t \le 1$.
We then focus on concentration of $\sum_{t=1}^T X_t$.
To that end, we need a variance-aware Bernstein inequality Lemma~\ref{lem:bern} for martingales.
Lemma~\ref{lem:bern} implies that with probability at least $1-\frac{\delta}{T^2}$, we have  
\[ 
\sum_{t=1}^T X_t
\le \sum_{t=1}^T \EB_t X_t + 3 \sqrt{\sum_{t=1}^T \Var[X_t|\FM_{t-1}]  \cdot \log\frac{2KT^2}{\delta}} + 5\log\frac{2KT^2}{\delta}
\]
where $K := 1+\ceil{2\log_2V} $ and $V^2$ is an upper bound satisfying $\sum_{t=1}^T \EB[X_t^2|\FM_{t-1}]  \le V^2$.

First notice that for any $t \ge 1$, we have
\begin{align}
	\begin{split}
		\label{eq:help3}
		\EB_t X_t
		&=1- \EB_t  \frac{\tau_t}{\sqrt{\tau_t^2 + z_t^2(\Btheta^*)}}
		= \EB_t \frac{z_t^2({\Btheta^*})}{\sqrt{\tau_t^2 + z_t^2(\Btheta^*)} (\sqrt{\tau_t^2 + z_t^2(\Btheta^*)}+\tau_t)} \\
		&\le \frac{1}{2\tau_t^2} \EB_t z_t^2(\Btheta^*) \le 
		\frac{b^2}{2\tau_t^2} \le \frac{b^2}{2\tau_0^2} \frac{w_t^2}{1+w_t^2}
	\end{split}
\end{align}
which implies that
\[
\sum_{t=1}^T \EB_t X_t \le  \frac{b^2}{2\tau_0^2} \frac{w_t^2}{1+w_t^2}
\le  	\frac{b^2}{2\tau_0^2}  \sum_{t=1}^T \min\{ 1, w_t^2\}
\le \frac{\kappa b^2}{\tau_0^2}
\]
where the last inequality uses Lemma~\ref{lem:w-sum} and thus
\[
\sum_{t=1}^T \min\{ 1, w_t^2\}
=
\sum_{t=1}^T \min\left\{1, \left\| \frac{\Bphi_t}{\sigma_t}\right\|_{\H_{t-1}^{-1}}^2 \right\} \le 2d \log\left(1 + \frac{T L^2}{d\lambda \sigma_{\min}^2}\right) = 2 \kappa.
\]

Secondly, we have
\begin{align*}
	\Var[X_t|\FM_{t-1}]
	&\le 	\EB[X_t^2|\FM_{t-1}] 
	\le 	\EB_t   \left(1-  \frac{\tau_t}{\sqrt{\tau_t^2 + z_t^2(\Btheta^*)}}\right)^2
	\overset{(*)}{\le}  \frac{1}{4} \frac{\EB_t z_t^2(\Btheta^*)}{\tau_t^2} \le  \frac{b^2}{4\tau_t^2} 
\end{align*}
where $(*)$ uses $1- \frac{\tau_t}{\sqrt{\tau_t^2 + z_t^2(\Btheta^*)}} \le \frac{z_t^2(\Btheta^*)}{2\tau_t \sqrt{\tau_t^2+z_t^2(\Btheta^*)}}$ which is also used in~\eqref{eq:help3}.
As a result, we have
\[
\sum_{t=1}^T 	\Var[X_t|\FM_{t-1}]
\le \sum_{t=1}^T \frac{b^2}{4\tau_t^2}  \le  \frac{b^2}{4\tau_0^2} \sum_{t=1}^T \frac{w_t^2}{1+w_t^2} \le \frac{\kappa b^2}{2\tau_0^2}.
\]
Once requiring $\tau_0^2 \ge 2\kappa b^2$, we have $\sum_{t=1}^T \Var[X_t|\FM_{t-1}] \le 1$ and thus we can set $V=1$ and obtain $K=1$.
Putting them together, if $\tau_0^2 \ge \frac{2\kappa b^2}{\log\frac{2T^2}{\delta}}$, with probability at least $1-\delta$, we have
\begin{align}
	\label{eq:H1-1}
	\sum_{t=1}^T X_t
	& \le \frac{\kappa b^2}{\tau_0^2} + \frac{3b}{\tau_0}  \sqrt{\frac{\kappa \log\frac{2T^2}{\delta}}{2}} + 5\log \frac{2T^2}{\delta} \nonumber \\
	&\le \frac{1}{2}\log\frac{2T^2}{\delta} + \frac{3}{2}\log\frac{2T^2}{\delta} + 5\log \frac{2T^2}{\delta}
	\nonumber \\
	&\le 9\log \frac{2T^2}{\delta} = \frac{1}{12c_0^2}
\end{align}
where the last equation is due to the definition of $c_0$.
Finally, taking a union bound for the last inequality from $T=1$ to $\infty$ and using the fact that $\sum_{t=1}^{\infty} t^{-2} < 2$, we have $\sum_{t=1}^T X_t \le \frac{1}{12c_0^2}$ for all $T \ge 1$ with probability at least $1-2\delta$.
%Since $X_t$ is non-negative, we have $\sum_{s=1}^t X_s \le \frac{1}{12c_0^2}$ for all $t \in [T]$ with probability at least $1-\delta$. 

On the other hand, by the choice of $\sigma_t$, we have $\sigma_t^2 \ge \frac{1}{c_0^2} \cdot \| \Bphi_{t}\|_{\H_{t-1}^{-1}}^2$, which implies
\begin{equation}
	\label{eq:H1-2}\sup_{t \in [T]} \frac{w_t^2}{1+w_t^2} 
	\le \sup_{t \in [T]} w_t^2 \le c_0^2 .
\end{equation}
Plugging~\eqref{eq:H1-1} and~\eqref{eq:H1-2} into~\eqref{eq:H1-0}, we have
\begin{equation}
	\label{eq:H1}
	\H_{1, T} \preceq \frac{1}{4} \H_T.
\end{equation}

\paragraph{Analysis of $\H_{2,T}$}
We first notice that
\begin{align}
	\label{eq:help4}
	&\left|\left(  \frac{\tau_t}{\sqrt{\tau_t^2 + z_t^2(\Btheta)}} \right)^3- \left(  \frac{\tau_t}{\sqrt{\tau_t^2 + z_t^2(\Btheta^*)}} \right)^3   \right|
	\le 3	\left|\frac{\tau_t}{\sqrt{\tau_t^2 + z_t^2(\Btheta)}} -   \frac{\tau_t}{\sqrt{\tau_t^2 + z_t^2(\Btheta^*)}}   \right| \nonumber \\
	& \qquad \le  \frac{3 \tau_t}{\sqrt{\tau_t^2 + z_t^2(\Btheta)} \sqrt{\tau_t^2 + z_t^2(\Btheta^*)}} \frac{|z_t^2(\Btheta)-z_t^2(\Btheta^*)|}{\sqrt{\tau_t^2 + z_t^2(\Btheta)} + \sqrt{\tau_t^2 + z_t^2(\Btheta^*)}}.
\end{align} 
Notice that $z_t(\Btheta) = z_t(\Btheta^*) + \langle \frac{\Bphi_t}{\sigma_t}, \Btheta-\Btheta^*\rangle$.
It then follow that for any $c > 0$
\begin{gather*}
	z_t^2(\Btheta) \le \left(1+\frac{1}{c}\right) z_t^2(\Btheta^*) + (1+c) \left\langle\frac{\Bphi_t}{\sigma_t}, \Btheta-\Btheta^*\right\rangle^2;\\
	z_t^2(\Btheta^*) \le \left(1+\frac{1}{c}\right) z_t^2(\Btheta) + (1+c) \left\langle\frac{\Bphi_t}{\sigma_t}, \Btheta-\Btheta^*\right\rangle^2,
\end{gather*}
By discussing which is larger between $	z_t^2(\Btheta)$ and $z_t^2(\Btheta^*) $, we have
\begin{equation}
	\label{eq:help5}
	|z_t^2(\Btheta)-z_t^2(\Btheta^*)|
	\le \frac{1}{c} \min \left\{ z_t^2(\Btheta), z_t^2(\Btheta^*) \right\}
	+ \left(1+c\right) \left\langle\frac{\Bphi_t}{\sigma_t}, \Btheta-\Btheta^*\right\rangle^2.
\end{equation}
Plugging~\eqref{eq:help5} into~\eqref{eq:help4}, we have that
\begin{align}
	\label{eq:help6}
	&\left|\left(  \frac{\tau_t}{\sqrt{\tau_t^2 + z_t^2(\Btheta)}} \right)^3- \left(  \frac{\tau_t}{\sqrt{\tau_t^2 + z_t^2(\Btheta^*)}} \right)^3   \right| \nonumber \\
	& \qquad	\le \frac{3 \tau_t}{\tau_t^2 +  \min \left\{ z_t^2(\Btheta), z_t^2(\Btheta^*)\right\} } \frac{\frac{1}{c} \min \left\{ z_t^2(\Btheta), z_t^2(\Btheta^*)\right\}}{2\sqrt{\tau_t^2 +  \min \left\{ z_t^2(\Btheta), z_t^2(\Btheta^*)\right\}  }} + \frac{3(1+c)}{2\tau_t^2} \left\langle\frac{\Bphi_t}{\sigma_t}, \Btheta-\Btheta^*\right\rangle^2\nonumber\\
	&\qquad \le \frac{3}{2c} + \frac{3(1+c)}{2\tau_t^2} \left\langle\frac{\Bphi_t}{\sigma_t}, \Btheta-\Btheta^*\right\rangle^2
	\overset{(a)}{\le}  \frac{3}{2c} + \frac{6(1+c)}{\tau_t^2} \frac{L^2B^2}{\sigma_t^2} \nonumber \\
	&\qquad  \le  \frac{3}{2c} + \frac{6(1+c)}{\tau_0^2} \frac{w_t^2L^2B^2}{\sigma_{t}^2}
	\overset{(b)}{\le}   \frac{3}{2c} + \frac{6(1+c)c_1d}{\tau_0^2} 
\end{align}
where $(a)$ uses $\left\langle\frac{\Bphi_t}{\sigma_t}, \Btheta-\Btheta^*\right\rangle \le \left\| \frac{\Bphi_t}{\sigma_t} \right\| \left( \|\Btheta\| + \|\Btheta^*\| \right) \le \frac{2LB}{\sigma_{t}}$ due to $\left\| \Bphi_t \right\| \le L$ and $\Btheta, \Btheta^* \in \mathrm{Ball}_d(B)$ and
$(b)$ uses the following result.
By the definition of $\sigma_{t}$, we have $\sigma_t \ge {\sqrt{LB}\|\Bphi_t\|^{\frac{1}{2}}_{\H_{t-1}^{-1}}}/{c_1^{\frac{1}{4}}d^{\frac{1}{4}}}$ which implies  $\sigma_t^2 \ge \frac{w_t^2L^2B^2}{c_1d}$.
As a result of~\eqref{eq:help6}, by definition of $\H_{3,T}$, we have
\begin{equation}
	\label{eq:H3}
	-  \left( \frac{3}{2c} + \frac{6(1+c)c_1d}{\tau_0^2} \right)  \sum_{t=1}^T \frac{\Bphi_t}{\sigma_t}\frac{\Bphi_t^\top}{\sigma_t}
	\preceq \H_{2,T}
\end{equation}

 \paragraph{Putting pieces together}
Plugging~\eqref{eq:H1} and~\eqref{eq:H3} into~\eqref{eq:hessian}, with probability at least $1-\delta$, for any $T \ge 1$ and for all $\Btheta \in \mathrm{Ball}_d(B)$,  we have 
\begin{align*}
	\nabla^2 L_{T} (\Btheta)
	&\succeq \H_T - \frac{1}{4} \H_T
	-  \left( \frac{3}{2c} + \frac{6(1+c)c_1d}{\tau_0^2}  \right)  \sum_{t=1}^T \frac{\Bphi_t}{\sigma_t}\frac{\Bphi_t^\top}{\sigma_t}\\
	&\succeq \frac{3\lambda}{4} \I + \left( 1 - \frac{1}{4} -\frac{3}{2c} -\frac{6(1+c)c_1d}{\tau_0^2} \right) \sum_{t=1}^T \frac{\Bphi_t}{\sigma_t} \frac{\Bphi_t^\top}{\sigma_t}.
\end{align*}
Notice that $c_1 = \frac{1}{42 \cdot \frac{2T^2}{\delta}}$.
If we set $c=6$ and $ \tau_0\sqrt{\frac{2T^2}{\delta}} \ge \max\{ \sqrt{2\kappa} b , 2\sqrt{d}   \}$, we have
\[
\max \left\{\frac{3}{2c} , \frac{6(1+c)c_1d}{\tau_0^2} \right\} \le \frac{1}{4}.
\]

As a result, we have
\[
\nabla^2 L_{T} (\Btheta)
\succeq \frac{3\lambda}{4} \I + \frac{1}{4} \sum_{t=1}^T \frac{\Bphi}{\sigma_t} \frac{\Bphi^\top}{\sigma_t} \succeq \frac{1}{4} \H_T.
\]

\subsection{Proof of Lemma~\ref{lem:grad}}
\label{proof:grad}
For simplicity, we denote $z_t^* = z_t(\Btheta^*)$ for short.
By triangle inequality, we have
\begin{align}
	\label{eq:grad}
	\|\nabla L_{T}(\Btheta^*) \|_{\H_T^{-1}}
	&\le  \| \lambda \Btheta^* \|_{\H_T^{-1}} + 
	\left\| \sum_{t=1}^T \frac{\tau_t z_t^* }{\sqrt{ \tau_t^2 +  (z_t^*)^2}  } \frac{\Bphi_t}{\sigma_t} \right\|_{\H_T^{-1}} \nonumber \\
	&\le  \| \lambda \Btheta^* \|_{\H_T^{-1}}+ \underbrace{ \left\| \sum_{t=1}^T \frac{\tau_t z_t^* }{\sqrt{ \tau_t^2 +  (z_t^*)^2}  } \frac{\Bphi_t}{\sigma_t} \right\|_{\H_T^{-1}}}_{:=\d_T}.
\end{align}

\paragraph{For the residual term $ \| \lambda \Btheta^* \|_{\H_T^{-1}}$}
Notice that $\H_T \succeq \lambda \I$ and thus $ \H_T^{-1} \preceq \lambda^{-1} \I_d$.
Therefore, $\| \lambda \Btheta^* \|_{\H_T^{-1}}  \le \sqrt{\lambda} B$.

\paragraph{For the self-normalized term $ \| \d_T \|_{\H_T^{-1}}$}
The fact that $\BH_T = \BH_{T-1} + \frac{\Bphi_T \Bphi_T^\top}{\sigma_T^2}$ together with the Woodbury matrix identity implies that
\begin{equation}
	\label{eq:woodbury}
	\BH_T^{-1} =\BH_{T-1}^{-1} - \frac{\BH_{T-1}^{-1}\Bphi_T \Bphi_T^\top\BH_{T-1}^{-1} }{\sigma_T^2 (1+w_T^2)} 
	\quad \text{where} \quad
	w_T^2 :=  \frac{ \Bphi_T^\top \H_{T-1}^{-1} \Bphi_T}{\sigma_T^2}
	= \left\|\frac{\Bphi_T}{\sigma_T}\right\|_{\H_{T-1}}^2.
\end{equation}
Clearly, $w_T$ is $\FM_{T-1}$-measurable and thus is predictable.
By definition of $\d_T$ and~\eqref{eq:woodbury}, 
\begin{align}
	\label{eq:help7}
	\|  \d_T\|_{\H_T^{-1}}^2 
	&= \left(  \d_{T-1} + \frac{\estau_T z_T^*}{\sqrt{ \estau_T^2 +  (z_T^*)^2} } \frac{\Bphi_T}{\sigma_T} \right)^{\top}\BH_{T}^{-1}\left(  \d_{T-1} + \frac{\estau_T z_T^*}{\sqrt{ \estau_T^2 +  (z_T^*)^2} } \frac{\Bphi_T}{\sigma_T} \right) \nonumber \\
	&=\|  \d_{T-1}\|_{\H_{T-1}^{-1}}^2 -  \frac{1}{1+w_T^2} \left(  	\frac{\d_{T-1}^\top  \H_{T-1}^{-1} \Bphi_T}{\sigma_T} \right)^2 \nonumber\\
	& \qquad \qquad+ 
	\frac{2\estau_T z_T^*}{\sqrt{\estau_T^2 +  (z_T^*)^2} } \frac{\d_{T-1}^\top \BH_{T}^{-1}\Bphi_T}{\sigma_T} 
	+ \frac{\tau_T^2 (z_T^*)^2}{\tau_T^2 + (z_T^*)^2} \frac{ \Bphi_T^\top \BH_{T}^{-1}\Bphi_T}{\sigma_T^2} \nonumber \\
	&\le \|  \d_{T-1}\|_{\H_{T-1}^{-1}}^2 +  
	\underbrace{	 \frac{2\estau_T z_T^*}{\sqrt{\estau_T^2 +  (z_T^*)^2} } \frac{\d_{T-1}^\top \BH_{T}^{-1}\Bphi_T}{\sigma_T}  }_{I_1}
	+ \underbrace{  \frac{\tau_T^2 (z_T^*)^2}{\tau_T^2 + (z_T^*)^2} \frac{ \Bphi_T^\top \BH_{T}^{-1}\Bphi_T}{\sigma_T^2}}_{I_2}.
\end{align}
For $I_1$, by~\eqref{eq:woodbury}, we have
\begin{align*}
	I_1 &=  \frac{2\estau_T z_T^*}{\sqrt{\estau_T^2 +  (z_T^*)^2}} \frac{1}{\sigma_T}\d_{T-1}^\top  \left( \BH_{T-1}^{-1} -  \frac{\BH_{T-1}^{-1}\Bphi_T \Bphi_T^\top\BH_{T-1}^{-1} }{\sigma_T^2 (1+w_T^2)}  \right)  \Bphi_T\\
	&= \frac{2\estau_T z_T^*}{\sqrt{\estau_T^2 +  (z_T^*)^2} }
	\frac{1}{1+w_T^2} 
	\frac{\d_{T-1}^\top  \H_{T-1}^{-1} \Bphi_T}{\sigma_T}.
\end{align*}
For $I_2$, we have
\begin{align*}
	I_2  &= \frac{\estau_T^2 (z_T^*)^2}{\estau_T^2 + (z_T^*)^2} 
	\frac{ \Bphi_T^\top \BH_{T}^{-1}\Bphi_T}{\sigma_T^2}\\
	&= 
	\frac{\estau_T^2  (z_T^*)^2}{\estau_T^2 + (z_T^*)^2}  \frac{1}{\sigma_T^2}
	\Bphi_T^\top \left( \BH_{T-1}^{-1} -  \frac{\BH_{T-1}^{-1}\Bphi_T \Bphi_T^\top\BH_{T-1}^{-1} }{\sigma_T^2 (1+w_T^2)}  \right)  \Bphi_T\\
	&= \frac{\estau_T^2 (z_T^*)^2}{\estau_T^2 + (z_T^*)^2}  \left(  w_T^2 - \frac{w_T^4}{1+w_T^2} \right) \\
	&= \frac{\estau_T^2  (z_T^*)^2}{\estau_T^2 + (z_T^*)^2} \frac{w_T^2}{1+w_T^2}.
\end{align*}
Using the equations for $I_1, I_2$ and iterating~\eqref{eq:help7}, we have
\[
\|  \d_T\|_{\H_T^{-1}}^2 
\le \sum_{t=1}^T
\frac{\estau_t z_t^*}{\sqrt{\estau_t^2 +  (z_t^*)^2}}
\frac{2}{1+w_t^2} 	\frac{\d_{t-1}^\top  \H_{t-1}^{-1} \Bphi_t}{\sigma_t}+ 
\sum_{t=1}^T  \frac{\estau_t^2  (z_t^*)^2}{\estau_t^2 + (z_t^*)^2} \frac{w_t^2}{1+w_t^2}.
\]

Recall that
\[
\kappa = d\log\left(1 + \frac{T L^2}{d\lambda \sigma_{\min}^2}\right) .
\]

\begin{lem}
	\label{lem:I1}
	Assume $\EB[(z_t^*)^2|\FM_{t-1}] \le b^2$ for all $t \ge 1$.
	Let $A_{t}$ denotes the event where $\|\d_{n}\|_{\BH_{n}^{-1}} \le \alpha_n$ for all $n \in [t]$.
	With probability at least $1-\delta/2$, we have for all $T \ge 1$,
	\begin{equation*}
		\sum_{t=1}^T
		\frac{2\tau_t z_t^*\1_{A_{t-1}}}{(\tau_t^2 +  (z_t^*)^2)^{1/2} }
		\frac{1}{1+w_t^2} 	\frac{\d_{t-1}^\top  \H_{t-1}^{-1} \Bphi_t}{\sigma_t} \le 4\max_{t \in [T]}\alpha_t  \cdot \left[  \frac{\kappa b^2}{4\tau_0}  + b\sqrt{\kappa \log\frac{2T^2}{\delta}} + \frac{2\tau_0}{3}\log\frac{2T^2}{\delta}\right].
	\end{equation*}
\end{lem}

\begin{lem}
	\label{lem:I2}
	Assume $\EB[(z_t^*)^2|\FM_{t-1}] \le b^2$ for all $t \ge 1$.
	For a fixed $\tau \ge 0$, with probability at least $1-\delta/2$, the follow inequality uniformly holds for all $T \ge 1$,
	\begin{equation*}
		\sum_{t=1}^T  \frac{\tau_t^2  (z_t^*)^2}{\tau_t^2 + (z_t^*)^2} \frac{w_t^2}{1+w_t^2} \le  
		\left[\sqrt{2\kappa} b + \tau_0 \sqrt{\log \frac{2T^2}{\delta}}\right]^2.
	\end{equation*}
\end{lem}

For any $T \ge 1$, we define
\begin{equation}
	\label{eq:alpha}
	\alpha_T = 8\left[ \frac{\kappa b^2}{\tau_0} +  b\sqrt{\kappa\log\frac{2T^2}{\delta}} + \tau_0\log\frac{2T^2}{\delta}   \right].
\end{equation}
As a result of Lemma~\ref{lem:I1} and Lemma~\ref{lem:I2}, with probability at least $1-\delta$, for all $T \ge 1$,
\[
\sum_{t=1}^T
\frac{2\tau_t z_t^*\1_{A_{t-1}}}{\sqrt{\tau_t^2 +  (z_t^*)^2} }
\frac{1}{1+w_t^2} 	\frac{\d_{t-1}^\top  \H_{t-1}^{-1} \Bphi_t}{\sigma_t} 
\le \frac{\alpha_T^2}{2} 
\quad \text{and} \quad
\sum_{t=1}^T  \frac{\tau_t^2  (z_t^*)^2}{\tau_t^2 + (z_t^*)^2} \frac{w_t^2}{1+w_t^2} \le \frac{\alpha_T^2}{2}.
\]
We can conclude that all $\{A_t\}_{t \ge 0}$ is true and thus $	\|  \d_T\|_{\H_T^{-1}}  \le \alpha_T$ for all $T \ge 1$.

\subsection{Proof of Lemma~\ref{lem:I1}}
\label{proof:I1}
\begin{proof}[Proof of Lemma~\ref{lem:I1}]
	We will make use of the Freedman inequality Lemma~\ref{lem:freedman} to prove our result.	
	Recall that $\tau_t = \tau_0\frac{\sqrt{1+w_t^2}}{w_t}$.
	Set $Y_t = \frac{\tau_t z_t^*}{\sqrt{\tau_t^2 +  (z_t^*)^2}}
	\frac{2}{1+w_t^2} 	\frac{\d_{t-1}^\top  \H_{t-1}^{-1} \Bphi_t\1_{A_{t-1}}}{\sigma_t}$ with the event $A_{t-1}$ defined in the lemma.
	%	For easy reference, we assume $\EB[ (z_t^*)^2|\FM_{t-1}] \le b^2$.
	%	Latter on, we will plug in different values of $b$, but $b=1$ in this part.
	For simplicity, we denote $X_t = Y_t - \EB[Y_t|\FM_{t-1}]$.
	Notice that 
	\[
	\left| \frac{\d_{t-1}^\top  \H_{t-1}^{-1} \Bphi_t}{\sigma_t}  \cdot \1_{A_{t-1}} \right|
	\le  \| \d_{t-1} \1_{A_{t-1}} \|_{\H_{t-1}^{-1}} \cdot \left\| \frac{\Bphi_t}{\sigma_t}\right\|_{\H_{t-1}^{-1}}
	\le \alpha_{t-1} w_t.
	\]
	As a result, we have
	\[
	|Y_t| \le \tau_t  \alpha_{t-1} \cdot  \frac{2w_t}{1 + w_t^2} \le 2\tau_0 \alpha_{t-1}
	\quad \text{and thus} \quad
	|X_t| \le  |Y_t| + |\EB[Y_t|\FM_{t-1}]| \le 4\tau_0 \alpha_{t-1}.
	\]
	We also find that 
	\begin{align*}
		\EB[X_t^2|\FM_{t-1}] 
		&\overset{(a)}{\le} \EB[Y_t^2|\FM_{t-1}] 
		= \EB \left[ \left(\frac{2w_t}{1+w_t^2} \right)^2 \|\d_{t-1}\|_{\BH_{t-1}^{-1}}^2\1_{A_{t-1}}
		\frac{\tau_t^2 (z_t^*)^2}{\tau_t^2 +  (z_t^*)^2}
		\bigg|\FM_{t-1}\right]
		\\
		&	\overset{(b)}{\le} 
		\left(\frac{2w_t}{1+w_t^2} \right)^2 \alpha_{t-1}^2 b^2
		\le \min\{ 1, 2 w_t \}^2 \alpha_{t-1}^2 b^2
		\le 4 \min\{ 1, w_t^2 \} \alpha_{t-1}^2 b^2
	\end{align*}
	where $(a)$ uses $\EB(X-\EB X)^2 \le \EB X^2$ for any random variable $X$ and $(b)$ uses $\EB[\varepsilon_t^2|\FM_{t-1}] \le b^2\sigma_t^2$ due to $\EB[ (z_t^*)^2|\FM_{t-1}] \le b^2$.
	
	Notice that $\|\Bphi_t\|/\sigma_t  \le \|\Bphi_t\|/\sigma_{\min}  \le L/\sigma_{\min}$.
	Then by Lemma~\ref{lem:w-sum}, we have
	\begin{equation}
		\label{eq:w-sum}
		\sum_{t=1}^T \min\{1, w_t^2 \} \le 2d \log\left(1 + \frac{T L^2}{d\lambda \sigma_{\min}^2}\right) := 2 \kappa.
	\end{equation}
	Hence, by~\eqref{eq:w-sum},
	\begin{align*}
		\sum_{t=1}^T \EB[X_t^2|\FM_{t-1}]  
		&\le 4\sum_{t=1}^T \min\{ 1, w_t^2 \} \alpha_{t-1}^2b^2
		\le 4 \max_{t \in [T]}\alpha_t^2   \cdot \sum_{t=1}^T \min\{ 1, w_t^2 \} b^2\\
		&\le \max_{t \in [T]}\alpha_t^2 \cdot8db^2\log\left(1 + \frac{T L^2}{d\lambda \sigma_{\min}^2}\right) \le 8\kappa b^2 \cdot \max_{t \in [T]}\alpha_t^2 .
	\end{align*}

	On the other hand, using $\EB[z_t^*|\FM_{t-1}] = 0$ we have
	\begin{align*}
		\left|\EB\left[ \frac{\tau_t z_t^*}{  \sqrt{\tau_t^2 +  (z_t^*)^2 } }\bigg|\FM_{t-1}\right] \right|
		= \left|\EB\left[ \left(\frac{\tau_t }{  \sqrt{\tau_t^2 +  (z_t^*)^2 } }-1\right)  z_t^*\bigg|\FM_{t-1}\right] \right|
		\le  \EB \left[\frac{(z_t^*)^2}{2\tau_t}\bigg|\FM_{t-1}\right] \le  \frac{b^2}{2\tau_t}
	\end{align*}
	which implies
	\begin{align*}
		\left|\sum_{t=1}^T \EB[Y_t|\FM_{t-1}] \right|
		&\le \sum_{t=1}^T\frac{b^2}{2\tau_t} \frac{w_t}{1+w_t^2} \alpha_{t-1}
		\le \frac{b^2}{2\tau_0}\sum_{t=1}^T  \frac{w_t^2}{1+w_t^2}  \alpha_{t-1}\\
		&\le \sup_{t \in [T]} \alpha_t \cdot \frac{b^2}{2\tau}\sum_{t=1}^T \min\{1, w_t^2\}  \le
		\sup_{t \in [T]} \alpha_t \cdot \frac{\kappa b^2}{\tau_0}.
	\end{align*}
	By Freedman inequality in Lemma~\ref{lem:freedman}, it follows that for a given $T$ and $\tau_0$, with probability $1-\frac{\delta}{2T^2}$,
	\begin{align*}
		\sum_{t=1}^T Y_t 
		&\le 
		\left|\sum_{t=1}^T \EB[Y_t|\FM_{t-1}] \right|+
		4\max_{t \in [T]}\alpha_t  \cdot \left[  b \sqrt{\kappa \log\frac{2T^2}{\delta}} + \frac{2\tau_0}{3}\log\frac{2T^2}{\delta}\right]\\
		&\le 4\max_{t \in [T]}\alpha_t  \cdot \left[
		\frac{\kappa b^2}{4\tau_0} + b \sqrt{\kappa \log\frac{2T^2}{\delta}} + \frac{2\tau_0}{3}\log\frac{2T^2}{\delta}
		\right].
	\end{align*}
	Finally, taking a union bound for the last inequality from $T=1$ to $\infty$ and using the fact that $\sum_{t=1}^{\infty} t^{-2} < 2$ complete the proof.
\end{proof}

\subsection{Proof of Lemma~\ref{lem:I2}}
\label{proof:I2}

\begin{proof}[Proof of Lemma~\ref{lem:I2}]
	Set $Y_t =  \frac{\tau_t^2  (z_t^*)^2}{\tau_t^2 + (z_t^*)^2} \frac{w_t^2}{1+w_t^2}$ and $X_t = Y_t - \EB[Y_t|\FM_{t-1}]$.
	Recall that $\tau_t = \tau_0 \frac{\sqrt{1+w_t^2}}{w_t}$.
	Clearly, we have $|Y_t|\le  \tau_t^2\frac{w_t^2}{1+w_t^2} \le \tau_0^2$ and thus $|X_t| = |Y_t-\EB[Y_t|\FM_{t-1}]|
	\le \max\{ |Y_t|, |\EB[Y_t|\FM_{t-1}]| \}
	\le \tau_0^2$.
	We also find that 
	\begin{align*}
		\EB[X_t^2|\FM_{t-1}] 
		&\overset{(a)}{\le} \EB[Y_t^2|\FM_{t-1}] 
		\le \left(\frac{w_t^2}{1+w_t^2}\right)^2\EB\left[  \left( \frac{\tau_t^2  (z_t^*)^2}{\tau_t^2 + (z_t^*)^2}\right)^2\bigg| \FM_{t-1} \right]\\
		&\le \tau_t^2\left( \frac{w_t^2}{1+w_t^2}\right)^2\EB\left[(z_t^*)^2| \FM_{t-1} \right]
		\overset{(b)}{\le} 
		\tau_0^2b^2 \frac{w_t^2}{1+w_t^2}
	\end{align*}
	where $(a)$ uses $\EB(X-\EB X)^2 \le \EB X^2$ for any random variable $X$ and $(b)$ uses $\EB[ (z_t^*)^2|\FM_{t-1}] \le b^2$ due to $\EB[\varepsilon_t^2|\FM_{t-1}] \le b^2 \nu_t^2$.
	Hence, by~\eqref{eq:w-sum}, we have
	\[
	\sum_{t=1}^T \EB[X_t^2|\FM_{t-1}] \le \tau_0^2b^2 \sum_{t=1}^T \frac{w_t^2}{1+w_t^2} \le \tau_0^2b^2  \sum_{t=1}^T \min\{ 1, w_t^2 \} \le 2 \kappa \tau_0^2b^2 .
	\]
	On the other hand,
	\begin{align*}
		\sum_{t=1}^T \EB[Y_t|\FM_{t-1}] 
		&= \sum_{t=1}^T\frac{w_t^2}{1+w_t^2}\EB\left[  \frac{\tau_t^2  (z_t^*)^2}{\tau_t^2 + (z_t^*)^2}\bigg| \FM_{t-1} \right]\\
		&\le \sum_{t=1}^T\frac{w_t^2}{1+w_t^2}\EB\left[ (z_t^*)^2 | \FM_{t-1} \right]\le \sum_{t=1}^T \min\{1, w_t^2\}b^2 
		\le 2\kappa b^2 .
	\end{align*}
	By Lemma~\ref{lem:freedman}, it follows that with probability $1-\frac{\delta}{2 T^2}$,
	\[
	\sum_{t=1}^T Y_t \le \sum_{t=1}^T \EB[Y_t|\FM_{t-1}]  +  2\tau_0 b\sqrt{\kappa\log\frac{2T^2}{\delta}} + \frac{2\tau_0^2}{3}\log\frac{2T^2}{\delta}
	\]
	for a given $T$ and $\tau_0$.
	Putting all pieces together, it follows that with probability $1-\frac{\delta}{2 T^2}$,
	\[
	\sum_{t=1}^T Y_t  \le 
	\left[\sqrt{2\kappa} b+ \tau_0 \sqrt{\log \frac{2T^2}{\delta}}\right]^2.
	\]
	Finally, taking a union bound for the last inequality from $T=1$ to $\infty$ and using the fact that $\sum_{t=1}^{\infty} t^{-2} < 2$ complete the proof.

	%\[
	%\tau = \tau_T  := \sqrt{d \ln\left( \frac{d\lambda + T^2 L^2}{d \lambda} \right) \bigg/ \ln \frac{T^2}{\delta}},
	%\]
\end{proof}

\section{Proof of Theorem~\ref{thm:mdp}}
\label{proof:mdp}

\paragraph{Measurability}
Let $\FM_{h,k}$ denote the $\sigma$-field generated by all random variables up tp and including the $h$-th step and $k$-th episode.
More specifically, let $I_{h, k} = \{ (i, j): i \in [H], j \in [k-1] \ \text{or} \ i \in [h], j=k \}$ denote the set of index pairs up to and including the $h$-th step and $k$-th episode and then $\FM_{h, k} = \sigma \left( \cup_{(i, j) \in I_{h, k} }\left\{ s_{i, j}, a_{i, j}, r_{i, j} \right\} \right).$
We make a convention that $\FM_{0, k} = \FM_{H, k-1}$.
From our algorithm, we know that (i) $Q_h^k, V_h^k, \pi_h^k \in \FM_{H, k-1}$ for any $Q \in \{ \overoptQ, \optQ, \pesQ, \overpesQ \}$ and $V \in \{ \overoptV, \optV, \pesV, \overpesV \}$, and (ii) 
\[
 \Bmu_{h-1, k},\Btheta_{h, k}, \Bpsi_{h, k},\sigma_{h, k},U_{h, k}, J_{h, k}, E_{h, k}, \Bphi_{h, k}, \TBphi_{h, k}, w_{h, k}, \tw_{h, k}, \tau_{h, k}, \ttau_{h, k}, \H_{h, k}, \TH_{h, k} \in \FM_{h, k}.
\]

\subsection{High-Probability Events}
\label{proof:high-prob-event}
Let $\kappa = d \log\left(1 + \frac{K}{d\lambda \sigma_{\min}^2}\right)$.
We first introduce the following high-probability events.

\begin{enumerate}[leftmargin=*]
	\item We define $\BM_{R^2}$ as the event that the following inequalities hold for all $h \in [H]$ and $k \in [K] \cup \{0\}$,
	\[
		\Bpsi_h^* \in \TRM_{h, k} : = \left\{
	\|\Bpsi\| \le W : \|\Bpsi_{h,k}-\Bpsi \|_{\TH_{h, k}^{-1}} \le \beta_{R^2}
	\right\} 
	\]
	where
	\[
		\beta_{R^2} = 128 \left(\frac{\sqrt{\kappa} \sigma_{R^2}}{\sigma_{\min}}+\sqrt{d}\right) \sqrt{\log \frac{2HK^2}{\delta}}+ 5\sqrt{\lambda} W.
	\]
\item We define $\BM_0$ as the event that the following inequalities hold for all $h \in [H]$ and $k \in [K] $,
\begin{gather*}
	\max \left\{
	 \left\| (\Bmu_h^*-\Bmu_{h, k-1}) \BoveroptV_{h+1}^k \right\|_{\H_{h, k-1}}, \left\| (\Bmu_h^*-\Bmu_{h, k-1}) \BoverpesV_{h+1}^k \right\|_{\H_{h, k-1}} 
	\right\}
	 \le \beta_0,\\
\left\| (\Bmu_h^*-\Bmu_{h, k-1}) [\BoveroptV_{h+1}^k]^2 \right\|_{\H_{h, k-1}}  \le \HM \beta_0
\end{gather*}
where
\begin{gather}
		\beta_0 =  \frac{4\HM}{\sigma_{\min}} \sqrt{
			d^3 H \iota_0^2+\log \frac{2H}{\delta}  } + 3\sqrt{d \lambda} \HM \nonumber \\
		\iota_0 = \max\left\{
		 \log \left(1+ \frac{8LK}{\lambda\HM \sqrt{d} \sigma_{\min}^2}\right),\log \left( 1 +\frac{ 32 B^2K^2}{ \sqrt{d}\lambda^3 \HM^2 \sigma_{\min}^4 } \right),  \log \left( 1 + \frac{K}{\lambda\sigma_{\min}^2} \right) 
		\right\}. \label{eq:iota}
\end{gather}
Here we choose $B \ge 3(\beta_R + \beta_V)$ and $L = W + \HM \sqrt{\frac{dK}{\lambda}}$.
	\item We define $\BM_{R}$ as the event that the following inequalities hold for all $h \in [H]$ and $k \in [K]\cup \{0\}$,
\begin{gather}
	\Btheta_h^* \in \RM_{h, k} : = \left\{
	\|\Btheta\| \le W : \|\Btheta_{h,k}-\Btheta_h^* \|_{\H_{h, k}} \le \beta_R
	\right\}, \nonumber \\
	\left| [\HVB_h \HR_h-\VB_h R_h](s_{h, k}, a_{h, k}) \right| \le R_{h, k} :=\beta_{R^2}  \| \TBphi_{h, k}\|_{\TH_{h, k-1}^{k-1}}  + 2\HM \beta_R \|\Bphi_{h, k}\|_{\H_{h, k-1}^{-1}}, \tag{\ref{eq:R}}
\end{gather}
where
\begin{gather*}
	\beta_R = 128 (\sqrt{\kappa}+\sqrt{d}) \sqrt{\log \frac{2HK^2}{\delta}}+ 5\sqrt{\lambda} W.
\end{gather*}
\item We define $\BM_h$ as the event such that for all episode $k \in [K]$, all stages $h \le h' \le H$,
\begin{gather}
			\label{eq:fine-CI}
\max \left\{
	\left\| (\Bmu_{h'}^*-\Bmu_{h', k-1}) \BoveroptV_{h'+1}^k \right\|_{\H_{h', k-1}}, \left\| (\Bmu_{h'}^*-\Bmu_{h', k-1}) \BoverpesV_{h'+1}^k \right\|_{\H_{h', k-1}}\right\} \le  \beta_V
\end{gather}
where
\begin{gather}
	\beta_V = \OM\left(  \sqrt{d} \iota_1^2 +  \sqrt{d \lambda} \HM  \right) \nonumber \\
		\iota_1 = \max\left\{
	\iota_0,  \log \frac{4H K^2}{\delta}, \log\left(
	1 + \frac{4L\sqrt{d^3H}}{\sigma_{\min}}
	\right), \log \left( 1 + \frac{8\sqrt{d^7}HB^2}{\lambda \sigma_{\min}^2} \right)
	\right\}.  \label{eq:iota1}
\end{gather}
For simplicity, we further define $\BM_V := \BM_1$.
\end{enumerate}
Our ultimate goal is to show $\BM_V$ holds with high probability, a target used in previous work~\citep{hu2022nearly,he2022nearly}.
More specifically, we first obtain coarse confidence sets for all parameters in the sense that the confidence radius (that is $\beta_{R^2}$ and $\beta_0$) is loose.
In our analysis,  $\BM_{R^2} \cap \BM_0$ serves as the ‘coarse’ event where the concentration results hold with a larger confidence radius, and $\BM_R \cap \BM_V$ serves as a ‘refined’ event where the confidence radius (that is $\beta_{R}$ and $\beta_V$) is much is tighter.
Our first result is that $\BM_{R^2} \cap \BM_0$ holds with high probability as shown in Lemma~\ref{lem:CI-rewards-variance} and~\ref{lem:CI-Value}.
Their proofs are collected in Appendix~\ref{proof:CI-rewards-variance} and~\ref{proof:CI-Value00}.

\begin{lem}
	\label{lem:CI-rewards-variance}
	If we set 
	\begin{equation}
		\label{eq:ttau}
		\ttau_0  = \max\left\{  \frac{\sqrt{2\kappa} \sigma_{R^2}}{\sigma_{\min}} ,2\sqrt{d}  \right\} \bigg/\sqrt{\log \frac{2HK^2}{\delta}},
	\end{equation}
	the event $\BM_{R^2}$ holds with probability at least $1-3\delta$.
\end{lem}

\begin{lem}
		\label{lem:CI-Value}
		The event $\BM_0$ holds with probability at least $1-3\delta$.
\end{lem}

These coarse confidence sets are then used to estimate variance for the reward functions and value functions.
A key step is to show the adapted variance $\sigma_{h, k}$'s are indeed upper bounds of these variances (that is $[\VB_h R_h](s_{h,k}, a_{h, k}) + [\VB_h V_{h+1}^*](s_{h,k}, a_{h, k})$) for all $h \in [H]$.
A frequently used argument is backward induction. 
That is given the estimation is optimistic at the stage $h+1$, we then show the optimistic estimation is maintained at the stage $h$.
Induction over the stage $h$ would complete the proof.
The following lemma provides estimation error bounds for $[\VB_h R_h](s_{h,k}, a_{h, k})$ and shows that the event $\BM_{R}$ holds with high probability.
Its proof is deferred in Appendix~\ref{proof:CI-rewards}.

\begin{lem}
	\label{lem:CI-rewards}
	If we set 
	\begin{equation}
		\label{eq:tau}
		\tau_0  = \max\{ \sqrt{2\kappa}, 2\sqrt{d}  \} \bigg/\sqrt{\log \frac{2HK^2}{\delta}},
	\end{equation}
	the event $\BM_{R}$ holds with probability at least $1-6\delta$.
\end{lem}

In Lemma~\ref{lem:opt-pes}, we show that that our constructed value functions $\overoptV$ and $\overpesV$ are indeed optimistic and pessimistic estimators of the true value functions under the event defined before.
Its proof is deferred in Appendix~\ref{proof:opt-pes}.

\begin{lem}[Optimism and pessimism]
	\label{lem:opt-pes}
	For any $h \in [H]$, if $\BM_R \cap \BM_{h}$ holds, for any $k \in [K] \cup \{0\}$,
	\[
	\overpesV_{h}^k(\cdot)  \le V_h^*(\cdot) \le \overoptV_{h}^k(\cdot).
	\]
\end{lem}

With the established optimism and pessimism, we can establish upper bounds for the estimation errors of the three terms, namely $[\VB_h V_{h+1}^*](s_{h,k}, a_{h, k})$,  $\left[
\VB_h(\overoptV_{h+1}^k - V_{h+1}^*)\right](s_{h, k}, a_{h, k})$, and $\left[
\VB_h(\overpesV_{h+1}^k - V_{h+1}^*)\right](s_{h, k}, a_{h, k})$ in the following lemmas.
Their proofs are deferred in Appendix~\ref{proof:var0} and~\ref{proof:var1}.

\begin{lem}
	\label{lem:var0}
	On the event $\BM_0 \cap \BM_{h+1}$, it follows that for all $k \in [K]$
	\[
	\left|
	\left[\VB_h V_{h+1}^* - \HVB_h \overoptV_{h+1}^k\right](s_{h, k}, a_{h, k}) 
	\right| \le U_{h, k}
	\]
	where
	\begin{equation}
		\tag{\ref{eq:U}}
		U_{h, k} =
		\min\left\{
		\VM^2, 11\HM \beta_0\cdot \|\Bphi_{h, k}\|_{\H_{h, k-1}^{-1}} + 4\HM \cdot \HPB_{h, k}(\overoptV_{h+1}^k -\overpesV_{h+1}^k)(s_{h, k}, a_{h, k})
		\right\}
	\end{equation}
	with $\HPB_{h, k}(\cdot|s, a) = \Bmu_{h, k-1}^\top \Bphi(s, a)$.
\end{lem}

\begin{lem}
	\label{lem:var1}
	On the event $\BM_0 \cap \BM_R \cap \BM_{h+1}$, it follows that for all $j \le k \le K$
	\[
	\max \left\{
	\left[
	\VB_h(\overoptV_{h+1}^k - V_{h+1}^*)\right](s_{h, j}, a_{h, j})
	, 	\left[
	\VB_h(\overpesV_{h+1}^k - V_{h+1}^*)\right](s_{h, j}, a_{h, j})
	\right\}
	\le E_{h, j}
	\]
	where
	\begin{equation}
 \tag{\ref{eq:E}}
		E_{h, j} =
		\min\left\{
		\HM^2, 2\HM\beta_0 \|\Bphi_{h, j}\|_{\H_{h, j-1}^{-1}}
		+ \HM \cdot	\left[ \HPB_{h, j}(\overoptV_{h+1}^j - \overpesV_{h+1}^j)\right](s_{h, j}, a_{h, j})
		\right\}
	\end{equation}
	with $\HPB_{h, j}(\cdot|s, a) = \Bmu_{h, j-1}^\top \Bphi(s, a)$.
\end{lem}

With the last four lemmas, one can easily prove $\sigma_{h, k}$ indeed serves as an upper bound of the true variance of $V_{h+1}^*$ at stage $h$.
Therefore, by the backward induction, we can prove the following lemma whose proof is in Appendix~\ref{proof:fine-CI}.
\begin{lem}
	\label{lem:fine-CI}
On the event $\BM_0 \cap \BM_{R}$, the event $\BM_V$ holds with probability at least $1-2\delta$.
\end{lem}

\subsection{Regret Analysis}
\label{proof:mdp-regret}
In the previous subsection, we know that with probability at least $1-15\delta$, the event $\BM_V \cap \BM_R$ holds.
Based on Lemma~\ref{lem:opt-pes}, the optimism implies that
\[
\mathrm{Reg}(K) := \sum_{k=1}^K (V_1^*-V_1^{\pi_k})(s_{1, k})
\le
\sum_{k=1}^K (\overoptV_1^k-V_1^{\pi_k})(s_{1, k}).
\]
We then relate the suboptimality gap $ \sum_{k=1}^k (\overoptV_1^k-V_1^{\pi_k})(s_{1, k})$ to the term $ \sum_{k=1}^K \sum_{h=1}^H \|\Bphi_{h, k}\|_{\H_{h, k-1}^{-1}}$ in Lemma~\ref{lem:sub-gap}.
We emphasize that the bound in Lemma~\ref{lem:sub-gap} is much finer than previous bounds (e.g., Lemma B.1 in~\citep{he2022nearly}) in the sense that the rest term is $\TOM(H\HM)$ instead of previous $\TOM(\sqrt{HK}\HM)$.
This is because we adapt a variance-aware Berinstain inequality to relate the variance of $\sum_{k=1}^k (\overoptV_1^k-V_1^{\pi_k})(s_{1, k})$ with its expectations and use a recursion argument to simplify the final expression, while previous work directly apply Azuma-Hoeffding inequality to analyze the concentration of $\sum_{k=1}^k (\overoptV_1^k-V_1^{\pi_k})(s_{1, k})$, which inevitably introduces the additional $\TOM(\sqrt{K})$ dependence.
Its proof is collected in Appendix~\ref{proof:sub-gap}.
\begin{lem}[Suboptimality gap]
	\label{lem:sub-gap}
	With probability at least $1-\delta$, on the event $\BM_R \cap \BM_V$, it follows that
	\begin{gather*}
		\sum_{k =1}^K(\overoptV_{1}^k - V_1^{\pi_k})(s_{1, k})
		\le 6\beta  \sum_{k=1}^K \sum_{h=1}^H \|\Bphi_{h, k}\|_{\H_{h, k-1}^{-1}}   + 38 H \HM  \log\frac{4\ceil{\log_2 HK}}{\delta} 	\	\text{and} \  \\
		\sum_{k=1}^K \sum_{h=1}^H\PB_h(\overoptV_{h+1}^k - V_{h+1}^{\pi_k})(s_{h, k}, a_{h, k}) 
		\le 8H\beta \sum_{k=1}^K \sum_{h=1}^H  \|\Bphi_{h, k}\|_{\H_{h, k-1}^{-1} } + 38 H^2 \HM \log\frac{4\ceil{\log_2 HK}}{\delta}.
	\end{gather*}	
\end{lem}

Using a similar argument, we provide a finer bound for the gap between optimistic and pessimistic value functions $	\sum_{k=1}^K \sum_{h=1}^H\PB_h(\overoptV_{h+1}^k - \overpesV_{h+1}^{k})(s_{h, k}, a_{h, k})$ in Lemma~\ref{lem:opt-pes-gap}.
Its proof is provided in Appendix~\ref{proof:opt-pes-gap}.

\begin{lem}[Gap between optimistic and pessimistic value functions]
	\label{lem:opt-pes-gap}
	With probability at least $1-\delta$, on the event $\BM_V \cap \BM_R$, it follows that
	\[
	\sum_{k=1}^K \sum_{h=1}^H\PB_h(\overoptV_{h+1}^k - \overpesV_{h+1}^{k})(s_{h, k}, a_{h, k}) 
	\le 12H\beta \sum_{k=1}^K \sum_{h=1}^H  \|\Bphi_{h, k}\|_{\H_{h, k-1}^{-1} } + 38H^2 \HM \log\frac{4\ceil{\log_2 HK}}{\delta}.
	\]
\end{lem}

The following issue is to upper bound the term $ \sum_{k=1}^K \sum_{h=1}^H \|\Bphi_{h, k}\|_{\H_{h, k-1}^{-1}}$.
Since the estimation of reward variance concerns the other term $ \sum_{k=1}^K \sum_{h=1}^H \|\TBphi_{h, k}\|_{\TH_{h, k-1}^{-1}}$, we are motivated to analyze them simultaneously via $\sum_{k=1}^K \sum_{h=1}^Hb_{h, k}$ where $b_{h, k} = \max\left\{
\|\Bphi_{h, k}\|_{\H_{h, k-1}^{-1}} ,  \|\TBphi_{h, k}\|_{\TH_{h, k-1}^{-1}} 
\right\}$.
Previous works~\citep{hu2022nearly,he2022nearly} mainly use Cauchy–Schwarz inequality to analyze it and obtain 
\[
\sum_{k=1}^K \sum_{h=1}^Hb_{h, k} \le \sqrt{\left( \sum_{k=1}^K \sum_{h=1}^H \sigma_{h, k}^2 \right) \left( \sum_{k=1}^K \sum_{h=1}^H  \max\{ w_{h, k}^2, \tw_{h, k}^2 \}  \right) } = \TOM\left(  \sqrt{dH}  \cdot \sqrt{ \sum_{k=1}^K \sum_{h=1}^H \sigma_{h, k}^2 }  \right).
\]
where the last equality uses the elliptical potential lemmas in Lemma~\ref{lem:w-sum}.
A standard analysis of the law of total variation would imply $\sqrt{ \sum_{k=1}^K \sum_{h=1}^H \sigma_{h, k}^2 } = \TOM(\sqrt{H^2K})$.
However, this result doesn't satisfy our target for two reasons.
First, due to the use of adaptive Huber regression, our definition of $\sigma_{h, k}$ is more complicated than previous algorithms.
We need a more elaborate analysis to handle the additional terms in the definition of $\sigma_{h, k}$'s.
Second, the previous result considers the worst-case scenario, while our target is to provide a finer variance-aware regret.
Therefore, it is imperative to provide a finer bound for the sum of bonuses $\sum_{k=1}^K \sum_{h=1}^H b_{h, k}$.
We did it in Lemma~\ref{lem:sum-bonus}.

\begin{lem}[Sum of bonuses]
	\label{lem:sum-bonus}
	Set $\lambda = \frac{1}{\HM^2 + W^2}$. 
	Let $\AM_0$ denote the intersection event of Lemma~\ref{lem:sub-gap} and~\ref{lem:opt-pes-gap}.
	With probability at least $1-2\delta$, on the event $\BM_{R} \cap \BM_V \cap \BM_0 \cap \BM_{R^2} \cap \AM_0$, we have
 \begin{align*}
\sum_{k=1}^K \sum_{h=1}^H b_{h, k}  =
&\TOM \left(
\sqrt{d H K  \GM^*}   + H  d^{0.5} K^{0.5} \sigma_{\min}  + \frac{H^{2.5}d^{5.5}\HM^2 + Hd^{1.5} \sigma_{R^2}}{\sigma_{\min}} \right) \\ 
& \qquad\qquad \TOM \left(+  H^3 d^{4.5} \HM +  H d^{0.5} \sigma_{R} + Hd^{1.5}
\right).
 \end{align*}
	where $\TOM(\cdot)$ ignores constant factors and logarithmic dependence.
\end{lem}
We emphasize that Lemma~\ref{lem:sum-bonus} is perhaps the most technical lemma in our paper. 
To address the difficulty mentioned early, we divide the full index set $\IM := [H] \times [K]$ into three disjoint subsets $\IM = \cup_{i=1,2,3} \JM_{i}$ according to which value $\sigma_{h, k}$ takes (given $\sigma_{h, k}$ is the maximum value among five quantities).
For those indexes in $\JM_1$ where the bonuses are small enough, we still use the Cauchy–Schwarz inequality to bound $\sum_{(h, k) \in \JM_1} b_{h, k} \le \TOM\left(  \sqrt{dH}  \cdot \sqrt{ \sum_{(h, k) \in \IM}\sigma_{h, k}^2 }  \right)$.
This sum-of-squared-bonus quantity involves $\sum_{(h, k) \in \IM} E_{h, k}$ and $\sum_{(h, k) \in \IM} J_{h, k}$ which we then pay additional efforts to analyze.
For those indexes in $\JM_2$ or $\JM_3$ where the bonuses are relatively large, we directly analyze $\sum_{(h, k) \in \JM_2 \cup \JM_3} b_{h, k}$.
Thanks to the particular structure, $\sum_{(h, k) \in \JM_2 \cup \JM_3} b_{h, k}$ contributes to the non-leading term in the final bound.
Putting pieces together, we complete the proof. A formal proof can be found in Appendix~\ref{proof:sum-bonus}.

At the end of the subsection, we summarize the proof in a few lines.
	\begin{align*}
	\mathrm{Reg}(K) &= \sum_{k=1}^K (V_1^*-V_1^{\pi_k})(s_{1, k})
	\overset{(a)}{\le}\sum_{k=1}^K (\overoptV_1^k-V_1^{\pi_k})(s_{1, k})\\
	& \overset{(b)}{\le}3\beta  \sum_{k=1}^K \sum_{h=1}^H \|\Bphi_{h, k}\|_{\H_{h, k-1}^{-1}}   + 38 H \HM  \log\frac{4\ceil{\log_2 HK}}{\delta} \\
	& \overset{(c)}{\le} 3\beta  \sum_{k=1}^K \sum_{h=1}^Hb_{h, k}  + 38 H \HM  \log\frac{4\ceil{\log_2 HK}}{\delta} \\
	& \overset{(d)}{=}
	\TOM\left(		
	d\sqrt{H K  \GM^*}   + H  d K^{0.5} \sigma_{\min}  + \frac{H^{2.5}d^{6}\HM^2 + Hd^2 \sigma_{R^2}}{\sigma_{\min}} +  H^3 d^{5} \HM +  H d \sigma_{R} + Hd^2
	\right)
\end{align*}
where $(a)$ follows from the optimism result in Lemma~\ref{lem:opt-pes}, $(b)$ follows from the suboptimality gap result in Lemma~\ref{lem:sub-gap}, $(c)$ uses $b_{h, k} = \max\left\{
\|\Bphi_{h, k}\|_{\H_{h, k-1}^{-1}} ,  \|\TBphi_{h, k}\|_{\TH_{h, k-1}^{-1}} 
\right\}$, and $(d)$ follows from sum-of-bonus result in Lemma~\ref{lem:sum-bonus} and $\beta = \beta_R + \beta_V= \TOM(\sqrt{d})$.

\section{Proof of Theorem~\ref{thm:space}}
\label{proof:space-mdp}
\begin{proof}[Proof of Theorem~\ref{thm:space}]
	We consider the two complexities respectively.
	\paragraph{Space Complexity}
	First, in order to perform AdaOFUL, VARA needs to store all seen rewards and feature vectors (i.e., $\Bphi_{h,k}, \TBphi_{h,k}$), which is required by all RL/bandit algorithms robust to heavy-tailed rewards~\citep{shao2018almost,xue2021nearly,zhuang2021no}.
	AdaOFUL also keeps all robustification parameters $\tau_{h,k}, \ttau_{h, k}$.
	It then incurs $\OM(HKd)$ space storage in total.
 
	Second, due to the rare-switching technique, one can show that $\overoptQ_h^k$ (or $\overpesQ_h^k$) is the minimum (or maximum) of at most $\TOM(dH)$ temporary optimistic (or pessimistic) functions (see Lemma~\ref{lem:rare-update}). 
	It means that we need to store at most $\TOM(dH)$ different versions of $\Btheta_{h,k-1}, \Bmu_{h, k-1}\V_{h+1}^k, \H_{h,k-1}$'s.
	This incurs $\OM(d^3H^2)$ space cost.
 
	Last, for all $(h, k) \in [H] \times [K]$, we need to trace $\{\Bphi(s_{h, k}, a)\}_{a \in \AM}$ to evaluate each $\Bmu_{h, k}\V = \H_{h, k}^{-1} \sum_{j=1}^k \sigma_{h, j}^{-2} \Bphi_{h, j}V(s_{h+1, k})$ for $\V \in \{ \BoveroptV_{h+1}^k, [\BoveroptV_{h+1}^k]^2, \BoverpesV_{h+1}^k\}$, which takes $\OM(d|\AM|HK)$ space.
 
	To sum up, VARA takes $\OM(d^3H^2 + d|\AM|HK)$ space.
	
	\paragraph{Computational Complexity}
	% The computational complexity of the original LSVI-UCB++~\citep{jin2020provably} is 
	First, we use Nesterov accelerated method to compute each $\Btheta_{h, k}$.
	Since the loss function in~\eqref{eq:theta_hk} is $\lambda$-strongly convex and $\left(\lambda + \frac{K}{\sigma_{\min}^2}\right)$-smooth, the computational cost for each $\Btheta_{h, k}$ is $\TOM\left(d \sqrt{1+ \frac{K}{\lambda (\sigma_{\min}^*)^2}}\right)=\TOM(\max\{d, H^{-3/4}d^{-3/2}K^{3/4}\})$ and the total cost is $\TOM(HK(d + H^{-3/4} d^{-3/2}K^{3/4}))$.
 We emphasize that we don't need to compute $\Btheta_{h, k}$ exactly.
		It suffices to terminate at a solution $\widehat{\Btheta}_{h, k}$ once its accuracy satisfies $\|\widehat{\Btheta}_{h, k}-\Btheta_{h, k}\|_{\H_{h, k}} \le \sqrt{d}$.
		The iteration complexity is proportional to the root of the conditional number, i.e., $\TOM(\max\{1, d^{-7/4}K^{3/4}\})$.
		Since each iteration takes $\OM(d)$ operation, the computation complexity is $\TOM(\max\{d, d^{-3/4}K^{3/4}\})$.

	Second, each time when updating the value function, we take the minimum over at most $\TOM(dH)$ quadratic functions.
	Moreover, the Sherman-Morrison formula computes $\H_{h,k}^{-1}$ and its products with any vectors, which takes $\OM(d^2)$ operations.
	As a result, it needs $\TOM(d^3H)$ to evaluate the updated $Q_{h, k}(s, a)$ for a given pair $(s, a)$.
	Hence, computing $Q_{h, k}(s_{h,k}, \cdot)$, choosing $a_{h, k} = \argmax_{a \in \AM} Q_{h, k}(s_{h,k}, a)$, and estimating the variance $\sigma_{h, k}$ lead to $\TOM(d^3H^2|\AM|)$ computational complexity for each episode.
 
	Last, note $\Bmu_{h, k}\V = \H_{h, k}^{-1} \sum_{j=1}^k \sigma_{h, j}^{-2} \Bphi_{h, j}V(s_{h+1, k})$ for any value function $V(\cdot)$.
	If $V$ remains unchanged, we only need to compute the new term $\sigma_{h, k}^{-2} \Bphi_{h, k}V(s_{h+1, k})$, which has an $\TOM(d^3H|\AM|)$ complexity each time.
	If $V$ changes to $V'$, we need to recalculate $\Bmu_{h, k}\V'$, which has an $\TOM(d^3H|\AM|K)$ complexity each time.
	Combining the computational complexity for all horizons and noticing that the number of episodes that trigger the updating criterion is at most $\TOM(dH)$, VARA has a running time of $\TOM(d^4|\AM|H^3K + HK(d + H^{-3/4} d^{-3/2}K^{3/4}))$.
	In terms of the dependence on $K$, it is slightly worse than~LSVI-UCB++'s $\TOM(d^4|\AM|H^3K)$ since the adaptive Huber regression doesn't have a closed-form solution, but is better than LSVI-UCB's $\TOM(d^2|\AM|HK^2)$ due to the rare-switching mechanism.
\end{proof}

\section{Omitted lemmas in Section~\ref{proof:mdp}}
\label{proof:lemmas-mdp}

\subsection{Proof of Lemma~\ref{lem:CI-rewards-variance}}
\label{proof:CI-rewards-variance}
\begin{proof}[Proof of Lemma~\ref{lem:CI-rewards-variance}]	
	The proof idea of Lemma~\ref{lem:CI-rewards-variance} is similar to that of Theorem~\ref{thm:heavy} except for the following changes.
	First, $\TBphi_{h, k} = \TBphi(s_{h,k}, a_{h, k}) \in \RB^{d}$ is instead the feature vector.
	Second, in the particular setting, we should respectively replace $L, B, T, \delta$ therein with $1, W, K, \delta/H$ defined here and redefine $c_0, c_1$ as $c_0 = \frac{1}{6\sqrt{3\log\frac{2HK^2}{\delta}}}, c_1 = \frac{1}{42 \cdot \frac{2HK^2}{\delta}}$ respectively.
		Third, by the choice of $\sigma_{h,k}$, we have $\sigma_{h,k}^2 \ge \left( \frac{W}{\sqrt{c_1 d}} + \HM d^{2.5}H \right) b_{h, k}\ge  \frac{W}{\sqrt{c_1 d}}{\|\TBphi_{h, k}\|}_{\TH_{h, k-1}^{-1}} $, which implies that $ \frac{W^2 \tw_{h, k}^2}{\sigma_{h, k}^2} \le c_1 d$.
	Similarly, due to $\sigma_{h, k}^2 \ge  c_0^{-2} \|\TBphi_{h, k}\|^2_{\TH_{h-1, k}^{-1}}$, we have $\tw_{h, k}^2 \le c_0^2$.
	Last, for simplicity, we define $\varepsilon_{h, k} =  \frac{r_{h, k}^2 -  \langle \TBphi_{h, k}, \Bpsi_h^* \rangle }{\sigma_{h, k}}$ and $\GM_{h, k} = \sigma(\FM_{h-1, k} \cup \left\{ s_{h, k}, a_{h, k} \right\})$.
	Then, we have $\varepsilon_{h, k}  \in \FM_{h, k}$, $\EB[\varepsilon_{h, k} |\GM_{h, k}] = 0$ and $\Var[\varepsilon_{h, k} |\GM_{h, k}] \le \left(\frac{\sigma_{R^2}}{\sigma_{\min}} \right)^2:= b^2$.
	Theorem~\ref{thm:heavy} concerns the case where $b=1$, however, its proof considers the general case where $b$ can be arbitrary.
	As a result, by a similar argument in Appendix~\ref{proof:bandit} (which is doable due to the four conditions mentioned above), once setting $\ttau_0 \sqrt{\log \frac{2HK^2}{\delta}} = \max \left\{ 
	\sqrt{2\kappa} b, 2 \sqrt{d} \right\}$, with probability at least $1-3\delta$, we have for all $h \in [H]$ and $k \in [K]$, $\| \Bpsi_{h, k} - \Bpsi_h^*\|_{\TH_{h,k}} \le \beta_{R^2}$, that is the event $\BM_{R^2}$ holds.
\end{proof}

\subsection{Proof of Lemma~\ref{lem:CI-Value}}
\label{proof:CI-Value00}
	We will make use of the following general result frequently.
		The proof is quite standard~\citep{jin2020provably,wagenmaker2021first,hu2022nearly}.
		We provide proof in Appendix~\ref{proof:CI-Value} for completeness.
	\begin{lem}
		\label{lem:value-ci}
		Fix any $h \in [H]$.
		Consider a specific value function $f(\cdot)$ which satisfies 
		\begin{enumerate}[leftmargin=*,label=(\roman*)]
			\item $\sup_{s \in \SM}|f(s)| \le C_0$;
			\item $f \in \VM$ where $\VM$ is a class of functions with $\NM(\VM, \varepsilon)$ the $\varepsilon$-covering number of $\VM$ with respective to the distance $\mathrm{dist}(f, f'):= \sup_{s \in \SM} |f(s)-f'(s)|$.
		\end{enumerate}
		We assume there exists a deterministic $C_\sigma > 0$ and $\AM_{h, k}$ (which is $\FM_{h, k}$-measurable) such that $\AM_{h, k}  \subseteq \left\{ \sigma_{h, k}^2 \ge (\sV_h f)(s_{h,k}, a_{h, k})/ C_\sigma^2 \right\}$ for all $k \in [K]$.
		Let $\Bmu_{h, k}$ be defined~\eqref{eq:mu_hk} and $\sigma_{h, k}, \H_{h, k}$ be defined in our algorithm.
		Under any of the following conditions, with probability at least $1-\delta/H$, it follows for all $k \in [K] \cup \{0\}$, 
		\begin{equation}
			\label{eq:ci0}
			\Bmu_h^* \in  \left\{
			\Bmu: \left\| (\Bmu-\Bmu_{h, k}) \Bf\right\|_{\H_{h, k}} \le \beta
			\right\}.
		\end{equation}
		\begin{enumerate}[leftmargin=*,label=(\roman*)]
			\item If $f(\cdot)$ is a deterministic function and $\cap_{k \in [K]}\AM_{h, k}$ is true,~\eqref{eq:ci0} holds with
			\[
			\beta =  8C_\sigma \sqrt{d \log \left( 1 + \frac{K}{\sigma_{\min}^2d\lambda} \right) \log \frac{4H K^2}{\delta} } + \frac{8C_0}{ d^{2.5}H}\log\frac{4HK^2}{\delta} + \sqrt{d \lambda} C_0.
			\]
			\item If $f(\cdot)$ is a random function and $\cap_{k \in [K]}\AM_{h, k}$ is true,~\eqref{eq:ci0} holds with
			\[
			\beta = 8C_\sigma \sqrt{d \log \left( 1 + \frac{K}{\sigma_{\min}^2d\lambda} \right) \log \frac{4H K^2N_0}{\delta} } + \frac{8C_0}{ d^{2.5}H \HM} \log\frac{4HK^2N_0}{\delta}   + 3\sqrt{d \lambda} C_0
			\]
			where $N_0 = |\NM(\VM, \varepsilon_0)|$ and $\varepsilon_0 = \min \left\{C_\sigma \sigma_{\min}, \frac{\lambda C_0 \sqrt{d}}{K}\sigma_{\min}^2\right\} $.
			\item  If $f(\cdot)$ is a random function,~\eqref{eq:ci0} holds with
			\begin{gather*}
				\beta =  \frac{2C_0}{\sigma_{\min}} \sqrt{d\log \left( 1 + \frac{K}{\sigma_{\min}^2d\lambda} \right)  + \log \frac{N_1}{\delta} } + 3\sqrt{d \lambda} C_0.
			\end{gather*}
			where $N_1 = |\NM(\VM, \varepsilon_1)|$ and $\varepsilon_1 = \frac{\lambda C_0 \sqrt{d}}{K}\sigma_{\min}^2$.
		\end{enumerate}
	\end{lem}
		Using the last item suffices to prove Lemma~\ref{lem:CI-Value}.
		\begin{proof}[Proof of Lemma~\ref{lem:CI-Value}]
			Let $\VM^{+}$ denote the class of optimistic value functions mapping from $\SM$ to $\RB$ with the parametric form given in~\eqref{eq:function-pos} and $\VM^{-}$ the class of pessimistic value functions with the parametric form given in~\eqref{eq:function-pes}.
			By Lemma~\ref{lem:covering-1} and Lemma~\ref{lem:rare-update}, 
			\begin{align}
				\label{eq:covering}
				\log \NM(\VM^\pm, \varepsilon) 
				&\le \left[d \log \left(1+ \frac{4L}{\varepsilon}\right) + d^2 \log \left( 1 +\frac{ 8 d^{1/2} B^2}{\lambda \varepsilon^2} \right)\right]
			\end{align}
			where $B \ge \beta_0$ and $L = W + \HM \sqrt{\frac{dK}{\lambda}}$.
			\begin{enumerate}[label=(\roman*)]
				\item Let $\Bf =  \BoveroptV_{h+1}^k$.
				One can find that $\Bf \in \VM_f^+$ with parameter $L  = W + \frac{K \HM}{\lambda \sigma_{\min}^2}$.
				To plug in Lemma~\ref{lem:value-ci}, we first specify the parameters defined therein.
				We have $\|\Bf\|_{\infty} \le  C_0 = \HM$ and $\varepsilon_1 =  \frac{\lambda \HM \sqrt{d}}{K} \sigma_{\min}^2$. 
				By~\eqref{eq:covering}, it follows that
				\begin{align*}
							&	\log \NM(\VM^{+}, \varepsilon_1) \\
					&\le \left[ d \log \left(1+ \frac{4LK}{\lambda\HM \sqrt{d} \sigma_{\min}^2}\right) + d^2 \log \left( 1 +\frac{ 8  B^2K^2}{ \sqrt{d}\lambda^3 \HM^2\sigma_{\min}^4 } \right) \right] \cdot dH \log_2 \left( 1 + \frac{K}{\lambda\sigma_{\min}^2} \right)\\
					&\le\frac{2}{\log 2} d^3 H  \iota_0^2 \le 3 d^3 H\iota_0^2,
				\end{align*}
				By the third condition of Lemma~\ref{lem:value-ci}, with probability at least $1-\frac{\delta}{2H}$, $\left\| (\Bmu_h^*-\Bmu_{h, k-1}) \BoptV_{h+1}^k \right\| \le \beta_0$ for all $k \in [K]$.
				Similarly, we can also show that with probability at least $1-\frac{\delta}{2H}$, $\left\| (\Bmu_h^*-\Bmu_{h, k-1}) \BpesV_{h+1}^k \right\| \le \beta_0$ for all $k \in [K] $.
				Putting them together finishes the proof.
				\item The analysis on $\BoverpesV_{h+1}^k$ is similar to (i).
				\item
				 The analysis on $[\BoveroptV_{h+1}^k]^2$ is similar to (i) except for the following two changes. 
				First, $C_0 = \HM^2$ and $\varepsilon_1' =  \frac{\lambda \HM^2 \sqrt{d}}{K} \sigma_{\min}^2$.
				Second, with $[\VM^+]^2 = \{ f^2: f \in \VM^+ \}$, we have $[\overoptV_{h+1}^k]^2 \in [\VM^+]^2$ and
				\begin{align*}
					\log \NM([\VM^+]^2, \varepsilon_1') 
					&\overset{(a)}{\le} \log \NM(\VM^+, \frac{\varepsilon_1'}{2\HM}) 
					\le  \log \NM(\VM^+, \frac{\varepsilon_1}{2}) \le  3 d^3 H\iota_0^2.
				\end{align*}
				Here $(a)$ uses the fact that the $\frac{\varepsilon_1'}{2\HM}$-cover of $\VM^+$ is a $\varepsilon_1$-cover of $[\VM^+]^2$ (which is also supported by Lemma~\ref{lem:covering-square}).
			\end{enumerate}
		\end{proof}

\subsection{Proof of Lemma~\ref{lem:value-ci}}
\label{proof:CI-Value}
\begin{proof}[Proof of Lemma~\ref{lem:value-ci}]
	Since the case of $k=0$ is trivial, we focus on $k \in [K]$.
	By definition,
	\begin{align*}
		\Bmu_{h,k}  
		&= \H_{h, k}^{-1} \sum_{j=1}^k \sigma_{h, j}^{-2} \Bphi_{h, j} \Bdelta(s_{h+1, j})^\top 
		=\H_{h, k}^{-1} \sum_{j=1}^k \sigma_{h, j}^{-2} \Bphi_{h, j} \left(\Bphi_{h, j}^\top \Bmu_h^* - \Beps_{h, j}  \right)^\top\\
		&=\Bmu_h^* - \lambda \H_{h, k}^{-1} \Bmu_h^* -\H_{h, k}^{-1} \sum_{j=1}^k \sigma_{h, j}^{-2} \Bphi_{h, j}\Beps_{h, j}^\top.
	\end{align*}
	By the triangle inequality, it follows that
	\begin{align*}
		\left\| (\Bmu_h^*-\Bmu_{h, k}) \Bf\right\|_{\H_{h, k}}
		&\le  \lambda \| \H_{h, k}^{-1} \Bmu_h^*  \Bf \|_{\H_{h, k}}
		+ \left\|\H_{h, k}^{-1} \sum_{j=1}^k \sigma_{h, j}^{-2} \Bphi_{h, j}\Beps_{h, j}^\top \Bf\right\|_{\H_{h, k}}\\
		&=\lambda \| \Bmu_h^* \Bf\| _{\H_{h, k}^{-1}} +  \left\| \sum_{j=1}^k \sigma_{h, j}^{-2} \Bphi_{h, j}\Beps_{h, j}^\top \Bf\right\|_{\H_{h, k}^{-1}}\\
		&\le  \sqrt{d \lambda} C_0 +  \left\| \sum_{j=1}^k \sigma_{h, j}^{-2} \Bphi_{h, j}\Beps_{h, j}^\top \Bf\right\|_{\H_{h, k}^{-1}}
	\end{align*}
	where the last inequality uses $\|\Bmu_h^* \Bf\| \le \sqrt{d} C_0$.
	\begin{itemize}[leftmargin=*]
		\item Assume $f(\cdot)$ is a deterministic function.
		To evoke Lemma~\ref{lem:self-bern}, we set $\GM_j = \FM_{h, j}, \x_j = \sigma_{h, j}^{-1} \Bphi_{h, j}, \eta_j = \sigma_{h, j}^{-1} \Beps_{h, j}^\top \Bf  \cdot 1_{\AM_{h, j}}$ and $\Z_k = \lambda\I + \sum_{j=1}^k \sigma_{h, j}^{-2} \Bphi_{h, j}\Bphi_{h, j}^\top= \H_{h, k}$.
		Here $1_{\AM}$ is the indicator function of the event $\AM$.
		%	Since $f(\cdot)$ is deterministic, $\AM_{h, k} \in \FM_{h, k}$.
		
		Clearly $\x_j \in \GM_j, \EB[\eta_j|\GM_j] = 0$ and $\EB[\eta_j^2|\GM_j] \le C_\sigma^2$.
		We also have $\|\x_j\| \le \sigma_{\min}^{-1}, |\eta_j| \le  2C_0 \sigma_{h, j}^{-1}$ and $ \|\x_j\|_{\Z_{j-1}} = w_{h, j}$.
		As a result, $|\eta_j| \min \left\{1, \|\x_j\|_{\Z_{j-1}} \right\} \le 2C_0 \frac{w_{h, j}}{\sigma_{h, j}} \le \frac{2C_0}{\HM d^{2.5}H}$ where the last inequality uses $\sigma_{h, j}^2 \ge \HM d^{2.5} H \|\Bphi_{h, j}\|_{\H_{h, k-1}^{-1}}$ (which is equivalent to $ \frac{w_{h, j}}{\sigma_{h, j}} \le (d^{2.5}H\HM)^{-1}$).
		By Lemma Lemma~\ref{lem:self-bern}, it follows that with probability $1-\frac{\delta}{H}$, for all $k \in [K]$,
		\begin{align*}
			&\left\| \sum_{j=1}^k \sigma_{h, j}^{-2} \Bphi_{h, j}\Beps_{h, j}^\top \Bf 1_{\AM_{h, j}}\right\|_{\H_{h, k}^{-1}}
			= \left\| \sum_{j=1}^k \x_j \eta_j \right\|_{\Z_k^{-1}} \\
			&\le 8C_\sigma \sqrt{d \log \left( 1 + \frac{K}{\sigma_{\min}^2d\lambda} \right) \log \frac{4H K^2}{\delta} } + \frac{8C_0}{\HM d^{2.5}H}\log\frac{4HK^2}{\delta}.
		\end{align*}
		Finally, on the event $\cap_{k \in [K]}\AM_{h, k}$, we will have all the indicator functions equal to one.

		\item  If $f(\cdot)$ is a random function, we would use a covering argument to handle the possible correlation between $f(\cdot)$ and history data, which would unfortunately enlarge $\beta$.
		
		Denote the $\varepsilon_0$-net of $\VM$ by $\NM(\VM, \varepsilon_0)$ where $\varepsilon_0 = \min \left\{C_\sigma \sigma_{\min}, \frac{\lambda C_0 \sqrt{d}}{K}\sigma_{\min}^2\right\} $.
		Hence, for any $f \in \VM$, there exists $\bar{\Bf} \in\NM(\VM, \varepsilon_0)$ such that $\|\bar{\Bf} - \Bf\|_{\infty} = \sup_{s \in \SM}|f(s)-\bar{f}(s)| \le \varepsilon_0$. 
		Then,
		\begin{align*}
			\left\| \sum_{j=1}^k \sigma_{h, j}^{-2} \Bphi_{h, j}\Beps_{h, j}^\top \Bf\right\|_{\H_{h, k}^{-1}}
			&\le \underbrace{	\left\| \sum_{j=1}^k \sigma_{h, j}^{-2} \Bphi_{h, j}\Beps_{h, j}^\top \bar{\Bf}\right\|_{\H_{h, k}^{-1}}}_{(I)}
			+ \underbrace{	\left\| \sum_{j=1}^k \sigma_{h, j}^{-2} \Bphi_{h, j}\Beps_{h, j}^\top (\Bf-\bar{\Bf})\right\|_{\H_{h, k}^{-1}}}_{(II)}
			.
		\end{align*}
		For the term $(II)$, due to $\|\Bphi_{h, j}\| \le 1$ and $|\Beps_{h, j}^\top (\Bf-\bar{\Bf})| \le \|\Beps_{h, j}\|_1 \|\Bf-\bar{\Bf}\|_{\infty} \le 2\varepsilon_0$, we have
		\[
		\left\| \sum_{j=1}^k \sigma_{h, j}^{-2} \Bphi_{h, j}\Beps_{h, j}^\top (\Bf-\bar{\Bf})\right\|_{\H_{h, k}^{-1}}
		\le \frac{2K\varepsilon_0}{\sigma_{\min}^2\sqrt{\lambda}} \le 2 \sqrt{d \lambda} C_0.
		\]
		For the term $(I)$, we define $\VM_{h,k} = \left\{ f' \in \VM:4 C_\sigma^2\sigma_{h, k}^2 \ge (\sV_h f')(s_{h,k}, a_{h, k})\right\}$.
		Since the definition of $\VM_{h,k} $ involves only $\sigma_{h, k}, s_{h, k}, a_{h, k} \in \FM_{h, k}$, for any fixed function $f \in \VM$, $1_{f \in \VM_{h,k}}$ is $\FM_{h, k}$-measurable.
		On the event $\AM_{h, k}$, by definition of $\varepsilon_0$,
		\[
		(\sV_h \bar{f} )(s_{h,k}, a_{h, k})
		\le 2  (\sV_h f )(s_{h,k}, a_{h, k}) + 2 (\sV_h (\bar{f}-f))(s_{h,k}, a_{h, k})
		\le 2 C_\sigma^2 \sigma_{h, k}^2 + 2\varepsilon_0^2 \le 4 C_\sigma^2 \sigma_{h, k}^2.
		\]
		Hence, $\AM_{h, k}  \subseteq \left\{ \sigma_{h, k}^2 \ge (\sV_h f)(s_{h,k}, a_{h, k})/ C_\sigma^2 \right\} \subseteq \{ \exists \bar{f} \in  \NM(\VM, \varepsilon_0) \cap \VM_{h,k} \}$ for all $k \in [K]$.
		
		In the following, we will evoke Lemma~\ref{lem:self-bern} to analyze the term $(I)$.
		For any fixed $f' \in \VM$, we set $\GM_j = \FM_{h, j}, \x_j = \sigma_{h, j}^{-1} \Bphi_{h, j}, \eta_j = \sigma_{h, j}^{-1} \Beps_{h, j}^\top \Bf'  \cdot 1_{f' \in \VM_{h,k}}$ and $\Z_k = \lambda\I + \sum_{j=1}^k \sigma_{h, j}^{-2} \Bphi_{h, j}\Bphi_{h, j}^\top= \H_{h, k}$.
		Moreover, due to the choice of $\sigma_{h, j}$, it follows that
		\begin{align*}
			\left|\eta_j \min \left\{ 1, \| \x_j\|_{\Z_{j-1}^{-1}} \right\}\right|
			\le \left| \frac{C_0}{\sigma_{h, j}}   \right|  \cdot \left\| \frac{\Bphi_{h, j}}{\sigma_{h, j}} \right\|_{\H_{h,j-1}^{-1}} 
			\le C_0 \frac{b_{h,j}}{\sigma_{h, j}^2} \le \frac{C_0}{ d^{2.5}H \HM}.
		\end{align*}
		By Lemma~\ref{lem:self-bern} and the union bound, it follows that with probability $1-\frac{\delta}{H}$, for all $k \in [K]$,
		\begin{align*}
			&\sup_{f' \in \NM(\VM, \varepsilon_0)}\left\| \sum_{j=1}^k \sigma_{h, j}^{-2} \Bphi_{h, j}\Beps_{h, j}^\top \Bf' 1_{f' \in \VM_{h,k}}\right\|_{\H_{h, k}^{-1}}\\
			&\le 8C_\sigma \sqrt{d \log \left( 1 + \frac{K}{\sigma_{\min}^2d\lambda} \right) \log \frac{4H K^2N_0}{\delta}} + \frac{8C_0}{ d^{2.5}H \HM}\log\frac{4HK^2N_0}{\delta}.
		\end{align*}
		where $N_0 = |\NM(\VM, \varepsilon_0)|$. 
		
		As a result, we know that $\left\| \sum_{j=1}^k \sigma_{h, j}^{-2} \Bphi_{h, j}\Beps_{h, j}^\top \bar{\Bf} 1_{\bar{f} \in \VM_{h,k}}\right\|_{\H_{h, k}^{-1}}$ is no more than the RHS of the last inequality.
		On the event $\cap_{k \in [K]}\AM_{h, k}$, we have $\bar{f} \in \cap_{k \in [K]} \VM_{h,k}$ and thus all the indicator functions equal to one, completing the proof.
		
		\item The proof is almost similar to the second item except that we use Lemma~\ref{lem:self-hoff} to analyze the term $(I)$.
		Noticing we also have $|\eta_j| = |\sigma_{h, j}^{-1} \Beps_{h, j}^\top \Bf'| \le \frac{2C_0}{\sigma_{\min}}$.
		By Lemma~\ref{lem:self-hoff} and the union bound, it follows that with probability $1-\frac{\delta}{H}$, for all $k \in [K]$,
		\begin{align*}
			\sup_{f' \in \NM(\VM_f, \varepsilon_1)}\left\| \sum_{j=1}^k \sigma_{h, j}^{-2} \Bphi_{h, j}\Beps_{h, j}^\top \Bf' 1_{f' \in \VM_{h,k}}\right\|_{\H_{h, k}^{-1}}
			\le \frac{2C_0}{\sigma_{\min}} \sqrt{d\log \left( 1 + \frac{K}{\sigma_{\min}^2d\lambda} \right)  + \log\frac{N_1}{\delta} }.
		\end{align*}
		Pay attention that here we don't utilize the variance information so that we change $N_0 := |\NM(\VM, \varepsilon_0)|$ to $N_1 := |\NM(\VM, \varepsilon_1)|$ and don't require $\cap_{k \in [K]}\AM_{h, k}$ is true.
	\end{itemize}
\end{proof}

\subsection{Proof of Lemma~\ref{lem:CI-rewards}}
\label{proof:CI-rewards}
\begin{proof}[Proof of Lemma~\ref{lem:CI-rewards}]
	The proof idea of Lemma~\ref{lem:CI-rewards} is similar to that of Lemma~\ref{lem:CI-rewards-variance} except that we pay more attention on the reward variance.
	
	Given that $\BM_{R^2}$ holds, we have $\Bpsi_h^* \in \TRM_{h, k}$ for all $h \in [H]$ and $k \in [K] \cup \{0\}$.
	
	We will prove the lemma by induction over $k$.
	When $k=0$, we have $\Btheta_{h, 0} = 0, \H_{h, 0} = \lambda \I$ and $\|\Btheta_{h,0}-\Btheta_h^* \|_{\H_{h, 0}} = \sqrt{\lambda} \|\Btheta_h^*\| \le \sqrt{\lambda }W \le \beta_R$ for all $h \in [H]$.
	If we suppose $\Btheta_h^* \in \RM_{h, j}$ holds for all $h \in [H]$ and $j \in [k-1]$, we are going to prove  $\Btheta_h^* \in \RM_{h, k}$ uniformly for $h \in [H]$.
	The first thing we will show is 
	\begin{equation}
		\label{eq:R-variance-1}
		\sigma_{h, j}^2 \ge [\HVB_h R_h](s_{h, j}, a_{h, j}) + R_{h, j}
		\quad \text{for all} \quad h \in [H] \ \text{and} \ j \in [k].
	\end{equation}
	Notice that $[\VB_h R_h](s_{h, k}, a_{h, k}) =   \langle \TBphi_{h, k}, \Bpsi_h^*\rangle - \langle \Bphi_{h, k}, \est_h^*  \rangle^2$.
	We then have for all $h \in [H], j \in [k]$,
	\begin{align*}
		&|[\HVB_h R_h-\VB_h R_h](s_{h, j}, a_{h, j})|\\
		& \le \left|
		\langle \TBphi_{h, j}, \Bpsi_{h, j-1}  \rangle -
		\langle \TBphi_{h, j}, \Bpsi_h^*\rangle
		\right| + \left|
		\langle \Bphi_{h, j}, \est_h^*  \rangle^2-
		\langle \Bphi_{h, j}, \est_{h, j-1}  \rangle_{[0, \HM]}^2
		\right|\\
		&\le| \langle \TBphi_{h, k}, \Bpsi_{h, k-1}  - \Bpsi_h^* \rangle|
		+ 2\HM |\langle \Bphi_{h, k}, \Btheta_{h, j-1}  - \Btheta_h^* \rangle |\\
		&\le \| \TBphi_{h, j}\|_{\TH_{h, j-1}^{-1}} \|\Bpsi_{h, j-1}  - \Bpsi_h^*\|_{\TH_{h, j-1}}
		+ 2\HM  \|\Bphi_{h, j}\|_{\H_{h, j-1}^{-1}}  \| \Btheta_{h, j-1}  - \Btheta_h^*\|_{\H_{h, j-1}}\\
		&\le  \beta_{R^2}  \| \TBphi_{h, j}\|_{\TH_{h, j-1}^{-1}}  + 2\HM \beta_R \|\Bphi_{h, j}\|_{\H_{h, j-1}^{-1}} = R_{h, j}
	\end{align*}
	where the last inequality uses the hypothesis and the condition that $\BM_{R^2}$ holds.
	As a result, we establish~\eqref{eq:R-variance-1}.
	
	Let $\GM_{h, j} = \sigma(\FM_{h-1, j} \cup \left\{ s_{h, j}, a_{h, j} \right\})$.
	One can show that both $R_{h, j}$ and $\sigma_{h, j}^2$ are $\GM_{h,j}$-measurable.
	As a result, the event $\EM_{h, j} := \left\{ \sigma_{h, j}^2 \ge 
	[\VB_h R_h](s_{h, j}, a_{h, j}) 
	\right\}$ is also $\GM_{h,j}$-measurable.
	On the event $\BM_{R^2}$, it is obvious that $\cap_{h \in [H]}\cap_{j \in [k]}\EM_{h,j}$ is true since~\eqref{eq:R-variance-1} is true.
	
	On the other hand, we set $\varepsilon_{h, j} =  \frac{r_{h, j} -  \langle \Bphi_{h, j}, \Btheta_h^* \rangle }{\sigma_{h, j}}1_{\EM_{h, j}}$ as the standardized reward.
	We then have $\varepsilon_{h, j}  \in \FM_{h, j}$, $\EB[\varepsilon_{h, j} |\GM_{h, j}] = 0$ and $\Var[\varepsilon_{h, j} |\GM_{h, j}] \le 1$.
	We define $\widehat{\Btheta}_{h, k}$ as the solution of adaptive Huber regression to the response $ \{r_{h, j} 1_{\EM_{h, j}} \}_{j \in [k]}$ and the feature $ \{ \Bphi_{h, j} 1_{\EM_{h, j}} \}_{j \in [k]}$.
	We also define  $\widehat{\H}_{h, k-1}$ as the counterpart matrix of $\H_{h, k}$ obtained by replacing $\Bphi_{h, k}$ with $ \Bphi_{h, k} 1_{\EM_{h, k}}$.
	We then apply Theorem~\ref{thm:heavy} to analyze the concentration of $\widehat{\Btheta}_{h, k}$.
	With probability at least $1-3\delta$, it follows that $\|\widehat{\Btheta}_{h, k} - \Btheta_{h}^*\|_{\widehat{\H}_{h, k-1}} \le \beta_R$ for all $h \in [H]$ and $k \in [K]$ .
	Because $\BM_{R^2}$ is true, all indicator functions equal to one.
	Therefore, we have $\widehat{\Btheta}_{h, k} = \Btheta_{h, k}$ and $\widehat{\H}_{h, k-1} = \H_{h, k-1}$, implying $\Btheta_h^* \in \RM_{h, k}$ uniformly for $h \in [H]$.
\end{proof}

\subsection{Proof of Lemma~\ref{lem:opt-pes}}
\label{proof:opt-pes}
\begin{proof}[Proof of Lemma~\ref{lem:opt-pes}]
	By symmetry, we only prove the RHS inequality, or say, the optimism inequality.
	We prove it by induction.
	The statement is true for $h = H+1$ since both $V_{H+1}^*(\cdot) = \overoptV_{H+1}^k(\cdot) = 0$ for all $k \in [K]$.
	Assume the statement is also true for $h+1$, implying $ V_{h+1}^*(\cdot) \le \overoptV_{h+1}^k(\cdot) $ for all $k \in [K]$.
	We assume there exists a sequence of updating episodes $1\le k_1 < \cdots < k_{N_k} \le K$ such that
	\begin{equation}
	\label{eq:Q-detial}
			\overoptQ_h^k(\cdot, \cdot)
	= \min_{i \in [N_k]} \left\{ \langle \Bphi(\cdot, \cdot),  \Btheta_{h, k_i-1} + \Bmu_{h, k_i-1}  \BoveroptV_{h+1}^{k_i} \rangle + \beta \|\Bphi(\cdot, \cdot)\|_{\H_{h, k_i-1}^{-1}}, \HM\right\}.
	\end{equation}
	Using $Q_h^*(s, a) = \langle \Bphi(s, a), \Btheta_h^* + \Bmu_h^* \V_{h+1}^* \rangle$, we have for any $(s, a) \in \SM \times \AM$ and $k \in [K]$,
	\begin{align*}
		& \langle \Bphi(\cdot, \cdot),  \Btheta_{h, k-1} + \Bmu_{h, k-1}  \BoveroptV_{h+1}^{k} \rangle + \beta \|\Bphi(\cdot, \cdot)\|_{\H_{h, k-1}^{-1}} - Q_h^*(\cdot, \cdot) \\
		&= \langle \Bphi(\cdot, \cdot),  \Btheta_{h, k-1} - \Btheta_h^* \rangle
		+ \langle \Bphi(\cdot, \cdot), \Bmu_{h, k-1}  \BoveroptV_{h+1}^{k}  - \Bmu_h^* \V_{h+1}^*\rangle  + \beta \|\Bphi(\cdot, \cdot)\|_{\H_{h, k-1}^{-1}}\\
		& \overset{(a)}{\ge} \langle \Bphi(\cdot, \cdot),  \Btheta_{h, k-1} - \Btheta_h^* \rangle
		+ \langle \Bphi(\cdot, \cdot), (\Bmu_{h, k-1} -\Bmu_h^*) \BoveroptV_{h+1}^{k} \rangle  + \beta \|\Bphi(\cdot, \cdot)\|_{\H_{h, k-1}^{-1}}\\
		&\overset{(b)}{\ge} \|\Bphi(\cdot, \cdot)\|_{\H_{h, k-1}^{-1}} \left[
		- \| \Btheta_{h, k-1} - \Btheta_h^*\|_{\H_{h ,k-1}}
		- \| (\Bmu_{h, k-1} -\Bmu_h^*) \BoveroptV_{h+1}^{k}\|_{\H_{h ,k-1}} + \beta
		\right] \overset{(c)}{\ge} 0
	\end{align*}
	where $(a)$ uses $\langle \Bphi(s, a), \Bmu_{h}^*  (\BoveroptV_{h+1}^{k}  -\V_{h+1}^*)\rangle = \PB_h(\overoptV_{h+1}^{k}  -V_{h+1}^*)(s, a) \ge 0$ from the hypothesis,
	$(b)$ follows from Cauchy-Schwarz inequality 
	and $(c)$ uses $\| \Btheta_{h, k-1} - \Btheta_h^*\|_{\H_{h ,k-1}}
	+ \| (\Bmu_{h, k-1} -\Bmu_h^*) \BoveroptV_{h+1}^{k}\|_{\H_{h ,k-1}}  \le \beta_R + \beta_V = \beta$ on the event $\BM_R \cap \BM_{h}$.
	
	As a result, by the last inequality and~\eqref{eq:Q-detial}, it follows that for all $k \in [K]$, $	\overoptQ_h^k(\cdot, \cdot)-Q_h^*(\cdot, \cdot) \ge 0$. Taking maximum over actions, we have $\overoptV_h^k(\cdot) \ge V_h^*(\cdot)$ for all $k \in [K]$, which implies the case of $h$ is also true.
\end{proof}

\subsection{Proof of Lemma~\ref{lem:var0}}
\label{proof:var0}
\begin{proof}[Proof of Lemma~\ref{lem:var0}]
	The proof technique has been used in Lemma C.13 in~\citep{hu2022nearly} and Lemma 7.2 in~\citep{he2022nearly}.
	We include the proof for completeness.
	By definition,
	\begin{gather*}
		[\VB_h V_{h+1}^*](s_{h, k}, a_{h, k})  = 	\langle \Bmu_{h}^* [\V_{h+1}^*]^2, \Bphi_{h, k} \rangle  - 	\langle \Bmu_{h}^* \V_{h+1}^*, \Bphi_{h, k} \rangle^2\\
		[\HVB_h \overoptV_{h+1}^k](s_{h, k}, a_{h, k})  = 	\langle \Bmu_{h,k-1} [\BoveroptV_{h+1}^k]^2, \Bphi_{h, k-1} \rangle_{[0, \HM^2]}  - 	\langle \Bmu_{h,k-1} \BoveroptV_{h+1}^k, \Bphi_{h, k} \rangle_{[0, \HM]}^2.
	\end{gather*}
	Therefore, it follows that
	\begin{align*}
		&\left|
		\left[\VB_h V_{h+1}^* - \HVB_h \overoptV_{h+1}^k\right](s_{h, k}, a_{h, k}) 
		\right| \\
		& \le
		\left|
		\left[ \VB_h \overoptV_{h+1}^k - \HVB_h \overoptV_{h+1}^k\right](s_{h, k}, a_{h, k}) 
		\right|+ \left|
		\left[\VB_h V_{h+1}^* - \VB_h \overoptV_{h+1}^k\right](s_{h, k}, a_{h, k}) 
		\right|.
	\end{align*}
 We then bound the two terms in the RHS of the last inequality as follows.
	\begin{align*}
		&\left|
		\left[ \VB_h \overoptV_{h+1}^k - \HVB_h \overoptV_{h+1}^k\right](s_{h, k}, a_{h, k}) 
		\right|\\
		&\le \left| \langle\Bmu_{h}^* [\BoveroptV_{h+1}^k]^2, \Bphi_{h, k} \rangle -\langle \Bmu_{h,k-1} [\BoveroptV_{h+1}^k]^2, \Bphi_{h, k} \rangle_{[0, \HM^2]}    \right| \\
		& \qquad +  \left| \langle \Bmu_{h}^* \BoveroptV_{h+1}^k, \Bphi_{h, k} \rangle^2 - \langle \Bmu_{h,k-1} \BoveroptV_{h+1}^k, \Bphi_{h, k} \rangle_{[0, \HM]}^2  \right| \\
		&\le  \left| \langle (\Bmu_{h}^*-\Bmu_{h,k-1}) [\BoveroptV_{h+1}^k]^2, \Bphi_{h, k} \rangle   \right|  + 2\HM \cdot \left| \langle \Bmu_{h}^* \BoveroptV_{h+1}^k, \Bphi_{h, k} \rangle - \langle \Bmu_{h,k-1} \BoveroptV_{h+1}^k, \Bphi_{h, k} \rangle_{[0, \HM]}  \right| \\
		&\le  \left\|(\Bmu_{h, k-1} - \Bmu_{h}^*) [\BoveroptV_{h+1}^k]^2 \right\|_{\H_{h, k-1}} \|\Bphi_{h, k}\|_{\H_{h, k-1}^{-1}}
		+ 2\HM \|\Bphi_{h, k}\|_{\H_{h, k-1}^{-1}} 
		\left\|(\Bmu_{h, k-1} - \Bmu_{h}^*) \BoveroptV_{h+1}^k \right\|_{\H_{h, k-1}} 
	\end{align*}
	where the second inequality uses the fact that both $\langle \Bmu_{h}^* \BoveroptV_{h+1}^k, \Bphi_{h, k} \rangle$ and $\langle \Bmu_{h,k-1} \BoveroptV_{h+1}^k, \Bphi_{h, k} \rangle_{[0, \HM]} $ lie between $0$ and $\HM$.
	Similarly, it follows that
	\begin{align*}
		&\left|
		\left[\VB_h V_{h+1}^* - \VB_h \overoptV_{h+1}^k\right](s_{h, k}, a_{h, k}) 
		\right| \\
		&\le \left| \PB_h[[V_{h+1}^*]^2 - [\overoptV_{h+1}^k]^2](s_{h, k}, a_{h, k}) 
		\right| + \left| [\PB_h V_{h+1}^*]^2(s_{h, k}, a_{h, k})  -  [\PB_h\overoptV_{h+1}^k]^2(s_{h, k}, a_{h, k}) 
		\right|  \\
		&\le \left| \PB_h[(\overoptV_{h+1}^k-V_{h+1}^*)(\overoptV_{h+1}^k+V_{h+1}^*)](s_{h, k}, a_{h, k}) 
		\right| \\
          &\qquad + \left| [\PB_h\overoptV_{h+1}^k-\PB_h V_{h+1}^*][[\PB_h\overoptV_{h+1}^k+\PB_h V_{h+1}^*]](s_{h, k}, a_{h, k}) 
		\right|  \\
		&\le 4 \HM \cdot \PB_h[\overoptV_{h+1}^k-V_{h+1}^*](s_{h, k}, a_{h, k})  \\
		&\le 4 \HM \cdot \PB_{h}[\overoptV_{h+1}^k-\overpesV_{h+1}](s_{h, k}, a_{h, k})   \\
		&\le 4 \HM \cdot \HPB_{h, k}[\overoptV_{h+1}^k-\overpesV_{h+1}](s_{h, k}, a_{h, k})   \\
		& \quad	+ 4\HM  \|\Bphi_{h, k}\|_{\H_{h, k-1}^{-1}}  \cdot \left[
		\left\|(\Bmu_{h, k-1} - \Bmu_{h}^*) \BoveroptV_{h+1}^k \right\|_{\H_{h, k-1}} 
		+	\left\|(\Bmu_{h, k-1} - \Bmu_{h}^*) \BoverpesV_{h+1}^k \right\|_{\H_{h, k-1}} 
		\right]
	\end{align*}
	where the third and forth inequality we use the optimism and pessimism in Lemma~\ref{lem:opt-pes} and the last inequality uses the following result.
	\begin{align*}
		\left| [\PB_h\overoptV_{h+1}^k-\HPB_{h, k}\overoptV_{h+1}^k](s_{h, k}, a_{h, k})   \right|
		&= \left| \langle (\Bmu_{h}^*-\Bmu_{h,k-1})\BoveroptV_{h+1}^k, \Bphi_{h, k} \rangle \right| \\
		&\le \|\Bphi_{h, k}\|_{\H_{h, k-1}^{-1}} 
		\left\|(\Bmu_{h, k-1} - \Bmu_{h}^*) \BoveroptV_{h+1}^k \right\|_{\H_{h, k-1}} .
	\end{align*}
	A similar inequality can be derived for $	\left| [\PB_h\overpesV_{h+1}^k-\HPB_{h, k}\overpesV_{h+1}^k](s_{h, k}, a_{h, k})   \right|$.
	Finally, we have
	\begin{align*}
		&\left|
		\left[\VB_h V_{h+1}^* - \HVB_h \overoptV_{h+1}^k\right](s_{h, k}, a_{h, k}) 
		\right|\\
		&\le \left\|(\Bmu_{h, k-1} - \Bmu_{h}^*) [\BoveroptV_{h+1}^*]^2 \right\|_{\H_{h, k-1}} \|\Bphi_{h, k}\|_{\H_{h, k-1}^{-1}}+ 4\HM  
		\HPB_{h, k}(\overoptV_{h+1}^k -\overpesV_{h+1}^k)(s_{h, k}, a_{h, k})  \\
		& \qquad + \HM\|\Bphi_{h, k}\|_{\H_{h, k-1}^{-1}} \cdot  \left[
		6\left\|(\Bmu_{h, k-1} - \Bmu_{h}^*) \BoveroptV_{h+1}^k \right\|_{\H_{h, k-1}} +
		4	\left\|(\Bmu_{h, k-1} - \Bmu_{h}^*) \BoverpesV_{h+1}^k \right\|_{\H_{h, k-1}} 
		\right].
	\end{align*}
	We complete the proof by noting that on the event $\BM_0$, we have
	\begin{gather*}
		\max \left\{ \left\| (\Bmu_h^*-\Bmu_{h, k-1}) \BoveroptV_{h+1}^k \right\|_{\H_{h, k-1}} ,  \left\| (\Bmu_h^*-\Bmu_{h, k-1}) \BoverpesV_{h+1}^k \right\|_{\H_{h, k-1}}
		\right\}\le  \beta_0, \\
		\left\| (\Bmu_{h}^*-\Bmu_{h, k-1}) [\BoveroptV_{h+1}^k]^2 \right\|_{\H_{h, k-1}}  \le \HM \beta_0.
	\end{gather*}
\end{proof}

\subsection{Proof of Lemma~\ref{lem:var1}}
	\label{proof:var1}
	\begin{proof}[Proof of Lemma~\ref{lem:var1}]
		For any $j \le k$, we have
		\begin{align*}
			\left[\VB_h(\overoptV_{h+1}^k - V_{h+1}^*)\right](s_{h, j}, a_{h, j})
			& \overset{(a)}{\le}	\left[ \PB_h(\overoptV_{h+1}^k - V_{h+1}^*)^2\right](s_{h, j}, a_{h, j})\\
			&\overset{(b)}{\le} \HM 	\left[ \PB_h(\overoptV_{h+1}^k - \overpesV_{h+1}^k)\right](s_{h, j}, a_{h, j})\\
				&\overset{(c)}{\le} \HM 	\left[ \PB_h(\overoptV_{h+1}^j - \overpesV_{h+1}^j)\right](s_{h, j}, a_{h, j})
		\end{align*}
		where $(a)$ uses the fact that $\Var(X) \le \EB X^2$ for any random variable $X$, $(b)$ uses $0 \le \overpesV_{h+1}^k(\cdot)  \le V_{h+1}^*(\cdot) \le \overoptV_{h+1}^k(\cdot) \le \HM$ on the event $\BM_R \cap \BM_{h+1}$ from Lemma~\ref{lem:opt-pes}, and $(c)$ uses that $\overoptV_{h+1}^j(\cdot) \ge \overoptV_{h+1}^k(\cdot)$ and $\overpesV_{h+1}^j(\cdot) \le \overpesV_{h+1}^k(\cdot)$ by definition.
		On the other hand, on the event $\BM_0$, we have
		\begin{align*}
			\max \left\{ \left\| (\Bmu_h^*-\Bmu_{h, j-1}) \BoveroptV_{h+1}^j \right\|_{\H_{h, j-1}} ,  \left\| (\Bmu_h^*-\Bmu_{h, j-1}) \BoverpesV_{h+1}^j \right\|_{\H_{h, j-1}}
			\right\}\le  \beta_0.
		\end{align*}
		As a result,
		\begin{align*}
			\left[\left( \PB_h- \HPB_{h, j} \right) \overoptV_{h+1}^j \right](s_{h, j}, a_{h, j})
			&= \langle\Bphi_{h, j}, (\Bmu_h^* - \Bmu_{h, j-1})   \BoveroptV_{h+1}^j\rangle\\
			&\le  \|\Bphi_{h, j}\|_{\H_{h, j-1}^{-1}}  \|(\Bmu_h^* - \Bmu_{h, j-1})   \BoveroptV_{h+1}^j\|_{\H_{h, j-1}}
			\le \beta_0 \|\Bphi_{h, j}\|_{\H_{h, j-1}^{-1}} .
		\end{align*}
		Similarily, we have $\left[\left( \PB_h- \HPB_{h, j} \right) \overpesV_{h+1}^j \right](s_{h, j}, a_{h, j}) \le  \beta_0 \|\Bphi_{h, j}\|_{\H_{h, j-1}^{-1}}$.
		Therefore,
		\begin{align*}
			&\left[\VB_h(\overoptV_{h+1}^k - V_{h+1}^*)\right](s_{h, j}, a_{h, j}) \\
			&\le \HM \left[2 \beta_0 \|\Bphi_{h, k}\|_{\H_{h, j-1}^{-1}}
			+ 	\left[ \HPB_{h, j}(\overoptV_{h+1}^j - \overpesV_{h+1}^j)\right](s_{h, j}, a_{h, j})
			\right] =: E_{h, j}
		\end{align*}	
		Repeating the above argument, we have a similar inequality for $\overoptV_{h+1}^k$ due to symmetry.
	\end{proof}

\subsection{Proof of Lemma~\ref{lem:fine-CI}}
\label{proof:fine-CI}
\begin{proof}[Proof of Lemma~\ref{lem:fine-CI}]	
	Due to the backward recursion structure, we will use induction (over horizon $h$) to prove this lemma.
	First, \eqref{eq:fine-CI} is true for $h=H$ since $\overoptV_{H+1}^k(\cdot) = \overpesV_{H+1}^k(\cdot) = 0$ for all $k \in [K]$.
	Therefore, we have $\BM_H$ holds.
	Assume \eqref{eq:fine-CI} holds for horizons no smaller than $h+1$, i.e., $\BM_{h+1}$ holds with $h+1 \le H$.
	In the following, we will show, once $\BM_{h+1} \cap  \BM_0$ holds, $\BM_{h}$ holds with probability at least than $1-\frac{2\delta}{H}$.
	Repeating the argument, we have, given $\BM_H \cap  \BM_0$ holds, with probability at least $1-2\delta$, $\BM_1 \cap  \BM_0$ holds.
	Hence, $\PB(\BM_0 \cap  \BM_1) \ge 1- 5 \delta$.
	
	Note that
	\begin{align*}
		\max &\left\{ \left\| (\Bmu_h^*-\Bmu_{h, k-1}) \BoveroptV_{h+1}^k \right\|_{\H_{h, k-1}} ,  \left\| (\Bmu_h^*-\Bmu_{h, k-1}) \BoverpesV_{h+1}^k \right\|_{\H_{h, k-1}} 
		\right\}
		\le  \left\| (\Bmu_h^*-\Bmu_{h, k-1}) \V_{h+1}^* \right\|_{\H_{h, k-1}}   \\
		& +  \max \left\{
		\left\| (\Bmu_h^*-\Bmu_{h, k-1}) (\BoveroptV_{h+1}^k-\V_{h+1}^*) \right\|_{\H_{h, k-1}},\left\| (\Bmu_h^*-\Bmu_{h, k-1}) (\BoverpesV_{h+1}^k-\V_{h+1}^*) \right\|_{\H_{h, k-1}}
		\right\}.
	\end{align*}
	we would analyze the two terms in the RHS separately to proceed with the proof.
	\paragraph{For the first term}
	Since $\V_{h+1}^*$ is a deterministic function, we apply the first item in Lemma~\ref{lem:value-ci} to bound it.
	In the following, we specify the parameters defined therein.
	First, we have $C_0 = \HM$ and $\AM_{h, k} = \left\{
	\sigma_{h, k}^2 \ge (\VB_h V_{h+1}^*)(s_{h, k}, a_{h, k}) 
	\right\}$ is $\FM_{h, k}$-measurable.
	By Lemma~\ref{lem:var0}, on the event $\BM_0 \cap \BM_{h+1}$, we have for all $k \in [K]$, $\left|
	\left[\VB_h V_{h+1}^* - \HVB_h \overoptV_{h+1}^k\right](s_{h, k}, a_{h, k}) 
	\right| \le U_{h, k}$ with $U_{h, k}$ defined in~\eqref{eq:U}.
	Hence, $\sigma_{h, k}^2 \ge [\HVB_h \optV_{h+1}^k](s_{h, k}, a_{h, k}) +  U_{h, k} \ge [\VB_h V_{h+1}^*](s_{h, k}, a_{h, k})$ for all $k \in [K]$, implying $\cap_{k \in [K]}\AM_{h, k}$ holds under $\BM_0 \cap \BM_{h+1}$ and $C_{\sigma} = 1$.
	By Lemma~\ref{lem:value-ci}, with probability at least $1-\frac{\delta}{H}$, $ \left\| (\Bmu_h^*-\Bmu_{h, k-1}) \V_{h+1}^* \right\|_{\H_{h, k-1}}  \le \beta_1$ for all $k \in [K]$ with $\beta_1$ defined in the following.
	Finally, we simplify $\beta_1$ as
	\begin{align*}
		\beta_1	&:= 8\sqrt{d \log \left( 1 + \frac{K}{\sigma_{\min}^2d\lambda} \right) \log \frac{4H K^2}{\delta} } + \frac{8}{ d^{2.5}H}\log\frac{4HK^2}{\delta} + \sqrt{d \lambda} \HM\\
		&\le 8\sqrt{d} \iota_1 + \frac{8\iota_1}{d^{2.5}H} +  \sqrt{d \lambda} \HM
		\le 16 \sqrt{d} \iota_1 +  \sqrt{d \lambda} \HM.
	\end{align*}
	
	\paragraph{For the second term}
	Since both $\BoveroptV_{h+1}^k-\V_{h+1}^*$ and $\BoverpesV_{h+1}^k-\V_{h+1}^*$ are $\FM_{H, k-1}$-measurable random functions, we apply the second item in Lemma~\ref{lem:value-ci} to analyze the second term.
	In the following, we specify parameters defined therein.
	First, $C_0 = \HM$ and $\AM_{h, k} = \left\{ \sigma_{h, k}^2 \ge d^3 H \cdot E_{h, k} \right\}$ is $\FM_{h, k}$-measurable.
	By Lemma~\ref{lem:var1}, on the event $\BM_0\cap \BM_R \cap \BM_{h+1}$, we have simultaneously $ \left[\VB_h(\overoptV_{h+1}^k - V_{h+1}^*)\right](s_{h, j}, a_{h, j}) \le E_{h, j}$ and $ \left[\VB_h(\overpesV_{h+1}^k - V_{h+1}^*)\right](s_{h, j}, a_{h, j}) \le E_{h, j}$ for all $j \le k \le K$ with $E_{h, j}$ defined in~\eqref{eq:E}.
	As a result, for all $j \le k$,
	\begin{align*}
		\sigma_{h, j}^2 
		&\ge  d^3 H \cdot E_{h, j} \ge d^3H \cdot \max \left\{
		\left[
		\VB_h(\overoptV_{h+1}^k - V_{h+1}^*)\right](s_{h, j}, a_{h, j})
		, 	\left[
		\VB_h(\overpesV_{h+1}^k - V_{h+1}^*)\right](s_{h, j}, a_{h, j})
		\right\}.
	\end{align*}
	It implies $C_{\sigma} = \frac{1}{\sqrt{d^3H}}$ and for any $j \in [k]$,
	\[
	\AM_{h, j} \subseteq \left\{ \sigma_{h, j}^2 \ge C_{\sigma}^{-2}  \max \left\{
	\left[
	\VB_h(\overoptV_{h+1}^k - V_{h+1}^*)\right](s_{h, j}, a_{h, j})
	, 	\left[
	\VB_h(\overpesV_{h+1}^k - V_{h+1}^*)\right](s_{h, j}, a_{h, j})
	\right\}  \right\}.
	\]
	Finally, with by Lemma~\ref{lem:covering-1} and~\ref{lem:rare-update}, the covering entropy for $\varepsilon_0 = \min\left\{ \frac{\sigma_{\min}}{\sqrt{d^3H}}, \frac{\lambda \HM \sqrt{d}}{K}\sigma_{\min}^2 \right\}$ and the function class to which $\overoptV_{h+1}^k - V_{h+1}^*$ and $\overpesV_{h+1}^k - V_{h+1}^*$ belong is 
	\begin{align*}
		\log N_0 &= |\NM(\VM^{\pm}, \varepsilon_0)|
		\le \left[d \log\left(
		1 + \frac{4L}{\varepsilon_0}
		\right) + d^2 \log \left( 1 + \frac{8\sqrt{d}B^2}{\lambda \varepsilon_0^2} \right) \right] \cdot dH \log_2\left( 1+ \frac{K}{\lambda \sigma_{\min}^2} \right) \\
		&= \OM(d^3H \iota_1^2)
	\end{align*}
	By Lemma~\ref{lem:value-ci}, with probability at least $1-\frac{\delta}{H}$, 
	\[
	\max \left\{
	\left\| (\Bmu_h^*-\Bmu_{h, k-1}) (\BoveroptV_{h+1}^k-\V_{h+1}^*) \right\|_{\H_{h, k-1}},\left\| (\Bmu_h^*-\Bmu_{h, k-1}) (\BoverpesV_{h+1}^k-\V_{h+1}^*) \right\|_{\H_{h, k-1}}
	\right\} \le \beta_2
	\]
	for all $k \in [K]$ with $\beta_2$ defined in the following.
	Finally, we simplify $\beta_2$ as
	\begin{align*}
		\beta_2&=\frac{8}{\sqrt{d^3H}}\sqrt{d\log \left( 1 + \frac{K}{\sigma_{\min}^2d\lambda} \right) \log \frac{4N_0H K^2}{\delta} } + \frac{8}{d^{2.5}H}\log\frac{4N_0HK^2}{\delta} + \sqrt{d \lambda} \HM \\
		&\le 8\sqrt{\frac{\iota_1}{d^2H} \cdot \left(  \iota_1 + \OM(d^3H \iota_1^2) \right)} + \frac{8}{d^{2.5}H} \left( \iota_1 +  \OM(d^3H \iota_1^2) \right)+ \sqrt{d \lambda} \HM \\
		&= \OM\left(  \sqrt{d}\iota_1^{1.5} +  \iota_1 + \sqrt{d} \iota_1^2 +  \sqrt{d \lambda} \HM  \right)
		=\OM\left(  \sqrt{d} \iota_1^2 +  \sqrt{d \lambda} \HM  \right).
	\end{align*}
	
	\paragraph{Putting pieces together} we have shown that given $\BM_{h+1} \cap \BM_0$ is true, with probability at least $1-2\delta$, for all $h \in [H]$ and $k \in [K]$,
	\[
	\max \left\{ \left\| (\Bmu_h^*-\Bmu_{h, k-1}) \BoveroptV_{h+1}^k \right\|_{\H_{h, k-1}} ,  \left\| (\Bmu_h^*-\Bmu_{h, k-1}) \BoverpesV_{h+1}^k \right\|_{\H_{h, k-1}} 	\right\} \le \beta_1 + \beta_2 = \OM\left(  \sqrt{d} \iota_1^2 +  \sqrt{d \lambda} \HM  \right).
	\]
	Therefore, $\BM_V:=\BM_1$ holds.
\end{proof}

\subsection{Proof of Lemma~\ref{lem:sub-gap}}
\label{proof:sub-gap}
\begin{proof}[Proof of Lemma~\ref{lem:sub-gap}]

	For a given $k$, let $\kl$ denote the latest update episode before episode $k$, that is $\kl \le k < \kl+1$.
	By Lemma~\ref{lem:matrix-ratio}, due to $\H_{h, k-1} \succeq \H_{h, \kl-1}$ and $\det(\H_{h, k-1}) \le 2 \det(\H_{h, \kl-1})$, it follows that for any $\x \in \RB^d$,
	\begin{equation}
		\label{eq:bonus-ratio}
			\|\x\|_{ \H_{h, \kl-1}^{-1}} \le 2 \|\x\|_{ \H_{h, k-1}^{-1}}.
	\end{equation}	
	By definition, $\overoptQ_h^k(\cdot, \cdot) \le \langle \Bphi(\cdot, \cdot),  \Btheta_{h, \kl-1} + \Bmu_{h, \kl-1}  \BoveroptV_{h+1}^{\kl} \rangle + \beta \|\Bphi(\cdot, \cdot)\|_{\H_{h, \kl-1}^{-1}}$ and $Q_h^{\pi_k}(s, a) = \langle \Bphi(s, a), \Btheta_h^* + \Bmu_h^* \V_{h+1}^{\pi_k} \rangle$.
	Using $a_{h, k} = \pi_h^k(s_{h, k}) = \argmax_{a \in \AM} \overoptQ_h^k(s_{h, k}, a)$, we then have
	\begin{align*}
		(&\overoptV_{h}^k - V_h^{\pi_k})(s_{h, k})
		\le (\overoptQ_h^k -Q_h^{\pi_k}) (s_{h, k}, a_{h, k}) \\
		&\le\langle\Bphi_{h, k},  \Btheta_{h, \kl-1} + \Bmu_{h, \kl-1}  \BoveroptV_{h+1}^{\kl} - (\Btheta_h^* + \Bmu_h^* \V_{h+1}^{\pi_k} ) \rangle +\beta \|\Bphi(s_{h, k}, a_{h, k})\|_{\H_{h, \kl-1}^{-1}}\\
		&\overset{(a)}{\le} \langle\Bphi_{h, k},  (\Btheta_{h, \kl-1} -\Btheta_h^*) +  (\Bmu_{h, \kl-1} -\Bmu_h^*) \BoveroptV_{h+1}^{\kl} \rangle 
		+ \langle\Bphi_{h, k}, \Bmu_h^*(  \BoveroptV_{h+1}^{\kl} -  \V_{h+1}^{\pi_k}
		) \rangle +2\beta \|\Bphi_{h, k}\|_{\H_{h, k-1}^{-1}}\\
		&\overset{(b)}{\le}4\beta \|\Bphi_{h, k}\|_{\H_{h, k-1}^{-1}} + \langle\Bphi_{h, k}, \Bmu_h^*(  \BoveroptV_{h+1}^{k} -  \V_{h+1}^{\pi_k}
		) \rangle \\
		&\overset{(c)}{=}
		4\beta \|\Bphi_{h, k}\|_{\H_{h, k-1}^{-1}} + \PB_h( \overoptV_{h+1}^{k} -  V_{h+1}^{\pi_k})(s_{h, k}, a_{h, k})\\
		&\overset{(d)}{=}4\beta \|\Bphi_{h, k}\|_{\H_{h, k-1}^{-1}} + ( \overoptV_{h+1}^{k} -  V_{h+1}^{\pi_k})(s_{h+1, k}) + X_{h, k}.
	\end{align*}
Here $(a)$ uses~\eqref{eq:bonus-ratio}, $(b)$ uses 
\begin{align*}
&|\langle\Bphi_{h, k},  (\Btheta_{h, \kl-1} -\Btheta_h^*) +  (\Bmu_{h, \kl-1} -\Bmu_h^*) \BoveroptV_{h+1}^{\kl} \rangle | \\
&\le 
\| \Bphi_{h, k}\|_{\H_{h, \kl-1}^{-1}} \|(\Btheta_{h, \kl-1} -\Btheta_h^*) +  (\Bmu_{h, \kl-1} -\Bmu_h^*) \BoveroptV_{h+1}^{\kl}\|_{\H_{h, \kl-1}} \\
&\le \beta  \| \Bphi_{h, k}\|_{\H_{h, \kl-1}^{-1}} \le 2\beta  \| \Bphi_{h, k}\|_{\H_{h, k-1}^{-1}}
\end{align*}
on $\BM_R \cap \BM_V$, $(c)$ uses $
\langle\Bphi_{h, k}, \Bmu_h^*(  \overoptV_{h+1}^{\kl} -  V_{h+1}^{\pi_k}
) \rangle  = \PB_h( \BoveroptV_{h+1}^{\kl} -  \V_{h+1}^{\pi_k})(s_{h, k}, a_{h, k})=\PB_h( \BoveroptV_{h+1}^{k} -  \V_{h+1}^{\pi_k})(s_{h, k}, a_{h, k})$, and $(d)$ uses the notation 
\[
X_{h, k} :=  \PB_h( \overoptV_{h+1}^{k} -  V_{h+1}^{\pi_k})(s_{h, k}, a_{h, k}) - ( \overoptV_{h+1}^{k} -  V_{h+1}^{\pi_k})(s_{h+1, k}).
\]
The last inequality implies
\begin{align*}
(\overoptV_{h}^k - V_h^{\pi_k})(s_{h, k})
\le ( \overoptV_{h+1}^{k} -  V_{h+1}^{\pi_k})(s_{h+1, k}) + X_{h, k} + 4\beta \|\Bphi_{h, k}\|_{\H_{h, k-1}^{-1}}.
\end{align*}
Iterating the above inequality over $h$ and using $ \overoptV_{H+1}^{k}(\cdot) = V_{H+1}^{\pi_k}(\cdot) = 0$, we have
\begin{equation}
\label{eq:final-iter}
(\overoptV_{h}^k - V_h^{\pi_k})(s_{h, k})
\le \sum_{i=h}^H \left[ X_{i, k} + 4\beta \|\Bphi_{i, k}\|_{\H_{i, k-1}^{-1}}\right].
\end{equation}
Therefore, setting $h=1$ and summing~\eqref{eq:final-iter} over $k \in [K]$, we have
\begin{equation}
\label{eq:final0}
\sum_{k =1}^K(\overoptV_{1}^k - V_1^{\pi_k})(s_{1, k})
\le \sum_{k=1}^K \sum_{h=1}^H \left[ X_{h, k} + 4\beta \|\Bphi_{h, k}\|_{\H_{h, k-1}^{-1}}\right].
\end{equation}

	We then need to analyze $ \sum_{k=1}^K \sum_{h=1}^H  X_{h, k} $.
	Since $s_{h+1, k}$ is $\FM_{h+1, k}$-measurable, $\pi_k = \{ \pi_h^k \}_{h \in [H]}, \overoptV_{h+1}^k$ is $\FM_{H, k-1}$-measurable, we have $X_{h, k}$ is $\FM_{h+1, k}$-measurable.
	We also have $\EB[X_{h, k}|\FM_{h, k}] = 0$, $|X_{h, k}| \le 2\HM$ and 
	\begin{align*}
		\EB[X_{h, k}^2|\FM_{h, k}]  &\le  \EB[ (\overoptV_{h+1}^k - V_{h+1}^{\pi_k})^2(s_{h+1, k}) |\FM_{h, k}] \overset{(a)}{\le}  \HM \EB[ |\overoptV_{h+1}^k - V_{h+1}^{\pi_k}|(s_{h+1, k}) |\FM_{h, k}]\\
		&\overset{(b)}{=}\HM \EB[ (\overoptV_{h+1}^k - V_{h+1}^{\pi_k})(s_{h+1, k}) |\FM_{h, k}]   
		=\HM \PB_h(\overoptV_{h+1}^k - V_{h+1}^{\pi_k})(s_{h, k}, a_{h, k}) 
	\end{align*}
	where $(a)$ uses $ |\overoptV_{h+1}^k - V_{h+1}^{\pi_k}|(\cdot) \le \HM$ and $(b)$ uses the optimism in Lemma~\ref{lem:opt-pes}.
	By the variance-aware Freedman inequality in Lemma~\ref{lem:bern}, with probability at least $1-\frac{\delta}{2}$, it follows that
	\begin{equation}
		\label{eq:final1}
		\left| \sum_{k=1}^K \sum_{h=1}^H  X_{h, k} \right|
		\le 3 \sqrt{\iota} \cdot  \sqrt{\HM \cdot \sum_{k=1}^K \sum_{h=1}^H \PB_h(\overoptV_{h+1}^k - V_{h+1}^{\pi_k})(s_{h, k}, a_{h, k})  } + 10\HM\cdot \iota
	\end{equation}
	where $\iota = \log\frac{4\ceil{\log_2 HK}}{\delta} $.
	On the other hand, it follows that
	\begin{align*}
		\sum_{k=1}^K \sum_{h=1}^H& \PB_h(\overoptV_{h+1}^k - V_{h+1}^{\pi_k})(s_{h, k}, a_{h, k}) 
		= 	\sum_{k=1}^K \sum_{h=2}^H (\overoptV_{h}^k - V_{h}^{\pi_k})(s_{h, k}) + \sum_{k=1}^K \sum_{h=1}^H  X_{h, k} \\
		&\overset{(a)}{\le}\sum_{k=1}^K \sum_{h=2}^H 
		\sum_{i=h}^H \left[ X_{i, k} + 4\beta \|\Bphi_{i, k}\|_{\H_{i, k-1}^{-1}}\right] + \sum_{k=1}^K \sum_{h=1}^H  X_{h, k}\\
		&=\sum_{k=1}^K \sum_{h=2}^H (H-h+1)\left[ X_{h, k} + 4\beta \|\Bphi_{h, k}\|_{\H_{h, k-1}^{-1}}\right] + \sum_{k=1}^K \sum_{h=1}^H  X_{h, k}\\
		&\overset{(b)}{\le} 4H\beta \sum_{k=1}^K \sum_{h=2}^H  \|\Bphi_{h, k}\|_{\H_{h, k-1}^{-1} }+ \sum_{k=1}^K \sum_{h=1}^H  X_{h, k}  b_{h}
	\end{align*}
	where $(a)$ uses~\eqref{eq:final-iter} and $(b)$ uses the notation $b_h = 1$ if $h=1$; otherwise $= H-h+2$ for $2 \le h \le H$.
	Clearly, we have $|b_h| \le H$ for all $h \in [H]$.
	By the variance-aware Freedman inequality in Lemma~\ref{lem:bern}, with probability at least $1-\frac{\delta}{2}$, it follows that
	\[
	\left| \sum_{k=1}^K \sum_{h=1}^H  X_{h, k} b_h \right|
	\le 3 H \sqrt{\iota} \cdot  \sqrt{\HM \cdot \sum_{k=1}^K \sum_{h=1}^H \PB_h(\overoptV_{h+1}^k - V_{h+1}^{\pi_k})(s_{h, k}, a_{h, k})  } + 10H\HM\cdot \iota.
	\]
	As a result, we have
	\begin{align*}
		\sum_{k=1}^K \sum_{h=1}^H\PB_h(\overoptV_{h+1}^k - V_{h+1}^{\pi_k})(s_{h, k}, a_{h, k}) 
		&	\le  3 H \sqrt{\iota} \cdot  \sqrt{\HM \cdot \sum_{k=1}^K \sum_{h=1}^H \PB_h(\overoptV_{h+1}^k - V_{h+1}^{\pi_k})(s_{h, k}, a_{h, k})  }  \\
		& \qquad +
		4H\beta \sum_{k=1}^K \sum_{h=2}^H  \|\Bphi_{h, k}\|_{\H_{h, k-1}^{-1} } + 10 H \HM \iota.
	\end{align*}
	Using the inequality that $x \le 2(a^2 +  b^2)$ for any $x \le |a| \sqrt{x} + b^2$, we have
	\begin{equation}
		\label{eq:final-expect-sum}
		\sum_{k=1}^K \sum_{h=1}^H\PB_h(\overoptV_{h+1}^k - V_{h+1}^{\pi_k})(s_{h, k}, a_{h, k}) 
		\le 8H\beta \sum_{k=1}^K \sum_{h=1}^H  \|\Bphi_{h, k}\|_{\H_{h, k-1}^{-1} } + 38 H^2 \HM \iota.
	\end{equation}
	Putting pieces together, we have
	\begin{align*}
		\sum_{k =1}^K&(\overoptV_{1}^k - V_1^{\pi_k})(s_{1, k})
		\overset{\eqref{eq:final0}}{\le} \sum_{k=1}^K \sum_{h=1}^H \left[ X_{h, k} + 4\beta \|\Bphi_{h, k}\|_{\H_{h, k-1}^{-1}}\right]\\
		&\overset{\eqref{eq:final1}}{\le} 4\beta  \sum_{k=1}^K \sum_{h=1}^H \|\Bphi_{h, k}\|_{\H_{h, k-1}^{-1}} + 3 \sqrt{\iota} \cdot  \sqrt{\HM \cdot \sum_{k=1}^K \sum_{h=1}^H \PB_h(\overoptV_{h+1}^k - V_{h+1}^{\pi_k})(s_{h, k}, a_{h, k})  } + 10\HM\cdot \iota \\
		&\overset{\eqref{eq:final-expect-sum}}{\le}
		4\beta  \sum_{k=1}^K \sum_{h=1}^H \|\Bphi_{h, k}\|_{\H_{h, k-1}^{-1}} + 3 \sqrt{\iota} \cdot  \sqrt{\HM \cdot \left[8H\beta \sum_{k=1}^K \sum_{h=1}^H  \|\Bphi_{h, k}\|_{\H_{h, k-1}^{-1} } + 38 H^2 \HM \iota\right]   } + 10\HM\cdot \iota \\
		&\le 6\beta  \sum_{k=1}^K \sum_{h=1}^H \|\Bphi_{h, k}\|_{\H_{h, k-1}^{-1}}   + 38 H \HM \iota
	\end{align*}
	where the last inequality uses $\sqrt{a+b} \le \sqrt{a} + \sqrt{b}$ and $2\sqrt{ab} \le a+ b$ for non-negative numbers $a, b \ge 0$.
\end{proof}

\subsection{Proof of Lemma~\ref{lem:opt-pes-gap}}
\label{proof:opt-pes-gap}

\begin{proof}[Proof of Lemma~\ref{lem:opt-pes-gap}]
	The proof main idea is similar to that in Lemma~\ref{lem:sub-gap}.
		For a given $k$, let $\kl$ denote the latest update episode before episode $k$, that is $\kl \le k < \kl+1$.
	By definition, $\overoptQ_h^k(\cdot, \cdot) \le \langle \Bphi(\cdot, \cdot),  \Btheta_{h, \kl-1} + \Bmu_{h, \kl-1}  \BoveroptV_{h+1}^{\kl} \rangle + \beta \|\Bphi(\cdot, \cdot)\|_{\H_{h, \kl-1}^{-1}}$ and  $\overpesQ_h^k(\cdot, \cdot) \ge \langle \Bphi(\cdot, \cdot),  \Btheta_{h, \kl-1} + \Bmu_{h, \kl-1}  \BoverpesV_{h+1}^{\kl} \rangle - \beta \|\Bphi(\cdot, \cdot)\|_{\H_{h, \kl-1}^{-1}}$.
	Using 
 \[
 a_{h, k} = \pi_h^k(s_{h, k}) = \argmax_{a \in \AM} \overoptQ_h^k(s_{h, k}, a),
 \]
we then have
	\begin{align*}
		(&\overoptV_{h}^k - \overpesV_h^k)(s_{h, k})
		\le (\overoptQ_h^k -\overpesQ_h^k) (s_{h, k}, a_{h, k}) \\
		&\le \langle\Bphi_{h, k},   \Bmu_{h, \kl-1}  (\BoveroptV_{h+1}^{\kl}-\BoverpesV_{h+1}^{\kl}) \rangle 
		+ 2\beta \|\Bphi_{h, k}\|_{\H_{h, \kl-1}^{-1}}\\
		&\overset{(a)}{\le}\langle\Bphi_{h, k},   (\Bmu_{h, k-1} - \Bmu_h^*)  (\BoveroptV_{h+1}^{\kl}-\BoverpesV_{h+1}^{\kl}) \rangle
		+	\langle\Bphi_{h, k},   \Bmu_h^*  (\BoveroptV_{h+1}^{\kl}-\BoverpesV_{h+1}^{\kl}) \rangle + 4\beta \|\Bphi_{h, k}\|_{\H_{h, k-1}^{-1}}\\
		&\overset{(b)}{\le}6\beta \|\Bphi_{h, k}\|_{\H_{h, k-1}^{-1}} + \langle\Bphi_{h, k}, \Bmu_h^*(  \BoveroptV_{h+1}^{\kl} -  \BoverpesV_{h+1}^{\kl}
		) \rangle \\
		&\overset{(c)}{=}6\beta \|\Bphi_{h, k}\|_{\H_{h, k-1}^{-1}} + \PB_h( \overoptV_{h+1}^{k} -  \overpesV_{h+1}^{k})(s_{h, k}, a_{h, k})\\
		&\overset{(d)}{=}6\beta \|\Bphi_{h, k}\|_{\H_{h, k-1}^{-1}} + ( \overoptV_{h+1}^{k} -  \overpesV_{h+1}^{k})(s_{h+1, k}) + X_{h, k}.
	\end{align*}
Here $(a)$ uses~\eqref{eq:bonus-ratio}, $(b)$ uses 
\begin{align*}
&|\langle\Bphi_{h, k},   (\Bmu_{h, \kl-1} - \Bmu_h^*)  (\BoveroptV_{h+1}^{\kl}-\BoverpesV_{h+1}^{\kl}) \rangle |\\
&\le  \| \Bphi_{h, k}\|_{\H_{h, \kl-1}^{-1}} \|(\Bmu_{h, \kl-1} - \Bmu_h^*)  (\BoveroptV_{h+1}^{\kl}-\BoverpesV_{h+1}^{\kl})\|_{\H_{h, \kl-1}} \\
&\le 2\beta  \| \Bphi_{h, k}\|_{\H_{h, \kl-1}^{-1}} \le 2\beta  \| \Bphi_{h, k}\|_{\H_{h, k-1}^{-1}}
\end{align*}
on $\BM_V \cap \BM_R$, $(c)$ uses $
	\langle\Bphi_{h, k}, \Bmu_h^*(  \overoptV_{h+1}^{\kl} -  \overpesV_{h+1}^{\kl}
	) \rangle  = \PB_h( \overoptV_{h+1}^{\kl} -  \overpesV_{h+1}^{\kl})(s_{h, k}, a_{h, k})  = \PB_h( \overoptV_{h+1}^{k} -  \overpesV_{h+1}^k)(s_{h, k}, a_{h, k})$, and $(d)$ uses the notation 
	\begin{equation}
		\label{eq:X-gap}
		X_{h, k} :=  \PB_h( \overoptV_{h+1}^{k} -  \overpesV_{h+1}^k)(s_{h, k}, a_{h, k}) - ( \overoptV_{h+1}^{k} -  \overpesV_{h+1}^{k})(s_{h+1, k}).
	\end{equation}
	The last inequality implies
	\begin{align*}
		(\overoptV_{h}^k - \overpesV_h^{k})(s_{h, k})
		\le ( \overoptV_{h+1}^{k} -  \overpesV_{h+1}^{k})(s_{h+1, k}) + X_{h, k} + 6\beta \|\Bphi_{h, k}\|_{\H_{h, k-1}^{-1}}.
	\end{align*}
	Iterating the above inequality over $h$ and using $ \overoptV_{H+1}^{k}(\cdot) = \overpesV_{H+1}^{k}(\cdot) = 0$, we have
	\begin{equation}
		\label{eq:gap-iter}
		(\overoptV_{h}^k - \overpesV_h^{k})(s_{h, k})
		\le \sum_{i=h}^H \left[ X_{i, k} + 6\beta \|\Bphi_{i, k}\|_{\H_{i, k-1}^{-1}}\right].
	\end{equation}
	Using the last inequality, it follows that
	\begin{align}
		\label{eq:gap0}
		\sum_{k=1}^K \sum_{h=1}^H& \PB_h(\overoptV_{h+1}^k - \overpesV_{h+1}^{k})(s_{h, k}, a_{h, k}) 
		= 	\sum_{k=1}^K \sum_{h=2}^H (\overoptV_{h}^k - \overpesV_{h}^{k})(s_{h, k}) + \sum_{k=1}^K \sum_{h=1}^H  X_{h, k} \nonumber \\
		&\overset{(a)}{\le}\sum_{k=1}^K \sum_{h=2}^H 
		\sum_{i=h}^H \left[ X_{i, k} + 8\beta \|\Bphi_{i, k}\|_{\H_{i, k-1}^{-1}}\right] + \sum_{k=1}^K \sum_{h=1}^H  X_{h, k} \nonumber \\
		&=\sum_{k=1}^K \sum_{h=2}^H (H-h+1)\left[ X_{h, k} + 6\beta \|\Bphi_{h, k}\|_{\H_{h, k-1}^{-1}}\right] + \sum_{k=1}^K \sum_{h=1}^H  X_{h, k} \nonumber \\
		&\overset{(b)}{\le} 6H\beta \sum_{k=1}^K \sum_{h=2}^H  \|\Bphi_{h, k}\|_{\H_{h, k-1}^{-1} }+ \sum_{k=1}^K \sum_{h=1}^H  X_{h, k}  b_{h}
	\end{align}
	where $(a)$ uses~\eqref{eq:final-iter} and $(b)$ uses the notation $b_h = 1$ if $h=1$; otherwise $= H-h+2$ for $2 \le h \le H$.
	Clearly, we have $|b_h| \le H$ for all $h \in [H]$.
	
	We then need to analyze $ \sum_{k=1}^K \sum_{h=1}^H  X_{h, k}  b_{h}$ with $X_{h, k}$'s defined in~\eqref{eq:X-gap}.
	Since $s_{h+1, k}$ is $\FM_{h+1, k}$-measurable,  $\overoptV_{h+1}^k, \overpesV_{h+1}^k$ is $\FM_{H, k-1}$-measurable, we have $X_{h, k}$ is $\FM_{h+1, k}$-measurable.
	We also have $\EB[X_{h, k}|\FM_{h, k}] = 0$, $|X_{h, k}| \le 2\HM$ and 
	\begin{align*}
		\EB[X_{h, k}^2|\FM_{h, k}]  &\le  \EB[ (\overoptV_{h+1}^k - \overpesV_{h+1}^{k})^2(s_{h+1, k}) |\FM_{h, k}]\\
		& \overset{(a)}{\le}  \HM \EB[ |\overoptV_{h+1}^k - \overpesV_{h+1}^{k}|(s_{h+1, k}) |\FM_{h, k}]
		=\HM \PB_h(\overoptV_{h+1}^k - \overpesV_{h+1}^{k})(s_{h, k}, a_{h, k}) 
	\end{align*}
	where $(a)$ uses $ |\overoptV_{h+1}^k - \overpesV_{h+1}^{k}|(\cdot) \le \HM$.
	By the variance-aware Freedman inequality in Lemma~\ref{lem:bern}, with probability at least $1-\delta$, it follows that
	\begin{equation}
		\label{eq:gap-final1}
		\left| \sum_{k=1}^K \sum_{h=1}^H  X_{h, k} b_h \right|
		\le 3 H \sqrt{\iota} \cdot  \sqrt{\HM \cdot \sum_{k=1}^K \sum_{h=1}^H \PB_h(\overoptV_{h+1}^k - \overpesV_{h+1}^{k})(s_{h, k}, a_{h, k})  } + 10H \HM\cdot \iota
	\end{equation}
	where $\iota = \log\frac{4\ceil{\log_2 HK}}{\delta}$.
	As a result, plugging~\eqref{eq:gap-final1} into~\eqref{eq:gap0}, we have
	\begin{align*}
		\sum_{k=1}^K \sum_{h=1}^H\PB_h(\overoptV_{h+1}^k - \overpesV_{h+1}^{k})(s_{h, k}, a_{h, k}) 
		&	\le  3 H \sqrt{\iota} \cdot  \sqrt{\HM \cdot \sum_{k=1}^K \sum_{h=1}^H \PB_h(\overoptV_{h+1}^k - \overpesV_{h+1}^{k})(s_{h, k}, a_{h, k})  }  \\
		& \qquad +
		6H\beta \sum_{k=1}^K \sum_{h=2}^H  \|\Bphi_{h, k}\|_{\H_{h, k-1}^{-1} } + 10 H \HM \iota.
	\end{align*}
	Using the inequality that $x \le 2(a^2 +  b^2)$ for any $x \le |a| \sqrt{x} + b^2$, we have
	\begin{equation*}
		\sum_{k=1}^K \sum_{h=1}^H\PB_h(\overoptV_{h+1}^k - \overpesV_{h+1}^{k})(s_{h, k}, a_{h, k}) 
		\le 12H\beta \sum_{k=1}^K \sum_{h=1}^H  \|\Bphi_{h, k}\|_{\H_{h, k-1}^{-1} } + 38 H^2 \HM \iota.
	\end{equation*}
\end{proof}

\subsection{Proof of Lemma~\ref{lem:sum-bonus}}
\label{proof:sum-bonus}

\begin{proof}[Proof of Lemma~\ref{lem:sum-bonus}]
	Recall that $b_{h, k} = \max\{
	\|\Bphi_{h, k}\|_{\H_{h, k-1}^{-1}}, \|\TBphi_{h, k}\|_{\TH_{h, k-1}^{-1}}
	\}, w_{h, k} = \sigma_{h,k}^{-1} \|\Bphi_{h, k}\|_{\H_{h, k-1}^{-1}}$ and $\tw_{h, k} = \sigma_{h,k}^{-1} \|\TBphi_{h, k}\|_{\TH_{h, k-1}^{-1}}$.
	As a result, we have $\sigma_{h, k}^{-1} b_{h, k} = \max \left\{ w_{h, k}, \tw_{h, k} \right\}$.
	On the other hand,
	\begin{equation}
 \tag{\ref{eq:sigma}}
		\sigma_{h,k}^2 = \max\left\{  \sigma_{\min}^2,
		d^3H \cdot E_{h, k}, J_{h, k}, 
		c_0^{-2}b_{h, k}^2, \left( \frac{W}{\sqrt{c_1 d}} + \HM d^{2.5}H \right) b_{h, k}
		\right\}.
	\end{equation}	
	Based on what value $\sigma_{h, k}$ takes, we compose the full index set $\IM := [H] \times [K]$ into three disjoint sets with ties broken arbitrarily:
	\begin{gather*}
		\JM_{1} 
		= \left\{
		(h, k) \subseteq [H] \times [K]: 	\sigma_{h,k}^2 \in \left\{  \sigma_{\min}^2,
		d^3H \cdot E_{h, k}, U_{h, k} \right\}
		\right\}, \\
		\JM_{2} 
		= \left\{
		(h, k) \subseteq [H] \times [K]: \sigma_{h,k}^2 =c_0^{-2} b_{h, k}^2
		\right\}, \\
		\JM_{3} 
		= \left\{
		(h, k) \subseteq [H] \times [K]: 	\sigma_{h,k}^2 = \left( \frac{W}{\sqrt{c_1 d}} + \HM d^{2.5}H \right) b_{h, k}
		\right\}.
	\end{gather*}
	For simplicity, we denote $z_{h, k} :=\frac{b_{h, k}}{\sigma_{h, k}} =  \max\{w_{h, k}, \tw_{h, k}\}$.
	Therefore,
	\begin{align}
		\label{eq:three-terms}
		\sum_{k=1}^K \sum_{h=1}^H b_{h, k}
		=\sum_{(h, k) \in \IM}\sigma_{h, k} z_{h, k} =\sum_{i=1}^3 \sum_{(h, k) \in \JM_i}\sigma_{h, k} z_{h, k}.
	\end{align}
	Recall that  $\kappa = d \log\left(1 + \frac{K}{d\lambda \sigma_{\min}^2}\right)$, we have $\sum_{(h, k) \in \IM } z_{h, k}^2 \le4 H \kappa$.
	This is because
	\begin{align*}
		\sum_{(h, k) \in \IM } z_{h, k}^2
		&\le \sum_{(h, k) \in \IM } \left(w_{h, k}^2 +\tw_{h, k}^2\right) \overset{(a)}{=}\sum_{k=1}^K \sum_{h=1}^H \min \left\{1, w_{h, k}^2 \right\} + \sum_{k=1}^K \sum_{h=1}^H \min \left\{1, \tw_{h, k}^2 \right\}\\
		&\overset{(b)}{\le} 4Hd \log \left( 1 + \frac{K}{d\lambda \sigma_{\min}^2} \right) =4 H \kappa.
	\end{align*}
	where $(a)$ uses $z_{h, k} \le c_0 \le 1$ due to $\sigma_{h, k} \ge c_0^{-1} b_{h, k}, c_0 \le 1$ and $(b)$ uses Lemma~\ref{lem:w-sum}.
	We will frequently use the above inequality.
	
	Now, we are ready to analyze the three terms in the RHS of~\eqref{eq:three-terms} respectively.
	\begin{itemize}[leftmargin=*]
		\item For the first term, it follows that
		\begin{align*}
			\sum_{(h, k) \in \JM_1}\sigma_{h, k}z_{h, k}
			&\le  \sqrt{ \sum_{(h, k) \in \JM_1}\sigma_{h, k}^2} \sqrt{ \sum_{(h, k) \in \JM_{1}}z_{h, k}^2}\\
			&\le \sqrt{ \sum_{(h, k) \in \JM_1}  (\sigma_{\min}^2 + 
				d^3H\cdot E_{h, k} +  J_{h, k}) } \sqrt{ \sum_{(h, k) \in \JM_1 }z_{h, k}^2}\\
			&\le  \sqrt{ \sum_{(h, k) \in \IM}  (\sigma_{\min}^2 + 
				d^3H \cdot E_{h, k} +  J_{h, k}) } \sqrt{ \sum_{(h, k) \in \IM }z_{h, k}^2}\\
			&\le  \sqrt{ HK\sigma_{\min}^2 +  \sum_{(h, k) \in \IM} 
				(d^3H \cdot E_{h, k} +  J_{h, k}) } \cdot \sqrt{ 4H \kappa}.
		\end{align*}
		%	On the event $\BM_{R}$, by definition of $R_{h,k}$, we have
		%	\[
		%\sum_{(h, k) \in \IM} R_{h, k}  = \sum_{(h, k) \in \IM}	[\VB_h R_h](s_{h, k}, a_{h, k}) + 
		%(2\HM+1) \beta_{0, R} \cdot \sum_{(h, k) \in \IM} b_{h, k}
		%	\]
		
		We provide a upper bound for $\sum_{k=1}^K \sum_{h=1}^HE_{h, k}$ in Lemma~\ref{lem:sum-E} whose proof is deferred in Appendix~\ref{proof:sum-E}.
		\begin{lem}[Sum of $E_{h, k}$]
			\label{lem:sum-E}
			On the event $\BM_0 \cap \AM_0$,
			\[
			\sum_{k=1}^K \sum_{h=1}^HE_{h, k} = \OM \left(
			(\beta_0+H\beta) \HM\cdot \sum_{k=1}^K \sum_{h=1}^H
			\|\Bphi_{h, k}\|_{\H_{h, k-1}^{-1}}
			+ H^2 \HM^2 \log\frac{4\ceil{\log_2 HK}}{\delta}
			\right) .
%			(4\beta_0+16H\beta) \HM\cdot \sum_{k=1}^K \sum_{h=1}^H
%			\|\Bphi_{h, k}\|_{\H_{h, k-1}^{-1}}
%			+ 38 H^2 \HM^2 \log\frac{4\ceil{\log_2 HK}}{\delta}.
			\]
			where $\OM(\cdot)$ hides universal positive constants.
		\end{lem}
		We also provide a upper bound for $\sum_{k=1}^K \sum_{h=1}^H J_{h, k}$ in Lemma~\ref{lem:sum-J} whose proof is deferred in Appendix~\ref{proof:sum-J}.
		\begin{lem}[Sum of $J_{h, k}$]
			\label{lem:sum-J}
			Recall that $ J_{h, k} =  [\HVB_h R_h +  \HVB_h \overoptV_{h+1}^k](s_{h, k}, a_{h, k}) +R_{h, k} + U_{h, k}$ with $R_{h, k}, U_{h, k}$ defined in~\eqref{eq:R} and~\eqref{eq:U} respectively.
			On the event $\BM_R \cap \BM_V \cap \BM_0 \cap \AM_0$, with probability at least $1-2\delta$,
			\begin{align*}
		\sum_{k=1}^K \sum_{h=1}^H J_{h, k} 
		=\OM\left(
		\GM^*K  + [(\beta_0 + H \beta)  \HM  + \beta_{R^2}] \sum_{k=1}^K \sum_{h=1}^H  b_{h, k}+ H^2 \HM^2\log\frac{4\ceil{\log_2 HK}}{\delta} + H \sigma_{R}^2  \log \frac{1}{\delta}
		\right).
			\end{align*}
			
%			\begin{align*}
%				\sum_{k=1}^K \sum_{h=1}^H J_{h, k} 
%				&\le \min \left\{ \GM_0^* , 2\VM^2
%				\right\} \cdot K +  [(26\beta_0 + 138H \beta)  \HM  +2\beta_{R^2}] \sum_{k=1}^K \sum_{h=1}^H  b_{h, k} \\
%				& \qquad + 344 H^2 \HM^2 \log\frac{4\ceil{\log_2 HK}}{\delta} + 7 H \sigma_{R}^2 \log \frac{2\ceil{\log_2 K}}{\delta}.
%			\end{align*}
			where $\GM^*$ is defined in~\eqref{eq:G} and $\OM(\cdot)$ hides universal positive constants.
		\end{lem}
%		From Lemma~\ref{lem:sum-E}, on the event $\BM_0 \cap \AM_0$,
%		\[
%		\sum_{(h, k) \in \IM} E_{h, k} \le (4\beta_0+16H\beta) \HM\cdot \sum_{(h, k) \in \IM}
%		b_{h, k}
%		+ 38 H^2 \HM^2 \log\frac{4\ceil{\log_2 HK}}{\delta}.\]
%		From Lemma~\ref{lem:sum-J}, on the event $\BM_R \cap \BM_0 \cap \BM_1 \cap \AM_0$, with probability at least $1-\delta$, 
%		\[
%		\sum_{(h, k) \in \IM} J_{h, k} \le 
%		2\GM^*K +  [(26\beta_0 + 138H \beta)  \HM  +\beta_{R^2}] 	\sum_{(h, k) \in \IM}   b_{h, k}+ 344 H^2 \HM^2 \log\frac{4\ceil{\log_2 HK}}{\delta} + 2 H \sigma_{R}^2 \log \frac{1}{\delta}.
%		\]
		Putting pieces together and using $\sqrt{a+b+c} \le \sqrt{a} + \sqrt{b} + \sqrt{c}$, we have
		\begin{align}
			\begin{split}
				\label{eq:J1}
				\sum_{(h, k) \in \JM_1}  b_{h, k}&=\sum_{(h, k) \in \JM_1} \sigma_{h, k}z_{h, k}
				= \OM\left( \sqrt{ H \kappa} \cdot \sqrt{ 
					K  \left(H \sigma_{\min}^2 +   \GM^*  \right)
				} \right) \\
				& \quad +   \OM\left( \sqrt{ H \kappa} \cdot \sqrt{
				H^3d^3\HM^2\log\frac{4\ceil{\log_2 HK}}{\delta}  + H \sigma_{R}^2 \log \frac{1}{\delta}.
				}  \right) \\
				& \quad  +   \OM\left( \sqrt{ H \kappa} \cdot  \sqrt{
					[(\beta_0 + H \beta)  \HM d^3H +\beta_{R^2}] 	\sum_{(h, k) \in \IM}   b_{h, k}
				}  \right) .
			\end{split}
		\end{align}

		\item For the second term, due to $\sigma_{h,k} =c_0^{-1} b_{h, k}$, we have $z_{h, k} = b_{h, k}/\sigma_{h, k} = c_0 \le 1$ for all $(h, k) \in \JM_{2}$.
		Hence,
		\begin{align}
			\label{eq:J2}
			\sum_{(h, k) \in \JM_2}  b_{h, k}&=
			\sum_{(h, k) \in \JM_2}\sigma_{h, k}z_{h, k}
			= \frac{1}{c_0}\sum_{(h, k) \in \JM_2}\sigma_{h, k}z_{h, k}^2
			\le \frac{\sup_{(h, k) \in \IM} \sigma_{h, k} }{c_0}\sum_{(h, k) \in \JM_2}z_{h, k}^2\nonumber \\
			&\le \sup_{(h, k) \in \IM}  \frac{ \max\{ \|\Bphi_{h, k}\|_{\H_{h, k-1}^{-1}},\|\TBphi_{h, k}\|_{\TH_{h, k-1}^{-1}} \}}{c_0^2} \cdot \sum_{(h, k) \in \IM}z_{h, k}^2
			\le \frac{4H\kappa}{c_0^2 \sqrt{\lambda}}
		\end{align}
		where the last inequality uses $\|\Bphi_{h, k}\|_{\H_{h, k-1}^{-1}} \le \frac{1}{\sqrt{\lambda}}\|\Bphi_{h, k}\| \le \frac{1}{\sqrt{\lambda}}$ and  $\|\TBphi_{h, k}\|_{\TH_{h, k-1}^{-1}} \le \frac{1}{\sqrt{\lambda}}\|\TBphi_{h, k}\| \le \frac{1}{\sqrt{\lambda}}$ for any $(h, k) \in \IM$.

		\item For the third term, $\sigma_{h,k}^2 = \left( \frac{W}{\sqrt{c_1 d}} + \HM d^{2.5}H \right) b_{h, k}$ and thus $\sigma_{h,k} =  \left( \frac{W}{\sqrt{c_1 d}} + \HM d^{2.5}H \right)  z_{h, k}$.
		Hence,
		\begin{align}
			\label{eq:J3}
			\sum_{(h, k) \in \JM_3}  b_{h, k}&=
			\sum_{(h, k) \in \JM_3}\sigma_{h, k}z_{h, k}
			=  \left( \frac{W}{\sqrt{c_1 d}} + \HM d^{2.5}H  \right) \sum_{(h, k) \in \JM_3}z_{h, k}^2 \nonumber \\
			&\le  \left( \frac{W}{\sqrt{c_1 d}} + \HM d^{2.5}H  \right) \sum_{(h, k) \in \IM}z_{h, k}^2
			\le 4H \kappa \cdot \left( \frac{W}{\sqrt{c_1 d}} + \HM d^{2.5}H  \right) .
		\end{align}
	\end{itemize}
	Combing~\eqref{eq:J1},~\eqref{eq:J2} and~\eqref{eq:J3}, we have
	\[
	\sum_{(h, k) \in \IM}b_{h, k}
	= \OM \left( C + \sqrt{ H \kappa}  \sqrt{
		[(\beta_0+H\beta  ) \HM d^3 H + \beta_{R^2}]\cdot \sum_{(h, k) \in \IM} b_{h, k}
	} \right)
	\]
	where
	\begin{align*}
		C &= \sqrt{ H \kappa} \cdot \sqrt{ 
			K  \left(H \sigma_{\min}^2 +  \GM^* \right)} + H \kappa \cdot \left( \frac{W}{\sqrt{c_1 d}} + \frac{1}{c_0^2 \sqrt{\lambda}}  + \HM d^{2.5}H \right) \\
		& \quad + \sqrt{ H \kappa}  \cdot \sqrt{
		H^3 d^3 \HM^2  \log\frac{4\ceil{\log_2 HK}}{\delta} + H \sigma_{R}^2 \log \frac{1}{\delta}} .
	\end{align*}
	Using the inequality that $x \le 2(a^2 +  b^2)$ for any $x \le |a| \sqrt{x} + b^2$, we have
	\[
	\sum_{(h, k) \in \IM}b_{h, k}
	= \OM\left(
	C + H^2\HM\kappa d^3 \left(
	\beta_0+H\beta 
	\right) +  H \kappa \beta_{R^2}
	\right).
	\]
	In the following, we are going to simplify the last inequality.
	We will use $\TOM(\cdot)$ to hide logarithmic factors for simplicity.
	Notice that $\kappa = \TOM(d)$.
	By setting $\lambda = \frac{1}{\HM^2 + W^2}$, we have $\beta_R = \beta_V =  \TOM(\sqrt{d})$ and thus $\beta = \beta_V + \beta_R =  \TOM(\sqrt{d})$.
	Moreover, $\beta_{R^2} =  \TOM\left(\sqrt{d} + \sqrt{d} \frac{\sigma_{R^2}}{\sigma_{\min}}+ \sqrt{\lambda} W\right) =  \TOM\left(\sqrt{d} + \sqrt{d} \frac{\sigma_{R^2}}{\sigma_{\min}} \right)$ and $\beta_0 = \TOM\left( \frac{\sqrt{d^3H}\HM}{\sigma_{\min}} + \sqrt{d\lambda} \HM \right) =  \TOM\left( \frac{\sqrt{d^3H}\HM}{\sigma_{\min}} + \sqrt{d} \right) $.
	Therefore, 
	\begin{align*}
		\sum_{(h, k) \in \IM}b_{h, k}
		&	= \OM\left(
		C + H^2\HM\kappa d^3 \left(
		\beta_0+H\beta 
		\right) +  H \kappa \beta_{R^2}
		\right) \\
		&=  \OM(C) + \TOM \left( 
		\frac{H^{2.5}d^{5.5}\HM^2 + Hd^{1.5} \sigma_{R^2}}{\sigma_{\min}}
		+ H^3 d^{4.5} \HM + H d^{1.5} \right).
	\end{align*}
	We then analyze $C$.
	Using $\sqrt{a+b} \le \sqrt{a} + \sqrt{b}$ for non-negative numbers $a, b \ge 0$, we have
	\begin{align*}
		C = \TOM \left(
		\sqrt{d H K  \GM^*}  +
		H  d^{0.5} K^{0.5} \sigma_{\min} + H^2 d^{3.5} \HM + H^2 d^2 \HM + H d^{0.5} \sigma_{R} + Hd
		\right).
	\end{align*}
	Putting the results together, we have
	\begin{align*}
		\sum_{(h, k) \in \IM}b_{h, k}
		= \TOM \left(
		\sqrt{d H K  \GM^*}   + H  d^{0.5} K^{0.5} \sigma_{\min}  + \frac{H^{2.5}d^{5.5}\HM^2 + Hd^{1.5} \sigma_{R^2}}{\sigma_{\min}} +  H^3 d^{4.5} \HM +  H d^{0.5} \sigma_{R} + Hd^{1.5}
		\right).
	\end{align*}
\end{proof}

\subsubsection{Proof of Lemma~\ref{lem:sum-E}}
\label{proof:sum-E}
\begin{proof}[Proof of Lemma~\ref{lem:sum-E}]
	By the definition of $E_{h, k}$ in~\eqref{eq:E}, it follows that
	\begin{align*}
		\sum_{k=1}^K \sum_{h=1}^HE_{h, k}
		&\le \sum_{k=1}^K \sum_{h=1}^H
		\left[
		2\HM\beta_0 \|\Bphi_{h, k}\|_{\H_{h, k-1}^{-1}}
		+ \HM \cdot	\left[ \HPB_{h, k}(\overoptV_{h+1}^k - \overpesV_{h+1}^k)\right](s_{h, k}, a_{h, k})
		\right]\\
		&\overset{(a)}{\le} \sum_{k=1}^K \sum_{h=1}^H
		\left[
		4\HM\beta_0 \|\Bphi_{h, k}\|_{\H_{h, k-1}^{-1}}
		+ \HM \cdot	\left[ \PB_{h}(\overoptV_{h+1}^k - \overpesV_{h+1}^k)\right](s_{h, k}, a_{h, k})
		\right]\\
		&\overset{(b)}{\le}(4\beta_0+16H\beta) \HM\cdot \sum_{k=1}^K \sum_{h=1}^H
		\|\Bphi_{h, k}\|_{\H_{h, k-1}^{-1}}
		+ 38 H^2 \HM^2 \log\frac{4\ceil{\log_2 HK}}{\delta}\\
		&= \OM \left(
		(\beta_0+H\beta) \HM\cdot \sum_{k=1}^K \sum_{h=1}^H
		\|\Bphi_{h, k}\|_{\H_{h, k-1}^{-1}}
		+ H^2 \HM^2 \log\frac{4\ceil{\log_2 HK}}{\delta}
		\right) 
	\end{align*}
	where $(a)$ uses $\left| [(\HPB_{h, k}-\PB_h)\overoptV_{h+1}^k](s_{h, k}, a_{h, k})  \right|= |\langle \Bphi_{h, k}, (\Bmu_{h, k-1}-\Bmu_h^*) \overoptV_{h+1}^k \rangle| \le \beta_0  \|\Bphi_{h, k}\|_{\H_{h, k-1}^{-1}}$ on $\BM_0$ and $(b)$ follows from Lemma~\ref{lem:opt-pes-gap}.
\end{proof}

\subsubsection{Proof of Lemma~\ref{lem:sum-J}}
\label{proof:sum-J}
\begin{proof}[Proof of Lemma~\ref{lem:sum-J}]
	By Lemma~\ref{lem:CI-rewards}, on the event $\BM_R$, we have $\left| [\HVB_h \HR_h-\VB_h R_h](s_{h, k}, a_{h, k}) \right| \le R_{h, k}$ for all $h \in [H]$ and $k \in [K]$.
	By Lemma~\ref{lem:var0}, on the event $\BM_0 \cap \BM_V$, $[\HVB_h \overoptV_{h+1}^k](s_{h, k}, a_{h, k}) \le [\VB_h V_{h+1}^*](s_{h, k}, a_{h, k}) + U_{h, k}$ for all $h \in [H]$ and $k \in [K]$.
	Therefore,
	\begin{align*}
		\sum_{k=1}^K\sum_{h=1}^H J_{h, k} &\le \sum_{k=1}^K \sum_{h=1}^H  [\VB_h R_h+\VB_h V_{h+1}^*](s_{h, k}, a_{h, k})  +2 \sum_{k=1}^K \sum_{h=1}^H R_{h, k}+ 2 \sum_{k=1}^K \sum_{h=1}^H U_{h, k}\\
		& := (I) + (II) + (III).
	\end{align*}
	
	For the term $(III)$, we have
	\begin{align}
		\label{eq:sum-U}
		\sum_{k=1}^K \sum_{h=1}^H U_{h, k}
		&= \sum_{k=1}^K \sum_{h=1}^H \left[11\HM \beta_0\cdot \|\Bphi_{h, k}\|_{\H_{h, k-1}^{-1}} + 4\HM \cdot \HPB_{h, k}(\overoptV_{h+1}^k -\overpesV_{h+1}^k)(s_{h, k}, a_{h, k}) \right] \nonumber \\
		&\overset{(a)}{\le}  \sum_{k=1}^K \sum_{h=1}^H \left[19\HM \beta_0\cdot \|\Bphi_{h, k}\|_{\H_{h, k-1}^{-1}} + 4\HM \cdot \PB_{h}(\overoptV_{h+1}^k -\overpesV_{h+1}^k)(s_{h, k}, a_{h, k}) \right]\nonumber \\
		&\overset{(b)}{\le} (19\beta_0 + 64 H \beta) \HM \cdot	\sum_{k=1}^K \sum_{h=1}^H \|\Bphi_{h, k}\|_{\H_{h, k-1}^{-1}} + 152H^2 \HM^2 \log\frac{4\ceil{\log_2 HK}}{\delta},
	\end{align}
	where $(a)$ uses $\left| [(\HPB_{h, k}-\PB_h)\overoptV_{h+1}^k](s_{h, k}, a_{h, k})  \right|= |\langle \Bphi_{h, k}, (\Bmu_{h, k-1}-\Bmu_h^*) \overoptV_{h+1}^k \rangle| \le \beta_0  \|\Bphi_{h, k}\|_{\H_{h, k-1}^{-1}}$ on $\BM_0$; and $(b)$ follows from Lemma~\ref{lem:opt-pes-gap}.
	
	For the term $(II)$, we have
	\begin{equation}
		\label{eq:sum-R}
		\sum_{k=1}^K \sum_{h=1}^H R_{h, k} 
		=  \beta_{R^2} \sum_{k=1}^K \sum_{h=1}^H \| \TBphi_{h, k}\|_{\TH_{h, k-1}^{k-1}}  
		+  2\HM \beta_R  \sum_{k=1}^K \sum_{h=1}^H  \|\Bphi_{h, k}\|_{\H_{h, k-1}^{-1}}.
	\end{equation}
	
	We provide two ways to analyze the term $(I)$.
	\begin{itemize}[leftmargin=*]
		\item On one hand, we denote $X_k =  \sum_{h=1}^H [\VB_h R_h + \VB_h V_{h+1}^{*}](s_{h, k}, a_{h, k})$ for simplicity.
		Let $\GM_k:=\FM_{H, k}$ be the $\sigma$-field generated by all the random variables over the first $k$ episodes.
		Then $\pi_k$ is $\GM_{k-1}$-measurable, $X_k \ge 0$ is $\GM_k$-measurable, and $|X_k| \le H (\sigma_{R}^2+ \HM^2)$.
		Therefore, $|X_k - \EB[X_k|\GM_{k-1}]| \le H (\sigma_{R}^2+ \HM^2)$ and $\Var[X_k|\GM_{k-1}] \le H (\sigma_{R}^2+ \HM^2) \cdot \EB [X_k|\GM_{k-1}]$.
		By the variance-aware Freedman inequality in Lemma~\ref{lem:bern}, with probability at least $1-\delta$, we have
		\begin{align*}
		\sum_{k=1}^K X_k  
		&\le \sum_{k=1}^K \EB[X_k|\FM_{k-1}] 
		+ 3 \sqrt{H(\sigma_{R}^2+ \HM^2)  \sum_{k=1}^K \EB [X_k|\GM_{k-1}] \log \frac{2\ceil{\log_2K}}{\delta}} \\
		& \qquad + 5H (\sigma_{R}^2+ \HM^2) \log\frac{2\ceil{\log_2K}}{\delta} \\
		&\le 3\sum_{k=1}^K \EB[X_k|\FM_{k-1}] 
		+ 7H (\sigma_{R}^2+ \HM^2) \log\frac{2\ceil{\log_2K}}{\delta}.
		\end{align*}
		Notice that 
		\begin{align*}
		\EB[X_k|\FM_{k-1}]  &= \EB \left[ \sum_{h=1}^H [\VB_h R_h + \VB_h V_{h+1}^{*}](s_{h, k}, a_{h, k})\bigg| \GM_{k-1}\right]\\
		&=  \sum_{h=1}^H \sum_{(s, a)} d_h^{\pi_k}(s, a)[\VB_h R_h + \VB_h V_{h+1}^{*}](s, a)
		\end{align*}
		where $d_h^{\pi_k}(s, a) = \PB^{\pi_k}(s_h=s, a_h=a|s_0=s_{1,k})$ is the probability reaching $(s_{h,k}, a_{h, k}) = (s, a)$ at the $h$-th step when the agent starts from $s_{1, k}$ and follows the policy $\pi_k$.
		Therefore, we have
		\begin{align*}
		(I)  &\le 3 \sum_{k=1}^K\sum_{h=1}^H \sum_{(s, a)} d_h^{\pi_k}(s, a)[\VB_h R_h + \VB_h V_{h+1}^{*}](s, a) + 7H (\sigma_{R}^2+ \HM^2) \log\frac{2\ceil{\log_2K}}{\delta} \\
		&\le 3\GM_0^* K +  7H (\sigma_{R}^2+ \HM^2) \log\frac{2\ceil{\log_2K}}{\delta}
		\end{align*}
		where
		\[
		\GM_0^* = \frac{1}{K}\sum_{k=1}^K\sum_{h=1}^H \sum_{(s, a)} d_h^{\pi_k}(s, a)[\VB_h R_h + \VB_h V_{h+1}^{\pi_k}](s, a).
		\]
					
%		Set $\GM_0^* = \sum_{h=1}^H \sup_{(s, a) \in \SM \times \AM}[\VB_h R_h + \VB_h V_{h+1}^*](s, a) $.
%		It is obviously that
%		\[
%		(I) \le \GM_0^* \cdot K .
%		\]

		\item On the other hand, we have
		\begin{align*}
			(I) &= \sum_{k=1}^K \sum_{h=1}^H \left[\VB_h V_{h+1}^* - \VB_h V_{h+1}^{\pi_k}	\right] (s_{h, k}, a_{h, k}) 
			+ \sum_{k=1}^K \sum_{h=1}^H  [\VB_h R_h + \VB_h V_{h+1}^{\pi_k}](s_{h, k}, a_{h, k})\\
			&\overset{\eqref{eq:var-diff}}{\le} 2\HM \cdot \sum_{k=1}^K \sum_{h=1}^H \PB_h (\overoptV_{h+1}^k - V_{h+1}^{\pi_k}) (s_{h, k}, a_{h, k})  + \sum_{k=1}^K \sum_{h=1}^H  [\VB_h R_h + \VB_h V_{h+1}^{\pi_k}](s_{h, k}, a_{h, k})\\
			&\le 2\HM \cdot \sum_{k=1}^K \sum_{h=1}^H \PB_h (\overoptV_{h+1}^k - V_{h+1}^{\pi_k}) (s_{h, k}, a_{h, k})  + 2 \VM^2 K + 2H ( \sigma_{R}^2 + \HM^2 )\log\frac{1}{\delta}\\
			&\le  2 \VM^2 K + 2H ( \sigma_{R}^2 + \HM^2 )\log\frac{1}{\delta} + 16H\beta \HM \sum_{k=1}^K \sum_{h=1}^H  \|\Bphi_{h, k}\|_{\H_{h, k-1}^{-1} } + 76 H^2 \HM^2 \log\frac{4\ceil{\log_2 HK}}{\delta}\\
			&\le 2 \VM^2 K  + 16H\beta \HM \sum_{k=1}^K \sum_{h=1}^H  \|\Bphi_{h, k}\|_{\H_{h, k-1}^{-1} } + 78H^2 \HM^2  \log\frac{4\ceil{\log_2 HK}}{\delta} + 2 H \sigma_{R}^2 \log \frac{1}{\delta}
		\end{align*}
		where the first inequality uses~\eqref{eq:var-diff}, the second inequality uses Lemma~\ref{lem:total-varaince}, and the third inequality uses Lemma~\ref{lem:sub-gap}.
		\begin{align}
			\label{eq:var-diff}
			\left[\VB_h V_{h+1}^* - \VB_h V_{h+1}^{\pi_k}	\right] (s_{h, k}, a_{h, k}) 
			&= \PB_h [ V_{h+1}^*]^2 (s_{h, k}, a_{h, k})  -[ \PB_h V_{h+1}^*(s_{h, k}, a_{h, k}) ]^2 \nonumber \\
			&\quad - \left(
			\PB_h [ V_{h+1}^{\pi_k}]^2 (s_{h, k}, a_{h, k})  -[ \PB_h V_{h+1}^{\pi_k}(s_{h, k}, a_{h, k}) ]^2 \right) \nonumber \\
			&\overset{(a)}{\le}  \PB_h [ V_{h+1}^*]^2 (s_{h, k}, a_{h, k}) - \PB_h [ V_{h+1}^{\pi_k}]^2 (s_{h, k}, a_{h, k}) \nonumber \\
			&\overset{(b)}{\le} 2\HM \cdot \PB_h (V_{h+1}^* - V_{h+1}^{\pi_k}) (s_{h, k}, a_{h, k}) \nonumber \\
			&\overset{(c)}{\le} 2\HM \cdot \PB_h (\overoptV_{h+1}^k - V_{h+1}^{\pi_k}) (s_{h, k}, a_{h, k})
		\end{align}
		where $(a)$ uses $ V_{h+1}^*(\cdot) \ge  V_{h+1}^{\pi_k}(\cdot)$, $(b)$ uses $V_{h+1}^{\pi_k}(\cdot) \le V_{h+1}^*(\cdot) \le \HM$, and $(c)$ uses Lemma~\ref{lem:opt-pes}.
	\end{itemize}
	
	Finally, we are going to put pieces together.
	In order to simplicity notation, we use $b_{h, k} = \max\{
	\|\Bphi_{h, k}\|_{\H_{h, k-1}^{-1}}, \|\TBphi_{h, k}\|_{\TH_{h, k-1}^{-1}}
	\}$ and $\beta = \beta_V + \beta_R$.
	From the first bullet point, we have
	\begin{align*}
		\sum_{k=1}^K \sum_{h=1}^H J_{h, k}
	=\OM\left(
		\GM_0^* \cdot K + [(\beta_0 + H \beta) \HM +  \beta_{R^2} ]\cdot	\sum_{k=1}^K \sum_{h=1}^Hb_{h, k}+ H^2 \HM^2\log\frac{4\ceil{\log_2 HK}}{\delta} + H \sigma_{R}^2  \log \frac{1}{\delta}
	\right) .
	\end{align*}
		From the second bullet point, we have
	\begin{align*}
		\sum_{k=1}^K \sum_{h=1}^H J_{h, k}
		=\OM\left(
		\VM^2 K  + [(\beta_0 + H \beta)  \HM  + \beta_{R^2}] \sum_{k=1}^K \sum_{h=1}^H  b_{h, k}+ H^2 \HM^2\log\frac{4\ceil{\log_2 HK}}{\delta} + H \sigma_{R}^2  \log \frac{1}{\delta}
		\right) .
	\end{align*}
%	\begin{align*}
%	&\le
%	\left\{
%	\begin{array}{l}
%		3\GM_0^* \cdot K + [(26\beta_0 + 130H \beta) \HM +  2\beta_{R^2} ]\cdot	\sum_{k=1}^K \sum_{h=1}^Hb_{h, k}+ 311H^2 \HM^2 \log\frac{4\ceil{\log_2 HK}}{\delta} + 7 H \sigma_{R}^2 \log \frac{2\ceil{\log_2 K}}{\delta}\\
%		2 \VM^2 K  + [(38\beta_0 + 148H \beta)  \HM  + 2\beta_{R^2}] \sum_{k=1}^K \sum_{h=1}^H  b_{h, k}+ 382 H^2 \HM^2 \log\frac{4\ceil{\log_2 HK}}{\delta} + 2 H \sigma_{R}^2 \log \frac{1}{\delta}
%	\end{array}
%	\right.\\
%	&\le \min \left\{ \GM_0^* , 2\VM^2
%	\right\} \cdot K +  [(26\beta_0 + 138H \beta)  \HM  +2\beta_{R^2}] \sum_{k=1}^K \sum_{h=1}^H  b_{h, k}+ 344 H^2 \HM^2 \log\frac{4\ceil{\log_2 HK}}{\delta} + 7 H \sigma_{R}^2 \log \frac{2\ceil{\log_2 K}}{\delta}.
%\end{align*}
Taking minimum of the last two inequalities and using $\min\left\{ \GM_0^* , \VM^2
\right\}  \le \GM^*$ complete the proof.
\end{proof}

\subsubsection{Proof of Lemma~\ref{lem:total-varaince}}
\begin{lem}[Total variance lemma]
	\label{lem:total-varaince}
	With probability at least $1-\delta$, we have
	\[
	\sum_{k=1}^K \sum_{h=1}^H [\VB_h R_h + \VB_h V_{h+1}^{\pi_k}](s_{h, k}, a_{h, k})  \le  2 \VM^2 K + 2H (\sigma_{R}^2+ \HM^2) \log\frac{1}{\delta} .
	\]
\end{lem}
\begin{proof}[Proof of Lemma~\ref{lem:total-varaince}]
	The	proof uses a similar argument as Lemma C.5 in~\citep{jin2018q}.
	Notice that the first state $s_{1, k}$ is fixed and $a_{h, k} = \pi_h^k(s_{h, k})$.
	Therefore, $(s_{2, k}, \cdots, s_{H, k})$ is a sequence generated by following policy $\pi_k$ starting at $s_{1, k}$.
	Let $\GM_k$ be the $\sigma$-field generated by all the random variables over the first $k$ episodes.
	$X_k =  \sum_{h=1}^H [\VB_h R_h + \VB_h V_{h+1}^{\pi_k}](s_{h, k}, a_{h, k})$.
	We have the following properties about $X_k$.
	Clearly $\pi_k$ is $\GM_{k-1}$-measurable, $X_k \ge 0$ is $\GM_k$-measurable, and $|X_k| \le H (\sigma_{R}^2+ \HM^2)$.
	
	Let $\EB_k(\cdot) := \EB[\cdot|\GM_k]$ for simplicity.
	\begin{align*}
		\VM^2 &\ge \EB_{k-1}\left[ \sum_{h=1}^H R_h(s_{h, k}, a_{h, k}) - V_1^{\pi_k}(s_{1, k})\right]^2\\
		& \overset{(a)}{=}\EB_{k-1}\left[ \sum_{h=1}^H \left( R_h(s_{h, k}, a_{h, k})  + V_{h+1}^{\pi_k}(s_{{h+1}, k}) - V_h^{\pi_k}(s_{h, k})  \right) \right]^2\\
		& \overset{(b)}{=} \sum_{h=1}^H \EB_{k-1}\left[R_h(s_{h, k}, a_{h, k})  + V_{h+1}^{\pi_k}(s_{{h+1}, k}) - V_h^{\pi_k}(s_{h, k})  \right]^2\\
		& \overset{(c)}{=} \sum_{h=1}^H \EB_{k-1} \left[ [R_h-r_h]^2(s_{h, k}, a_{h, k})
		+\left[r_h(s_{h, k}, a_{h, k})  + V_{h+1}^{\pi_k}(s_{{h+1}, k}) - V_h^{\pi_k}(s_{h, k})  \right]^2 \right]\\
		&\overset{(d)}{=}  \EB_{k-1}\sum_{h=1}^H  [\VB_h R_h+ \VB_h V_{h+1}^{\pi_k}](s_{h, k}, a_{h, k}) = \EB[X_k|\FM_{k-1}]
	\end{align*}
	where $(a)$ uses $V_{H+1}^{\pi_k}(\cdot) = 0$, $(b)$ uses the independence due to the Markov
	property, $(c)$ holds since $R_h(s_{h, k}, a_{h, k})$ is independent with $s_{h+1, k}$ conditioning on $(s_{h, k}, a_{h, k})$, and $(d)$ uses $V_h^{\pi_k}(s_{h, k}) = r_h(s_{h, k}, a_{h, k})  +\EB_{s_{h+1,k} \sim \PB_h(\cdot|s_{h, k}, a_{h, k})}[V_{h+1}^{\pi_k}(s_{h+1,k})]$.
	Using $\Var[X_k|\GM_{k-1}] \le H (\sigma_{R}^2+ \HM^2) \cdot \EB [X_k|\GM_{k-1}]$, we have
	\begin{align*}
		\sum_{k=1}^K \Var[X_k|\GM_{k-1}]
		\le H (\sigma_{R}^2+ \HM^2) \cdot	\sum_{k=1}^K  \EB [X_k|\GM_{k-1}] \le (\sigma_{R}^2+ \HM^2) \VM^2 HK.
	\end{align*}
	By the Freedman inequality in Lemma~\ref{lem:freedman}, with probability at least $1-\delta$, we have
	\begin{align*}
		\sum_{k=1}^K &\sum_{h=1}^H [\VB_h R_h + \VB_h V_{h+1}^{\pi_k}](s_{h, k}, a_{h, k}) \\
		&= \sum_{k=1}^K X_k   \le \sum_{k=1}^K \EB[X_k|\FM_{k-1}] 
		+ \sqrt{2	(\sigma_{R}^2+ \HM^2) \VM^2 HK \log \frac{1}{\delta}}
		+ \frac{2}{3} H (\sigma_{R}^2+ \HM^2) \log\frac{1}{\delta} 	\\
		&\le \VM^2K + 2 \sqrt{  \VM^2K \cdot H (\sigma_{R}^2+ \HM^2) \log\frac{1}{\delta} 	 } + \frac{2}{3} H (\sigma_{R}^2+ \HM^2) \log\frac{1}{\delta} 	\\
		&\le 2 \VM^2 K + 2H (\sigma_{R}^2+ \HM^2) \log\frac{1}{\delta} .
	\end{align*}
\end{proof}

\section{Auxiliary Lemmas}
\label{proof:auxiliary}
\subsection{Concentration Inequalities}
\begin{lem}[Freedman inequality~\citep{freedman1975tail}]
	\label{lem:freedman}
	Let $\{ X_t \}_{t \in [T]}$ be a stochastic process that adapts to the filtration $\FM_t$ so that $X_t$ is $\FM_t$-measurable, $\EB[X_t|\FM_{t-1}] = 0$, $|X_t| \le M$ and $\sum_{t=1}^T \EB[X_t^2|\FM_{t-1}] \le V$ where $M>0$ and $V > 0$ are positive constants.
	Then with probability at least $1-\delta$, we have
	\[
	\sum_{t=1}^T X_t \le \sqrt{2V\ln\frac{1}{\delta}} + \frac{2M}{3}\ln\frac{1}{\delta}.
	\]
\end{lem}

\begin{lem}[Variance-aware Freedman inequality]
	\label{lem:bern}
	Let $\{ X_t \}_{t \in [T]}$ be a stochastic process that adapts to the filtration $\FM_t$ so that $X_t$ is $\FM_t$-measurable, $\EB[X_t|\FM_{t-1}] = 0$, $|X_t| \le M$ and $\sum_{t=1}^T \EB[X_t^2|\FM_{t-1}] \le V^2$ where $M>0$ and $V > 0$ are positive constants.
	Then with probability at least $1-\delta$, we have
	\[
	\left|\sum_{t=1}^T X_t\right|  \le 3 \sqrt{\sum_{t=1}^T \EB[X_t^2|\FM_{t-1}]  \cdot \log\frac{2K}{\delta}} + 5M\log\frac{2K}{\delta}
	\]
	where $K = 1+\ceil{2\log_2 \frac{V}{M} } $.
\end{lem}
\begin{proof}[Proof of Lemma~\ref{lem:bern}]
	By Theorem 5 in~\citep{li2021q}, we have for any positive integer $K \ge 1$, 
	\[
	\PB\left(
	\left|\sum_{t=1}^T X_t\right|  \le \sqrt{8\max\left\{  \sum_{t=1}^T \EB[X_t^2|\FM_{t-1}]  ,  \frac{V^2}{2^K}\right\}\cdot  \ln\frac{2K}{\delta}}+ \frac{4M}{3}\ln\frac{2K}{\delta}
	\right) \ge 1-\delta.
	\]
	By setting $K = 1+\ceil{2\log_2 \frac{V}{M} } $, we have $\frac{V^2}{2^K} \le M^2$.
	Using $\max\{a, b\} \le a + b$, $\sqrt{a+b} \le \sqrt{a} + \sqrt{b}$ for any $a, b \ge0$ and $\ln \frac{2K}{\delta} \ge 1$, we complete the proof.
\end{proof}

The following two lemmas are the counterpart lemmas of Theorem~\ref{thm:heavy} under light-tail assumption.
\begin{lem}[Bernstein inequality for self-normalized martingales, Lemma F.4 in~\citep{hu2022nearly}]
	\label{lem:self-bern}
	Let $\{ \GM_t\}_{t \ge 0}$ be a filtration and $ \{\x_t, \eta_t \}_{t \ge 0}$ be a stochastic process so that $\x_t \in \RB^d$ is $\GM_t$-measurable and $\eta_t \in \RB$ is $\GM_{t+1}$-measurable..
	If $\|\x_t\| \le L$ and $\{\eta_t\}_{t \ge 1}$ satisfies that $\EB[\eta_t|\GM_t] = 0$, $	\EB[\eta_t^2|\GM_t] \le \sigma^2$ and $|\eta_t \min \left\{ 1, \| \x_t\|_{\Z_{t-1}^{-1}} \right\}| \le M$ for all $ t \ge 1$.
	Then, for any $\delta \in (0, 1)$, with probability at least $1-\delta$, we have for all $t \ge 1$,
	\[
	\left\|  \sum_{j=1}^t \x_j \eta_j \right\|_{\Z_t^{-1}} \le 
	8 \sigma \sqrt{d \log\left(1+\frac{tL^2}{d\lambda}\right)\log \frac{4t^2}{\delta}} + 4 M \log \frac{4t^2}{\delta}
	\]
	where $\Z_t = \lambda \I + \sum_{j=1}^t \x_j \x_j^\top$ for $t\ge 1$ and $\Z_0 = \lambda \I$.
\end{lem}

\begin{lem}[Hoeffding inequality for self-normalized martingales, Theorem 1 in~\citep{abbasi2011improved}]
	\label{lem:self-hoff}
	Let $\{ \GM_t\}_{t \ge 0}$ be a filtration and $ \{\x_t, \eta_t \}_{t \ge 0}$ be a stochastic process so that $\x_t \in \RB^d$ is $\GM_t$-measurable and $\eta_t \in \RB$ is $\GM_{t+1}$-measurable..
	If $\|\x_t\| \le L$ and $\{\eta_t\}_{t \ge 1}$ satisfies that $\EB[\eta_t|\GM_t] = 0$ and $|\eta_t| \le M$ for all $ t \ge 1$.
	Then, for any $\delta \in (0, 1)$, with probability at least $1-\delta$, we have for all $t \ge 1$,
	\[
	\left\|  \sum_{j=1}^t \x_j \eta_j \right\|_{\Z_t^{-1}} \le 
	M \sqrt{d \log\left(1+\frac{tL^2}{d\lambda}\right) + \log \frac{1}{\delta}} 
	\]
	where $\Z_t = \lambda \I + \sum_{j=1}^t \x_j \x_j^\top$ for $t\ge 1$ and $\Z_0 = \lambda \I$.
\end{lem}

\subsection{Elliptical Lemmas}

\begin{lem}[Lemma 11 in~\citep{abbasi2011improved}]
	\label{lem:w-sum}
	Let $\{\x_t\}_{t \ge 1} \subset \RB^d$ and assume $\|\x_t\| \le L$ for all $t \ge 1$.
	Set $\Z_t = \sum_{s=1}^t \x_t \x_t^\top + \lambda \I$.
	Then it follows that
	\[
	\sum_{t=1}^T \min \left\{ 1,  \|\x_t\|^2_{\Z_{t-1}} \right\}
	\le 2d \log\left(\frac{d\lambda + T L^2}{d\lambda}\right).
	\]
\end{lem}

\begin{lem}[Lemma 12 in~\citep{abbasi2011improved}]
	\label{lem:matrix-ratio}
	Suppose $\A, \B \in \RB^{d \times d}$  are two positive definite matrices satisfying that $\A \succeq \B$, then for any $\x \in \RB^d$,
	\[
	\|\x\|_{\B^{-1}} \le \|\x\|_{\A^{-1}} \sqrt{\frac{\mathrm{det}(\A)}{\mathrm{det}(\B)}}.
	\]
\end{lem}

\subsection{Function Class and Covering Number}
This subsection collects important lemmas in~\citep{he2022nearly}.
Let $\KM = \left\{ k_1, k_2, \cdots\right\}$ denote the set of episodes where the algorithm updates the value function in Algorithm~\ref{algo:linear}.
For a given total number of episodes $K$, it definitely follows that $|\KM| \le K$.
Furthermore, due to the mechanism of rare-switching value function updates, $|\KM|$ is actually much smaller than $K$.
\begin{lem}
	\label{lem:rare-update}
	\[
	|\KM| \le  dH \log_2 \left( 1 + \frac{K}{\lambda \sigma_{\min}^2 }\right).
	\]
\end{lem}
\begin{proof}[Proof of Lemma~\ref{lem:rare-update}]
	The proof is almost identical to Lemma E.1 in~\citep{he2022nearly} except that we maintain the dependence on $\sigma_{\min}$.
According to the determinant-based criterion, for each episode $k_i$, there exists a stage $h' \in [H]$ such that $\mathrm{det}(
\H_{h',k_i-1}) \ge 2 \mathrm{det}( \H_{h',k_{i-1}-1})$.
Since we always have $\H_{h,k_i-1} \succeq \H_{h,k_{i-1}-1}$ for all $h \in [H]$, it then follows that
\[
\prod_{h \in [H]} \mathrm{det}(\H_{h,k_i-1}) \ge 2 \prod_{h \in [H]} \mathrm{det}(\H_{h,k_{i-1}-1}).
\]
By induction, it follows that
\[
\prod_{h \in [H]} \mathrm{det}(\H_{h,k_{|\KM|}-1}) \ge 2^{|\KM|}  \prod_{h \in [H]} \mathrm{det}(\H_{h,k_{1}-1})
\ge 2^{|\KM|} \prod_{h \in [H]} \mathrm{det}(\lambda \I) = 2^{|\KM|} \lambda^{d H }
\]
On the other hand, due to $\H_{h,k_{|\KM|}-1} \preceq \H_{h, K}$ the determinant $\mathrm{det}(\H_{h,k_{|\KM|}-1}$ is upper bounded by
\[
\prod_{h \in [H]} \mathrm{det}(\H_{h,k_{|\KM|}-1}) 
\le \prod_{h \in [H]} \mathrm{det}(\H_{h,K})  \le \left( \lambda + \frac{K}{\sigma_{\min}^2} \right)^{dH}.
\]
Combining the last two inequalities, we have
\[
	|\KM| \le  dH \log_2 \left( 1 + \frac{K}{\lambda \sigma_{\min}^2 }\right).
\]
\end{proof}
The optimistic value function $\overoptV_h^k(\cdot) = \min_{k_i \le k} \max_a \overoptQ_h^{k_i}(\cdot, a)$ belong to the function class $\VM^{+}$
\begin{equation}
\label{eq:function-pos}
\VM^{+} = \left\{
f | 	f (\cdot) = \max_{a \in \AM}\min_{i \le |\KM|} \min\left\{ \w_i^\top \Bphi(\cdot, a) + \beta \|\Bphi(\cdot, a) \|_{\H_i^{-1}}, \HM
\right\}, \beta \in [0, B],  \|\w_i \| \le L, \H_i \succeq \lambda \I
\right\}.
\end{equation}
while the pessimistic value function $\overpesV_h^k(\cdot) = \max_{k_i \le k} \max_a \overpesQ_h^{k_i}(\cdot, a)$ belong to the function class $\VM^{-}$,
\begin{equation}
	\label{eq:function-pes}
	\VM^{-} = \left\{
	f | 	f (\cdot) = \max_{a \in \AM}\max_{i \le |\KM|} \max\left\{ \w_i^\top \Bphi(\cdot, a) - \beta \|\Bphi(\cdot, a) \|_{\H_i^{-1}}, \HM
	\right\}, \beta \in [0, B], \|\w_i \| \le L, \H_i \succeq \lambda \I
	\right\}.
\end{equation}
Here $B$ upper bounds $\beta$ and $L = W + \HM \sqrt{\frac{dK}{\lambda}}$ is a uniformly bound for $\est_{h, k-1} + \Bmu_{h,k-1} \BoveroptV_{h+1}^k$ because
\[
\|\est_{h, k-1} + \Bmu_{h,k-1} \BoveroptV_{h+1}^k\|
\le \|\est_{h, k-1} \| +\| \Bmu_{h,k-1} \BoveroptV_{h+1}^k\|
\le W + \HM \sqrt{\frac{dK}{\lambda}}
\]
where the last inequality uses the boundedness of $\Btheta_{h, k-1}$'s and the inequality $\|\Bmu_{h,k-1} \BoveroptV_{h+1}^k\| \le \HM \sqrt{\frac{dK}{\lambda}}$ (whose proof can be found in Lemma E.2 of~\citet{he2022nearly}).

\begin{lem}[Covering number of value functions]
	\label{lem:covering-1}
	Let $\VM^{\pm}$ denote the class of optimistic or pessimistic value functions with definition in~\eqref{eq:function-pos} and~\eqref{eq:function-pes} respectively.
	Assume $\|\Bphi(s, a)\| \le 1$ for all $(s, a)$ pairs, and let $\NM(\VM, \varepsilon)$ be the $\varepsilon$-covering number of $\VM$ with respective to the distance $\mathrm{dist}(f, f'):= \sup_{s \in \SM} |f(s)-f'(s)|$.
	Then,
	\[
	\log \NM(\VM^\pm, \varepsilon) 
	\le  \left[ d \log \left(1+ \frac{4L}{\varepsilon}\right) + d^2 \log \left( 1 +\frac{ 8 d^{1/2} B^2}{\lambda \varepsilon^2} \right)\right] \cdot |\KM|.
	\]
\end{lem}
\begin{proof}[Proof of Lemma~\ref{lem:covering-1}]
	The result about $\VM_{f}^{+}$ follows from Lemma E.6 in~\citep{he2022nearly}.
	The result about $\VM_{f}^{-}$ follows from Lemma E.7 in~\citep{he2022nearly}.
\end{proof}

\begin{lem}[Covering number of squared functions, Lemma E.8 in~\citep{he2022nearly}]
	\label{lem:covering-square}
	For the squared function class $[\VM^{+}]^2 :=  \left\{ f^2| f \in \VM^{+}\right\}$, let $\NM([\VM^{+}]^2, \varepsilon)$ be the $\varepsilon$-covering number of $[\VM^{+}]^2$ with respective to the distance $\mathrm{dist}(f, f'):= \sup_{s \in \SM} |f(s)-f'(s)|$.
	Then
		\[
	\log \NM([\VM^{+}]^2, \varepsilon) 
	\le  \left[ d \log \left(1+ \frac{8HL}{\varepsilon}\right) + d^2 \log \left( 1 +\frac{ 32d^{1/2}H^2 B^2}{\lambda \varepsilon^2} \right)\right] \cdot |\KM|.
	\]
\end{lem}

\end{appendix}

\end{document}